\documentclass[sigconf]{acmart}
\usepackage{tablefootnote}
\usepackage{xspace}
\usepackage{enumitem}
\usepackage[ruled,linesnumbered,noend]{algorithm2e}
\usepackage{subfigure}
\usepackage{xcolor}
\usepackage{multirow}
\usepackage{url}
\usepackage{amsthm}
\usepackage{makecell}

\newcommand{\method}{\textsc{NeuKron}\xspace}
\newcommand{\kronfit}{{KronFit}\xspace}

\newcommand{\change}[1]{\textcolor{black}{#1}}
\newcommand{\red}[1]{\textcolor{black}{#1}}
\newcommand{\blue}[1]{\textcolor{black}{#1}}

\newcommand{\kijung}[1]{\textcolor{black}{#1}}

\newcommand{\fnorm}[1]{\lVert #1 \rVert_{F}}
\newcommand{\floor}[1]{\left\lfloor #1 \right\rfloor}
\newcommand{\ceil}[1]{\left\lceil #1 \right\rceil}
\newcommand{\kpower}[2]{#1^{\otimes#2}}

\newcommand{\model}{\mathbf{\Theta}}
\newcommand{\supplelink}{\cite{appendix}}
\newcommand{\atij}{\tilde{a}_{ij}}
\newcommand{\aij}{a_{ij}}
\newcommand{\bA}{\mathbf{A}}
\newcommand{\cX}{\mathcal{X}}
\newcommand{\bK}{\mathbf{K}}
\newcommand{\reA}{\mathbf{\tilde{A}}_\model}
\newcommand{\reX}{\mathbf{\tilde{\mathcal{X}}}_\model}
\newcommand{\rome}[1]{\uppercase\expandafter{\romannumeral #1\relax}}

\newcommand{\smallsection}[1]{\noindent\underline{\smash{\textbf{#1:}}}}
\newcommand{\mat}[1]{\mathbf{#1}}

\newcommand{\textt}[1]{\scalebox{1.0}[1.0]{\texttt{#1}}}

\newtheorem{theorem}{Theorem}
\newtheorem{lemma}{Lemma}

\newtheorem{problem}{Problem}
\newtheorem{example}{Example}

\SetKwComment{Comment}{$\triangleright$\ }{}

\SetCommentSty{mycommfont}
\SetAlFnt{\small}

\AtBeginDocument{%
  }

\copyrightyear{2023}
\acmYear{2023}
\setcopyright{acmlicensed}\acmConference[WWW '23]{Proceedings of the ACM Web Conference 2023}{May 1--5, 2023}{Austin, TX, USA}
\acmBooktitle{Proceedings of the ACM Web Conference 2023 (WWW '23), May 1--5, 2023, Austin, TX, USA}
\acmPrice{15.00}
\acmDOI{10.1145/3543507.3583226}
\acmISBN{978-1-4503-9416-1/23/04}

\begin{document}

\title{NeuKron: Constant-Size Lossy Compression of Sparse Reorderable Matrices and Tensors}

\settopmatter{authorsperrow=4}
\author{Taehyung Kwon}
\authornote{Both authors contributed equally to this research.}
\affiliation{%
  \institution{Kim Jaechul Graduate School of AI, KAIST}
  \city{Seoul}
  \country{South Korea}
}
\email{taehyung.kwon@kaist.ac.kr}

\author{Jihoon Ko}
\authornotemark[1]
\affiliation{%
  \institution{Kim Jaechul Graduate School of AI, KAIST}
  \city{Seoul}
  \country{South Korea}
}
\email{jihoonko@kaist.ac.kr}

\author{Jinhong Jung}
\affiliation{
  \institution{Dept. of CSE, Jeonbuk National University}
  \city{Jeonju}
  \country{South Korea}
}
\email{jinhongjung@jbnu.ac.kr}

\author{Kijung Shin}
\affiliation{%
  \institution{Kim Jaechul Graduate School of AI, KAIST}
  \city{Seoul}
  \country{South Korea}
}
\email{kijungs@kaist.ac.kr}


\begin{abstract}
Many real-world data are naturally represented as a sparse reorderable matrix, whose rows and columns can be arbitrarily ordered \blue{(e.g., the adjacency matrix of a bipartite graph)}.
Storing a sparse matrix in conventional ways requires an amount of space linear in the number of non-zeros, and lossy compression of sparse matrices (e.g., Truncated SVD) typically requires an amount of space linear in the number of rows and columns.
In this work, we propose \method for compressing a sparse reorderable matrix into a constant-size space.
\method generalizes Kronecker products using \red{a recurrent neural network} with a constant number of parameters.
\method updates the parameters so that a given matrix is approximated by the product and reorders the rows and columns of the matrix to facilitate the approximation.
The updates take time linear in the number of non-zeros in the input matrix, and the approximation of each entry can be retrieved in logarithmic time.
We also extend \method to compress sparse reorderable tensors \blue{(e.g. multi-layer graphs)}, which generalize matrices.
Through experiments on ten real-world datasets, we show that \method is
  \textbf{(a) Compact:} requiring up to five orders of magnitude less space than its best competitor with similar approximation errors,
  \textbf{(b) Accurate:} giving up to $10\times$ smaller approximation error than its best competitors with similar size outputs, and
  \textbf{(c) Scalable:} successfully compressing a matrix with over 230 million non-zero entries.
\end{abstract}



\maketitle
\section{Introduction}
\label{sec:intro}
\kijung{We consider a matrix to be \textit{sparse} if the number of non-zero entries is much smaller than that of all entries.}
Sparse matrices naturally represent many types of data from various domains, as follows:
\vspace{-2mm}
\begin{itemize}[leftmargin=*]
\item \textbf{E-commerce}: User-item matrices represent how many times each user purchased each item  \cite{he2016ups, mcauley2012learning}.
\item \textbf{Search Engines}: Document-keyword matrices represent how many times each document contains each keyword \cite{liu2009bbm}. User-ad matrices indicate how many times each user clicked each ad \red{given} by search engines \cite{sidana2017kasandr}.
\item \textbf{Social Media}: The adjacency matrices of social networks indicate \red{friendship} between users~\cite{dhulipala2016compressing,shin2019sweg}. User-group matrices indicate which user belongs to each group~\cite{yang2012defining}.
\item \textbf{Bibliography}: Author-paper matrices represent who authored each paper~\cite{Sinha-2015-MAG}. 
The adjacency matrices of collaboration networks represent co-authorships between authors~\cite{yang2012defining}.
\end{itemize}
\vspace{-2mm}

\blue{Despite their sparsity, many real-world matrices  require considerable space.
Examples include user-ad matrices~\cite{ting2018count} and the adjacency matrices of web graphs~\cite{boldi2004webgraph} with billions of rows or columns; and keyword-document matrices \cite{liu2009bbm} and the adjacency matrices of online social networks \cite{dhulipala2016compressing,shin2019sweg} with tens of billions of non-zeros.} 

Compression of such large sparse matrices becomes important as smartphones and IoT devices become popular. 
Such memory-limited mobile devices are often required to process a large amount of data without sending them to clouds or servers, due to potential privacy risks \cite{konevcny2016federated}.
Moreover, as the size of large-scale matrices grows rapidly, storing them is challenging also in desktops and servers \cite{shin2019sweg,dhulipala2016compressing,beutel2015accams}, and for federate learning, compressing matrices is required to reduce communication costs \cite{nam2022fedpara}.
As a result, a large number of lossy matrix-compression techniques \cite{sun2007less,beutel2015accams,drineas2006fast}  have been developed over the last few decades.

To the best of our knowledge, existing lossy-compression methods for sparse matrices create outputs whose sizes are at least linear in the numbers of rows and columns of the input matrix.
For example, given an $N$-by-$M$ matrix $\mathbf{A}$ and a positive integer $K$, truncated singular value decomposition (T-SVD) \cite{eckart1936approximation,stewart1993early} outputs two matrices of which the numbers of entries are $O(KN)$ and $O(KM)$.
Recent methods \cite{drineas2006fast,sun2007less,beutel2015accams} have the same limitations, while they provide a better trade-off between space and information loss than T-SVD. 

Can we compress a matrix into a constant-size space, which can even be smaller than the number of rows and columns? 
In this paper, we exploit the fact that \textbf{many real-world sparse matrices are \textit{reorderable}}, i.e., the rows and columns of the matrices can be arbitrarily ordered.\footnote{\label{footnote:ex} A matrix is \textit{non-reorderable} if the orders of rows and columns in it convey information. For example, images and multivariate time series are non-reorderable matrices since the orders of rows and columns in them indicate spatial and temporal adjacency.}
All of the matrices discussed in the first paragraph, \blue{which are essentially bipartite graphs (nodes of one type correspond to rows, and  nodes of the other type correspond to columns)}, are reorderable.
For example, in the case of a user-item matrix built based on e-commerce data, which user (item) comes next to which user (item) does not matter.
\kijung{Our key idea is to \textbf{order 
rows and columns} to facilitate our model to learn and exploit meaningful patterns in the input matrix for compression.} 

Specifically, we present \method, a constant-size lossy compression method for sparse reorderable matrices.
\red{It} consists of a machine-learning model and novel training schemes.
The model generalizes the Kronecker power and enhances its expressive power using a \red{recurrent neural network} with a constant number of parameters. \kijung{The training scheme, which is crucial for performance, is to reorder rows and columns in the input matrix to create patterns that the machine-learning model can exploit for better compression.
}
Consider \red{an} $N$-by-$M$ matrix with $L$ non-zeros, where $N \!\leq \!M$ without loss of generality. 
The model and the training schemes are designed carefully so that each training epoch takes $O(M+L\log M)$ time, and after training, the approximation of each entry can be retrieved in $O(\log M)$ time. 
Note that the time complexity of training depends only on the number of non-zeros instead of all entries.

In addition, we extend \method for lossy-compression of sparse reorderable tensors while maintaining its strengths. Tensors (i.e., multi-dimensional arrays) generalize matrices to  higher dimensions, and in other words, matrices are 2-order tensors. 
Sparse tensors have been used widely for various purposes, including context-aware recommender systems \cite{karatzoglou2010multiverse} and knowledge base completion \cite{lacroix2018canonical}), and for lossy compression of them, tensor decomposition methods (e.g., CP \cite{carroll1970analysis,bader2008efficient} and Tucker \cite{tucker1966some,bader2008efficient}) have been developed.

For evaluation, we perform extensive experiments using ten real-world matrices \blue{(spec., bipartite graphs)} and tensors. 
The results reveal the following advantages of \method:
\vspace{-1mm}
\begin{itemize}[leftmargin=*]
    \item \textbf{Compact:} Its output is up to \textbf{5 orders of magnitude smaller} than \red{competitors' with similar approximation error.}
    \item \textbf{Accurate:} It achieves up to \textbf{10.1$\times$ smaller approximation error} than its best competitors that give similar-size outputs.
    \item \textbf{Scalable:} Its running time is \textbf{linear} in the number of non-zero entries, and it successfully compresses matrices with \textbf{over 230 millions of non-zero entries} on commodity GPUs. 
\end{itemize}

\smallsection{Reproducibility}
The code and datasets are available at \supplelink.

\smallsection{Remarks on non-reorderable matrices}
While we focus on reorderable matrices in this paper, \method can also be applied to non-reorderable matrices if the mapping between the original and new orders of rows and columns are stored additionally. \blue{We present a related experiment in
Appendix~\ref{app:non_reorder}.}


\vspace{-1mm}

\vspace{-1mm}
\section{Related Works}
\label{sec:related}
In this section, we review lossy-compression methods for matrices and tensors. \kijung{Those for lossy compression of sparse matrices or tensors of any size are compared in Table~\ref{tab:proscons}
and also in Section~\ref{sec:experiments}.} 

\smallsection{Factorization-based matrix compression}
Given a matrix $\mathbf{A} \in \mathbb{R}^{N \times M}$, singular value decomposition (SVD) \cite{golub1971singular} decomposes $\mathbf{A}$ into $\mathbf{U}\Sigma \mathbf{V}^T$ where $\mathbf{U} \in \mathbb{R}^{N \times R}$, $\mathbf{V} \in \mathbb{R}^{M \times R}$, $\mathbf{\Sigma}$ is a diagonal matrix with its singular values, and $R$ is its rank.
Truncated SVD (T-SVD) \cite{eckart1936approximation,stewart1993early} outputs the $K$ $(\leq R)$ largest singular values and the corresponding vectors of $\mathbf{U}$ and $\mathbf{V}$ from which the rank-$K$ approximation of $\mathbf{A}$ best in terms of the Frobenius norm  can be obtained \cite{stewart1993early}.
Its outputs have $O(K(M+N))$ real values, \kijung{and typically most of them are non-zero.}
\kijung{
For further compression, CUR decomposition~\cite{drineas2006fast} aims to yield sparse outputs.
Specifically, a sparse matrix $\mathbf{A}$ is decomposed into $\mathbf{CUR}$ (i.e., $\mathbf{A}\approx\mathbf{CUR}$),
where $\mathbf{C} \in \mathbb{R}^{N \times K}$ and $\mathbf{R} \in \mathbb{R}^{K \times M}$ are constructed by sampling $K$ columns and rows from $\mathbf{A}$, respectively.
The matrix $\mathbf{U} \in \mathbb{R}^{K \times K}$ is dense but small, and it is determined by $\mathbf{C}$ and $\mathbf{R}$ so that the approximation error is minimized.}
%
%
Compact matrix decomposition (CMD) \cite{sun2007less} keeps only unique columns and rows in $\mathbf{C}$ and $\mathbf{R}$ for further efficiency. 

\smallsection{Co-clustering-based matrix compression} ACCAMS and bACCAMS~\cite{beutel2015accams} use an additive combination of small co-clusters to approximate a given matrix. While the numbers of parameters of them are linear in the numbers of rows and columns,
they produce intermediate results whose size is linear in the number of (potentially zero) known entries.
Thus, they are \kijung{computationally and memory inefficient} when most entries are known but zero. 

\begin{table}
    \centering
    \caption{\label{tab:proscons} \kijung{Comparison of lossy-compression methods for sparse matrices and tensors.
    For simplicity,
    we treat the tensor order and all hyperparameters as constants.
    Comparisons are relative, and we provide details in \cite{appendix}.
    }}
        \setlength{\tabcolsep}{2pt}
        \scalebox{0.75}{
        \hspace{-2mm}
        \begin{tabular}{ccccccc}
        \toprule
        \multirow{3}{*}{Methods} & Space \& & Training & Inference & Number  & Training \\
         & Accuracy & Complexity & Complexity & of Hyper- & Time \\
         & Trade-off & (per iteration) & (per entry) & parameters & (total) \\
        \midrule
        \method & Strong & $\propto$ \#non-zeros & $\propto$ $\log(N_{\max})$*  & 4** & Long \\
        \midrule
        T-SVD \cite{eckart1936approximation,vannieuwenhoven2012new} & Weak & $\propto$ \#non-zeros & constant & 1 & Short \\
        CMD \cite{sun2007less}, CUR \cite{drineas2006fast} & Moderate & $\propto$ \#non-zeros & constant & 2 & Moderate \\ 
        ACCAMS \cite{beutel2015accams} & Moderate & $\propto$ \#all-entries & constant & 2 & Moderate \\
        bACCAMS \cite{beutel2015accams} & Moderate & $\propto$ \#all-entries & constant & 4 & Long \\
        \kronfit \cite{leskovec2007scalable,leskovec2010kronecker} & Weak & $\propto$ \#non-zeros & $\propto$ $\log(N_{\max})$* & 4 & Long \\
        \midrule
        CP \cite{carroll1970analysis}, Tucker \cite{tucker1966some} & Weak & $\propto$ \#non-zeros & constant & 1 & Moderate \\
        \bottomrule
        \multicolumn{6}{l}{\kijung{* Here $N_{\max}=\max(N_1, \cdots, N_{D})$ is the maximum dimensionality (i.e., mode length).}} \\
        \multicolumn{6}{l}{\kijung{** The learning rate, the optimizer, the weight parameter for the criterion of switching,}} \\
        \multicolumn{6}{l}{\kijung{and the size of hidden dimensions in LSTM.}}
    \end{tabular}}
    \vspace{-1mm}
\end{table}

\smallsection{Kronecker product-based matrix compression}
The adjacency matrix of a Kronecker graph \cite{leskovec2005realistic} is a Kronecker power of a fixed seed matrix (e.g., $2$-by-$2$ matrix).
\kronfit~\cite{leskovec2007scalable, leskovec2010kronecker} searches for a seed matrix whose Kronecker power approximates the adjacency matrix of a given graph.
While \kronfit is designed for adjacency matrices, it can be easily extended to matrices of any size, and the output seed matrix can be considered as a constant-size lossy compression of a given matrix.
However, the approximation error is considerable, even when the seed matrix is large, due to the inflexibility of the Kronecker product, as shown in Section~\ref{sec:exp:model}.

\smallsection{Tensor compression} 
CP decomposition (CP)~\cite{carroll1970analysis} and Tucker decomposition (Tucker)~\cite{tucker1966some} generalize
the aforementioned T-SVD to higher-order tensors.
They approximate a given tensor using the sums and products (e.g., outer product and $n$-mode product) of much smaller low-rank tensors and matrices, which can be considered as a lossy compression of the given tensor.
Efficient CP and Tucker methods for sparse tensors have been developed \cite{bader2008efficient}.
For lossless compression of sparse tensors, compressed sparse fiber~(CSF)~\cite{smith2015splatt, smith2015tensor} is available.

\smallsection{Other related works}
\blue{Unipartite-graph summarization algorithms \cite{lee2020ssumm,lefevre2010grass,riondato2017graph} can be used for compressing adjacency matrices of unipartite graphs, while they cannot be directly applied to weighted and/or non-symmetric matrices, which we aim to compress.}
Matrix sketching methods replace a given large matrix with a more compact matrix that follows the properties of the input matrix, for example, by leaving only important columns (rows) of the input matrix \cite{drineas2006fast,drineas2008relative}.
These methods, however, cannot be applied to our problem 
because the entries of the input matrix cannot be estimated directly from their outputs. 
\vspace{-2mm}
\section{Notations and problem definition}
\label{sec:preliminaries}
\vspace{-0.5mm}

In this section, we introduce basic concepts and give a formal problem definition.
\red{See Table~\ref{tab:notation} for common notations.}

\vspace{-2mm}
\subsection{Notations and Concepts}
\vspace{-0.5mm}

\smallsection{Sparse reorderable matrix and tensor} 
A \textit{matrix} $\bA \in \mathbb{R}^{N \times M}$ is a 2-dimensional array with $N$ rows and $M$ columns, and real entries. 
A $D$-order \textit{tensor} $\cX \in \mathbb{R}^{N_1 \times \cdots \times N_{D}}$ is a $D$-dimensional array of size $N_1\times\cdots\times N_D$ with real entries. 
We use $a_{ij}$ or $A(i,j)$ to denote the $(i,j)$-th entry of $\bA$, and we use $x_{i_1, \cdots, i_D}$ to denote the ($i_1$, $\cdots$, $i_{D}$)-th entry of $\cX$.
\kijung{We consider a matrix or a tensor to be \textit{sparse} if the number of non-zero entries is much smaller than that of all entries.\footnote{\kijung{The ratio is at most 0.0046 in the datasets considered in the paper.}}}
We call a matrix \textit{reorderable} if its rows and columns can be arbitrarily ordered. We provide some examples of reorderable matrices where the orders of rows and columns do not convey any information and some examples of non-reorderable ones (see Footnote~\ref{footnote:ex}) in Section~\ref{sec:intro}.
Similarly, we call a tensor reorderable if the indices in each mode can be arbitrarily ordered.



\smallsection{Approximation error} The \textit{Frobenius norm} is a function $\lVert \cdot \rVert_{F} : \mathbb{R}^{N \times M} \rightarrow \mathbb{R}$ defined as the square root of the square sum of all entries in the given matrix.
Similarly, the Frobenius norm of a tensor is defined as the square root of the square sum of all entries in the given tensor.
The \textit{approximation error} of a matrix $\reA$ that approximates $\bA$ is defined as $\lVert \bA - \reA \rVert_F^2$.
Similarly, the approximation error of $\reX$ that approximates $\cX$ is defined as $\lVert \cX - \reX \rVert_F^2$. 

\smallsection{Kronecker product and power} 
Given two matrices $\bA \in \mathbb{R}^{N \times M}$ and $\mat{B} \in \mathbb{R}^{P \times Q}$, the \textit{Kronecker product} $\bA \otimes \mat{B} \in \mathbb{R}^{NP \times MQ}$ is a large matrix formed by multiplying $\mat{B}$ by each element of $\bA$, i.e.,
\vspace{-2mm}
\begin{align*}
\bA \otimes \mathbf{B} := 
\begin{bmatrix}
a_{11} \mathbf{B} & \cdots & a_{1M}\mathbf{B} \\
\vdots & \ddots & \vdots \\
a_{N1} \mathbf{B} & \cdots & a_{NM}\mathbf{B}
\end{bmatrix}.
\end{align*}
We denote the $l$-th \textit{Kronecker power} of $\bA$ as $\kpower{\bA}{l}$, where $\kpower{\bA}{l}=\kpower{\bA}{(l-1)} \otimes \bA$ and $\kpower{\bA}{1}=\bA$.
\vspace{-2mm}

\subsection{Problem Definition}
\vspace{-0.5mm}
\label{sec:preliminaries:problem}
The constant-size lossy matrix compression problem that we address in this paper is defined in Problem~\ref{prob}.
\noindent It should be noted that the given constant $k$ can be even smaller than $N$ and $M$.
The problem of \textit{constant-size lossy compression of a sparse reorderable tensor} can be defined by simply replacing the matrix $\bA$ with a tensor $\cX$ and $\lVert\bA - \reA\rVert_F^2$ with $\lVert \cX - \reX \rVert_F^2$.

\vspace{0.5mm}
\noindent\fbox{%
\parbox{0.975\columnwidth}{%
\vspace{-2mm}
\begin{problem}\label{prob}
\textsc{\normalfont{(Constant-size Lossy Compression of a Sparse}} \textsc{\normalfont{Reorderable Matrix)}} 
\begin{itemize}[leftmargin=*]
    \item \textbf{Given:} (1) a sparse and reorderable matrix $\bA \in \mathbb{R}^{N \times M}$, \\
    \indent \hspace{9mm} (2) a constant $k= O(1)$,
    \item \textbf{Find:} a model $\model$ 
    \item \textbf{to Minimize:} the approximation error 
        $\lVert\bA - \reA\rVert_F^2$, where $\reA$ is the matrix approximated from $\model$.
    \item \textbf{Subject to:}
    the number of parameters in $\model$ is at most $k$.
\end{itemize}
\end{problem}
\vspace{-1mm}
}%
}


\begin{table}[t]
    \centering
    \caption{Frequently-used notations}
    \label{tab:notation}
    \scalebox{0.8}{
    \begin{tabular}{c|l}
        \toprule
        \textbf{Symbol} &  \textbf{Definition} \\
        \midrule
        $\bA\in \mathbb{R}^{N \times M}$ & an $N$-by-$M$ sparse matrix \\
        $a_{ij}$ or $\bA(i, j)$ & ($i$, $j$)-th entry of $\bA$ \\
        $\bA_{i,:}$, $\bA_{:,i}$ & $i$-th row of $\bA$, $i$-th column of $\bA$ \\
        \midrule
        $\cX\in \mathbb{R}^{N_1 \times \cdots \times N_{D}}$ & tensor \\
        $D$ & order of $\cX$ \\
        $x_{i_1, \cdots, i_D}$ & ($i_1$, $\cdots$, $i_D$)-th entry of $\cX$ \\
        \midrule
        $\text{nnz}(\bA)$, $\text{nnz}(\cX)$ & number of non-zero entries in $\bA$ and $\cX$ \\
        $\fnorm{\bA}, \fnorm{\cX}$ & Frobenius norm of $\bA$ and $\cX$\\
        $\otimes$ & Kronecker product \\
        $\kpower{\bA}{l}$, $\kpower{\cX}{l}$ & $l$-th Kronecker power of $\bA$ and $\mathbf{\cX}$ \\
        $\model$ & a \method model which compresses $\bA$ and $\cX$ \\
        $\reA$, $\reX$  & approximated matrix and tensor of $\bA$ and $\cX$ by $\model$ \\
        $q$ & a parameter for the scale of model outputs \\
        $h$ & hidden dimension in LSTM \\
        $[n]$ & a set of integers from $1$ to $n$ (i.e., $\{1, 2, \cdots, n\}$) \\
        \bottomrule
    \end{tabular}
    }
\end{table}
\vspace{-1.5mm}
\section{Proposed Method}
\vspace{-1mm}
\label{sec:method}
\begin{figure*}[t]
    \vspace{-3mm}
     \centering
     \hspace{-1mm}
     \subfigure[Encoding Examples]{
         \centering
         \includegraphics[width=0.3\linewidth]{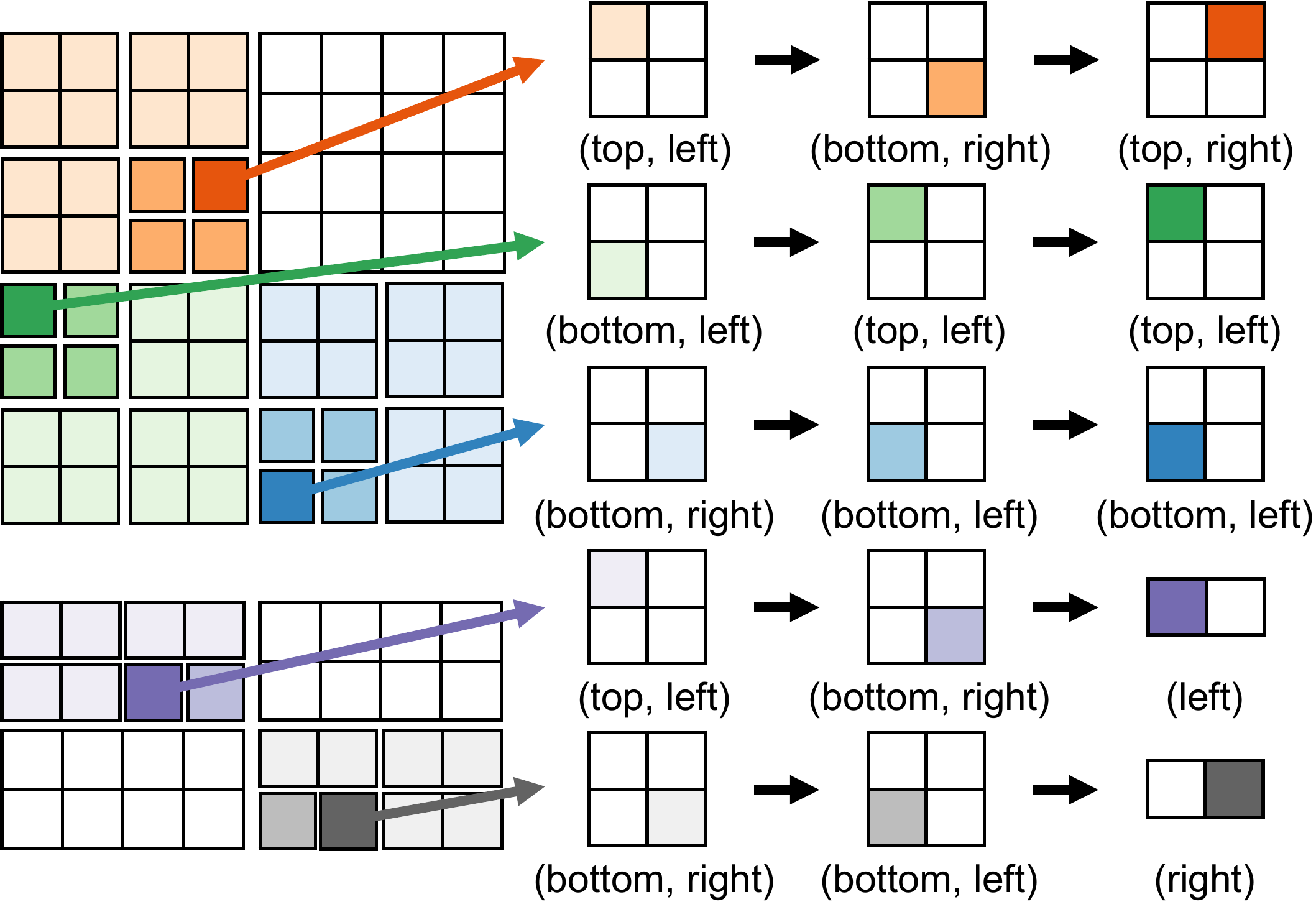}
         \label{fig:encoding}
     }
     \hspace{0.5mm}
     \subfigure[Model for Square Matrices]{
         \centering
         \includegraphics[width=0.3\linewidth]{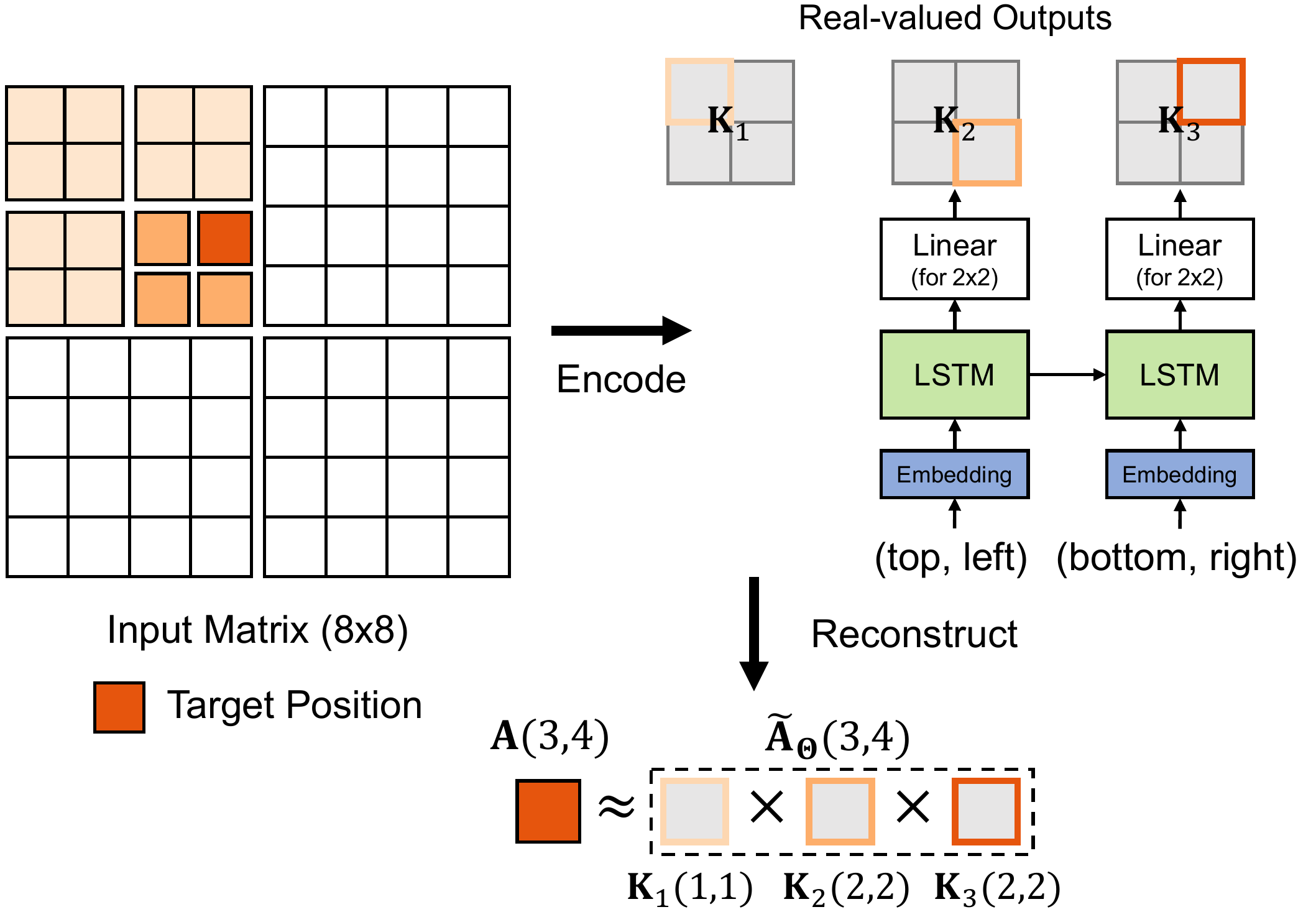}
         \label{fig:sq_model}
     }
     \subfigure[Model for Matrices of Any Size]{
         \centering
         \includegraphics[width=0.3\linewidth]{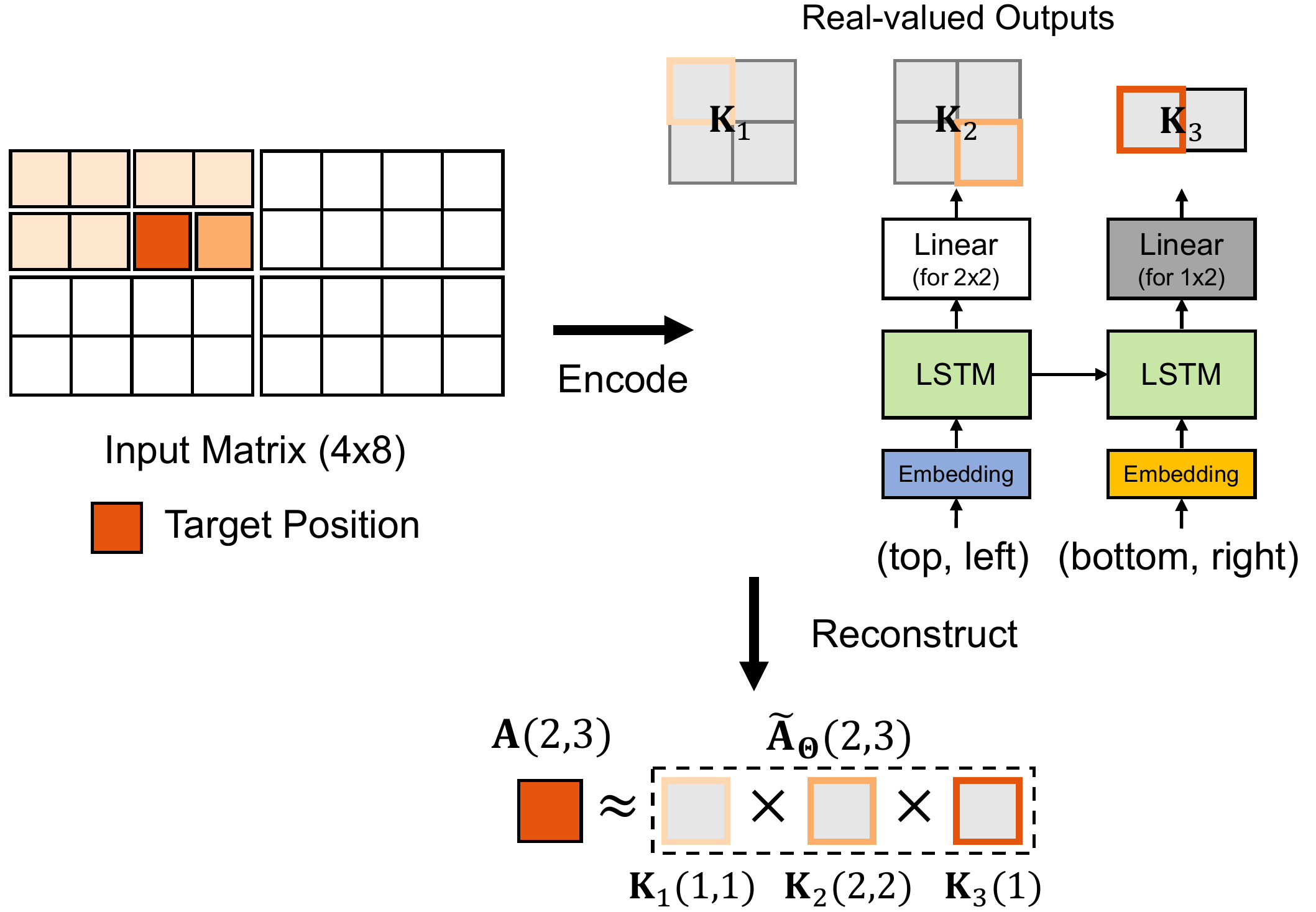}
         \label{fig:rec_model}
     }
     \vspace{-2mm}
    \caption{\underline{\smash{The overall approximation process of \method.}} It encodes the input position into a sequence by recursively dividing the input matrix. The sequence is fed into LSTM, and the outputs of LSTM are aggregated based on the Kronecker product.}
    \label{fig:model}
\end{figure*}
In this section, we present \method, a constant-space lossy compression method for sparse reorderable matrices and tensors. 
We first describe its neural network model and then the training strategies for it. After that, we analyze the computational complexity of \method.
For ease of explanation, we assume that the input is a matrix through the section, and then we describe the extensions for tensors in Section \ref{sec:tensor}.
\vspace{-2.5mm}
\subsection{Model}

\subsubsection{Overview}
\label{sec:model:overview}

When designing a neural network model $\model$ for \method, we aim to achieve the following goals: 
\begin{itemize}[leftmargin=*]
    \item \textbf{G1. Constant Size:} The number of parameters of the model should be constant, regardless of the size of the input matrix.
    \item \textbf{G2. Exploitation of Sparsity:}
    It should be possible to fit the model to the input by accessing only non-zero entries. 
    \item \textbf{G3. Fast Approximation:}
    From the trained model, it should be possible to approximate each entry of the input matrix in sublinear time (preferably, in \kijung{constant} or logarithmic time). 
\end{itemize}

For \textbf{G1}, given a matrix $\bA$ to be compressed, we encode the position $(i,j)$ of each entry $\aij$ as a sequence and use an auto-regressive sequence model, \red{specifically LSTM~\cite{hochreiter1997long}, which has a constant number of parameters,} to process the sequence.
\red{For our purpose, LSTM performs similarly with GRU~\cite{cho2014learning} and outperforms the decoder layer of Transformer~\cite{vaswani2017attention}, as shown empirically in \cite{appendix}.}
For an entry $\aij$, the sequence encoding the position $(i, j)$ is fed into LSTM, and the outputs of LSTM are combined for its approximation $\atij$  in logarithmic time, achieving \textbf{G3} (see Theorem~\ref{theorem:query} in Section~\ref{sec:method:analysis}).
Moreover, regarding \textbf{G2}, the outputs of LSTM are combined so that the sparsity can be exploited for efficient computation of the objective and its gradient (see Section~\ref{sec:method:train:model}). 
The details of encoding inputs and combining outputs are described in the following subsections.
\kijung{Regarding \textbf{G3}, it should be noticed that many factorization-based methods approximate each entry even in constant time (see Table~\ref{tab:proscons}).}

\subsubsection{Encoding inputs (lines~\ref{algo:model:encoding:start}-\ref{algo:model:encoding:end} of Algorithm~\ref{algo:model})}
\label{sec:model:input}

For simplicity, we assume an input matrix $\bA\!\in \mathbb{R}^{N \times M}$ where $N\!= \! M\!=\! 2^l$ (see Section~\ref{sec:model:general} for generalization to matrices of any size).
Algorithm \ref{algo:model} depicts how \method approximates such $\bA$.

For each entry $\aij$ of $\bA$, \method encodes its position $(i, j)$ in a sequence of length $l\!=\!\log_2\! M$ 
by recursively subdividing $\bA$ in a top-down manner. 
\method first chooses the partition where $\aij$ lies when $\bA$ is divided into $2 \times 2$ partitions of the same size (i.e., $2^{l-1} \times 2^{l-1}$). Each division gives four partitions at top left (\textt{TL}), top right (\textt{TR}), bottom left (\textt{BL}), and bottom right (\textt{BR}).
Then, \method repeats the process on the chosen partition until only the target entry $\aij$ is left.
The sequence of the positions of the chosen partition is used to encode $\aij$. In our implementation, each entry of the sequence, which is a position, is converted into a tuple in $\{1, 2\} \times \{1, 2\}$. 
Specifically, the $k$-th entry of the sequence that encodes the position $(i,j)$ is $(t(i,k),t(j,k))$ where
\begin{equation}
t(i, k) := \left(\floor{\frac{(i-1)}{2^{l-k}}} \text{ mod } 2\right) + 1. \label{eq:position}
\end{equation}

\definecolor{exstep1}{HTML}{FDD0AE}
\definecolor{exstep2}{HTML}{FDAE6B}
\definecolor{exstep3}{HTML}{E6550D}

\begin{example}[Encoding in Square Matrices]
\label{example:square}
Suppose we encode the position $(3, 4)$ of the square matrix in Figure~\ref{fig:model}(a), where $l=3$.
The position $(3, 4)$ is located in the \textcolor{exstep1}{top-left} partition of the input matrix, and it is located in the \textcolor{exstep2}{bottom-right} part of the chosen partition. 
Lastly, the position $(3, 4)$ is located at the \textcolor{exstep3}{top-right} one of the lastly chosen partition. 
Thus, the position $(3, 4)$ is encoded in the sequence \textcolor{exstep1}{\textt{TL}}\! $\rightarrow$\! \textcolor{exstep2}{\textt{BR}} \!$\rightarrow$\! \textcolor{exstep3}{\textt{TR}}, which becomes
$(1, 1)\!\rightarrow\!(2, 2)\!\rightarrow\!(1, 2)$ based on $t$ (Eq.~\eqref{eq:position}).
\end{example}

Each tuple in the sequence, except for the last one, goes through an embedding layer (line~\ref{algo:model:encoding:embedding}) to be converted into a corresponding embedded vector of size $h$, where $h$ is a hyperparameter. Then, the vector is fed into LSTM (line~\ref{algo:model:encoding:end}).

\subsubsection{Handling outputs (lines~\ref{algo:model:output:start}-\ref{algo:model:output:end} of Algorithm~\ref{algo:model})}
\label{sec:model:output}

Below, we present how \method produces an approximation. See Figure~\ref{fig:model}(b) for a pictorial description.
We again assume an input matrix $\bA\in \mathbb{R}^{N \times N}$ where $N = M = 2^l$ for ease of explanation. 
Given the position $(i,j)$ of a target entry $\aij$, \method creates $\bK_1\in \mathbb{R}^{2 \times 2}$, $\cdots$, $\bK_{l}\in \mathbb{R}^{2 \times 2}$.
Specifically, given the sequence of tuples that encode $(i,j)$  (see Section~\ref{sec:model:input} for encoding),
for each $k\in [l-1]$, the $k$-th LSTM cell receives the embedding of the $k$-th tuple, and then the hidden state of the cell goes through the linear layer and the Softplus activation to produce $\mathbf{K}_{k+1}$ (line~\ref{algo:model:output:matrix}).
The entries of $\mathbf{K}_1$ are separate learnable parameters.
The approximation $\atij$ is computed from the $(i,j)$-th entry of their Kronecker product $\mathbf{K}_1 \otimes \cdots \otimes \mathbf{K}_{l}$ as follows (line \ref{algo:model:output:end}):
\begin{equation}
    \atij := \sqrt{q} \cdot \prod\nolimits_{k=1}^{l} \mathbf{K}_{k} (t(i, k), t(j, k)) / \fnorm{\mathbf{K}_{k}}, \label{eq:decode}
\end{equation}
where $\prod_{k=1}^{l} \mathbf{K}_{k} (t(i, k), t(j, k))$ is the $(i,j)$-th entry of the Kronecker product, and $q$ is a learnable parameter.
It should be noticed that the entire Kronecker product does not have to be computed. By combining the outputs of LSTM using Eq.\eqref{eq:decode}, \textbf{G2} in Section~\ref{sec:model:overview} can be achieved. Specifically, using Eq.\eqref{eq:decode} enables the exploitation of the sparsity of the input matrix $\bA$ for linear-time training, as described in detail in Section~\ref{sec:method:train:model} (see Lemma~\ref{lemma:sq_sum}).

\subsubsection{Handling matrices of any size}
\label{sec:model:general}
Below, we describe how the above processes of \method are generalize to compress a matrix of any size.
For a given matrix $A\in \mathbb{R}^{N \times M}$, we consider integers $l_{\text{row}}$ and $l_{\text{col}}$ such that $2^{l_{\text{row}}}\geq N$ and $2^{l_{\text{col}}} \geq M$. 
Then, $A\in \mathbb{R}^{N \times M}$ is extended to the $2^{l_{\text{row}}}$-by-$2^{l_{\text{col}}}$ matrix with additional rows and columns filled with zeros.
Specifically, \method sets $l_{\text{row}}$ to $\lceil \log_2 N \rceil$ and set $l_{\text{col}}$  to $\lceil \log_2 M \rceil$ so that the number of new entries is minimized.

Without loss of generality, we assume $N \leq M$ and thus $l_{\text{row}}\!\leq \!l_{\text{col}}$. If $l_{\text{row}}\!=\!l_{\text{col}}$, the extended square matrix is considered as the input and processed as described in Sections~\ref{sec:model:input} and \ref{sec:model:output}.
Otherwise (i.e., if $l_{\text{row}}\!<\!l_{\text{col}}$), to encode the position $(i,j)$ of a target entry $\aij$, \method first recursively divides $\bA$ into $2 \times 2$ partitions, $l_{\text{row}}$ times, to obtain a partition has a size of $1 \times 2^{l_{\text{col}} - l_{\text{row}}}$, and then it recursively divides the partition into two partitions of the same size (i.e., $1 \times 2$), $l_{\text{col}} - l_{\text{row}}$ times. Each division gives two partitions at left (\textt{L}) and right (\textt{R}).
Specifically, the $k$-th entry of the sequence that encodes the position $(i,j)$ is $(t_{\text{row}}(i,k),t_{\text{col}}(j,k))$, where $\forall d\in\{\text{row},\text{col}\}$, 
\begin{align}
    t_d(i, k) = \begin{cases}
			\left(\floor{{(i-1)}/{2^{l_{d}-k}}} \text{ mod } 2\right) + 1, & \text{if $k \leq l_d$,}\\
            0, & \text{otherwise}.
		 \end{cases} 
		 \label{eq:position:non-square}
\end{align}

\begin{algorithm}[t]
\caption{Approximation process of \method for 
an $N$-by-$N(=2^l)$ matrix $\bA$ 
}\label{algo:model}
\SetKwInput{KwInput}{Input}
\SetKwInput{KwOutput}{Output}
\KwInput{(a) a position: $(i, j) \in [N] \times [N]$ where 
$N=2^l$ \\\\
(b) parameters of \texttt{Embedding}, \texttt{LSTM}, and the linear layer ($\mathbf{W}$, $\mathbf{b}$) \\
\blue{(c) scale parameter $q$ and the first matrix of Kronecker products $\mathbf{K}_1$}}
\KwOutput{an approximation $\atij$ of $\aij$, which is the $(i,j)$-th entry of the input matrix 
$\mathbf{A}\in\mathbb{R}^{N \times N}$}

 \For{$k$ $\leftarrow$ $1$ \textnormal{to} $l$}{\label{algo:model:encoding:start}
   $\mathbf{x}_k$ $\leftarrow$ $\texttt{Embedding}\big(t(i, k), t(j, k)\big)$ \label{algo:model:encoding:embedding} 
   \Comment*[f]{Sect. \ref{sec:model:input}} \\
 }
 $\mathbf{y}_2, \cdots, \mathbf{y}_l$ $\leftarrow$ $\texttt{LSTM}(\mathbf{x}_1, \mathbf{x}_2, \cdots, \mathbf{x}_{l-1})$ \label{algo:model:encoding:end}\\
 \For{$k$ $\leftarrow$ $2$ \textnormal{to} $l$}{\label{algo:model:output:start}
   $\mathbf{K}_k$ $\leftarrow$ $\texttt{Softplus}( \mathbf{W}\mathbf{y}_k + \mathbf{b})$
   \Comment*[f]{Sect. \ref{sec:model:output}} \\ \label{algo:model:output:matrix}
   }
 \Return $\atij \leftarrow$ $\sqrt{q} \cdot \prod_{k=1}^{l} \mathbf{K}_{k} \big(t(i, k), t(j, k)\big) / \fnorm{\mathbf{K}_{k}}$ \label{algo:model:output:end}
\end{algorithm}

\begin{example}[Encoding in Rectangular Matrices]
\label{example:rectangle}
Suppose we encode the position $(2, 3)$ of the non-square matrix in Figure \ref{fig:model}(a), where $(l_{\text{row}}, l_{\text{col}}) = (2, 3)$.
The position $(2, 3)$ is located in the \textcolor{exstep1}{top-left} partition and in the \textcolor{exstep2}{bottom-right} partition, respectively, in the first two divisions.
In the last division, the position $(2, 3)$ is located in the \textcolor{exstep3}{left} one.
Thus, the position $(2, 3)$ is encoded in the sequence \textcolor{exstep1}{\texttt{TL}} $\rightarrow$ \textcolor{exstep2}{\texttt{BR}} $\rightarrow$ \textcolor{exstep3}{\texttt{L}}, which becomes
$(1, 1)\rightarrow (2, 2)\rightarrow (0, 1)$ based on $t_d$ (Eq.~\eqref{eq:position:non-square}).
\end{example}

As in Section~\ref{sec:model:output}, \method produces an approximation of $\aij$ using the $(i,j)$-th entry of \kijung{the modified Kronecker product in Eq.~\eqref{eq:decode}} $\mathbf{K}_1 \otimes \cdots \otimes \mathbf{K}_{l_{\text{col}}}$.
The only difference is that $\mathbf{K}_{l_{\text{row}}+1},\cdots, \mathbf{K}_{l_{\text{col}}}$ are matrices of size $1 \times 2$, and for them, a separate embedding and linear layers are used, as described in Figure \ref{fig:model}(c).


\subsubsection{Comparison with Kronecker Graphs}
\label{sec:model:kron}
Our model $\model$ generalizes the Kronecker graph model \cite{leskovec2007scalable, leskovec2010kronecker} in two ways:
\begin{itemize}[leftmargin=*]
    \item While the Kronecker graph model uses the power of a single seed matrix, $\model$ uses the Kronecker product of potentially different matrices (i.e., $\mathbf{K}_1, \cdots, \mathbf{K}_l$) for approximation.
    \item  
    \kijung{In} $\model$, the matrices $\mathbf{K}_1, \cdots, \mathbf{K}_l$ may vary depending on the position of the target entry to be approximated. 
    Specifically, $\atij$ 
    is computed using the $(i,j)$-th entry of $\mathbf{K}^{(f_1(i),f_1(j))}_1 \otimes \mathbf{K}^{(f_2(i),f_2(j))}_2 \otimes \cdots \otimes \mathbf{K}^{(f_l(i),f_l(j))}_l,$
    where $f_k(i) = \lfloor (i-1)/2^{l-k}\rfloor$. 
\end{itemize}
This generalization leads to a significantly better trade-off between parameter size and approximation error in practice, as shown in Section \ref{sec:exp:model}.
\kijung{Notably, there are also two differences:}
\begin{itemize}[leftmargin=*]
    \item \kijung{While the Kronecker graph model is trained under a log-likelihood objective, $\Theta$ uses the squared Frobenius norm and normalizes the matrices to apply the tricks in Eq.~\eqref{eqn:new_loss} and Eq.~\eqref{eq:sq_sum}.}
    \item \kijung{As specified in Eq.~\eqref{eq:decode}, each matrix (i.e., $\mathbf{K}_1, \cdots, \mathbf{K}_l$) is normalized and mapped onto the unit hypersphere.}
\end{itemize}

\subsection{Training Strategies}

In this subsection, we propose novel training schemes for \method's model $\model$. 
We first present how to fit $\model$ to a given sparse reorderable matrix while exploiting its sparsity. 
Then, we present how to reorder the rows and columns of the input matrix so that $\model$ can be better fit to it.
These two steps are alternated until convergence, as described in Algorithm \ref{algo:neukron}.
Below, we assume a matrix $\bA\in \mathbb{R}^{N \times M}$ where $(N, M) =(2^{l_{\text{row}}}, 2^{l_{\text{col}}})$. As described in Section~\ref{sec:model:general}, a matrix of any size can be extended by zero-padding to satisfy this condition. We also assume $N\leq M$, without loss of generality.

\vspace{-1mm}
\subsubsection{Update of row/column orders}
\label{sec:method:perm}
It is crucial to properly order the rows and columns of a given reorderable matrix for \method's model $\model$ better fit the matrix.
This is because proper ordering reveals patterns (e.g., self-similarity and co-clusters), which $\model$ can exploit for accurate compression.

\smallsection{Overall process} 
For initialization, any co-clustering algorithms can be used. In our implementation, the matrix reordering scheme in \cite{jung2020fast} is used (see Section~\ref{sec:exp:method} for the effect of initialization).
After initialization, \method repeats (a) sampling two rows (or columns), (b) measuring the change in the approximation error (i.e., $\lVert\bA - \reA\rVert_F^2$), and (c) determining whether to swap the sampled rows (or columns) or not probabilistically using the following criterion: 
\begin{equation}
    u < \exp(-\gamma \cdot \Delta), \label{eq:accept}
\end{equation} where $u\!\sim\!U(0, 1)$, $\Delta$ is the change in the approximation error, and $\gamma>0$ is a hyperparameter that controls the probability of accepting swaps that increase the approximation error.

\smallsection{Similarity-aware sampling} 
Below, we describe how \method samples candidate pairs of rows (or columns) to be potentially swapped. Compared to a naive uniform sampling, the \kijung{proposed sampling method} has two advantages: 
\textbf{(a) effective}: it samples pairs based on the similarity of rows (or columns) so that swapping the pairs is likely to reduce the approximation error, and \textbf{(b) easy-to-parallelize}: it samples disjoint pairs, which can be processed in parallel.
The main idea is to select candidate pairs so that swapping pairs is likely to make similar rows (or columns) close to each other and thus to make them encoded in similar sequences in Section~\ref{sec:model:input}.
Below, we describe the sampling method step by step for sampling row pairs. Column pairs are sampled similarly.
\begin{itemize}[leftmargin=*]
    \item \textbf{Estimating similarity:} 
    In order to quickly estimate the similarity, min-hashing \cite{broder2000min} is used.
    Specifically, for a uniform random bijective function 
    $h_{\text{col}}: [M] \rightarrow [M]$ for the columns, the shingle $\min_{a_{ij} \neq 0}(h_{\text{col}}(j))$ of each $i$-th row is computed. 
    It can be shown that two rows have the same shingle with probability proportional to the Jaccard similarity of the column indices of their non-zeros \cite{broder2000min}. 
    
    \item \textbf{Locating similar rows/cols nearby:}
    We match rows with the same shingle disjointly, and for each matched rows, we sample pairs of rows to be swapped so that they are located in \textit{nearby positions}, which we define as positions whose binary representations differ in only $1$ bit.
    Let $p(i,k)$ be the position whose binary representation differs with that of $i$ only in the $k$-th bit. 
    Specifically, if two rows in the $i_1$-th and $i_2$-th positions are matched, we sample ($i_1$,$p(i_2,k)$) and ($i_2$,$p(i_1,k)$) so that $i_1$ and $i_2$ become nearby after swaps.
    The position $k\in[l_{\text{col}}]$ is sampled probabilistically (see Appendix~\ref{app:code} for details).

    
    \item \textbf{Pairing unmatched rows:} 
    The rows remaining unmatched are randomly matched, and for each matched rows, we sample pairs as described above.
    
\end{itemize}
We describe the entire process of reordering for rows in Algorithm~\ref{algo:update_perm}.

\begin{algorithm}[t]
\caption{Overall training process of \method}\label{algo:neukron}
\SetKwInput{KwInput}{Input}
\SetKwInput{KwOutput}{Output}
\KwInput{(a) a sparse reorderable matrix $\mathbf{A}$ \\
\qquad \quad (b) a number $T_p$ of permutation updates\\
}
\KwOutput{a \method model $\model$}
 Initialize $\model$ \\
 \While{not converged}{
    \For{$k$ $\leftarrow$ $1$ \textnormal{to} $T_p$}{
        $\mathbf{A}$ $\leftarrow$ \textsc{UpdateRowOrder}($\mathbf{A}$) \Comment*[f]{Sect. \ref{sec:method:perm}}\\
        $\mathbf{A}$ $\leftarrow$ \textsc{UpdateColOrder}($\mathbf{A}$) \Comment*[f]{Sect. \ref{sec:method:perm}} \\
    }
    $\model$ $\leftarrow$ \textsc{UpdateModel}($\mathbf{A}$, $\model$)  \Comment*[f]{Sect. \ref{sec:method:train:model}}
 }
 \Return{$\model$}
\end{algorithm}

\subsubsection{Update of model parameters}
\label{sec:method:train:model}

The objective function of optimization is $\lVert \bA - \reA \rVert_F^2$, as in Problem~\ref{prob}. 
Naively computing it takes $\Omega(NM\log{M})$ time since all $NM$ entries are approximated and approximating each entry takes $\Theta(\log{M})$ time (see Theorem~\ref{theorem:query} in Section~\ref{sec:method:analysis}).

For its efficient computation, we reformulate the error as
\vspace{-1mm}
\begin{align} \label{eqn:new_loss}
    &\fnorm{\mathbf{A} - \reA}^2 = \sum_{i=1}^{N} \sum_{j=1}^M (a_{ij} - \atij)^2
     = \sum_{a_{ij} \neq 0}(a_{ij} - \atij)^2  \\
     &+ \sum_{a_{ij}=0} \atij^2 \nonumber = \sum_{a_{ij} \neq 0}((a_{ij} - \atij)^2 - \atij^2) + 
     \sum_{i=1}^N\sum_{j=1}^M \atij^2.
\end{align}

In our model $\model$, the last term, (i.e., the sum of  squares) can be immediately computed from a learnable parameter $q\in \mathbb{R}^+$ (which is used in Eq.~\eqref{eq:decode}), as formalized in Lemma~\ref{lemma:sq_sum}.
\vspace{-1mm}
\begin{lemma} \label{lemma:sq_sum}
    For approximation by  Eq.~\eqref{eq:decode},
    Eq.~\eqref{eq:sq_sum} always holds.
    \begin{equation}
        \sum_{i=1}^{2^{l_{\text{row}}}}\sum_{j=1}^{2^{l_{\text{col}}}} \atij^2=q^{l_{col}} \label{eq:sq_sum}
    \end{equation}
\end{lemma}
\vspace{-2mm}
\begin{proof}
To prove this lemma, we use an induction. For $(l_{\text{row}}, l_{\text{col}}) \!=\! (1, 1)$ and $(l_{\text{row}}, l_{\text{col}}) \!=\! (0, 1)$, the statement holds trivially. 
Suppose the statement holds when $(l_{\text{row}}, l_{\text{col}}) \!=\! (0, l_2)$. For $(l_{\text{row}}, l_{\text{col}}) \!=\! (0, l_2 + 1)$, the statement also holds since
\begin{align*}
    \sum_{i=1}^{2^{l_{\text{row}}}}\sum_{j=1}^{2^{l_2 + 1}} \tilde{a}_{ij}^2 &= \frac{q\bK_{1}(1, 1)^2}{\fnorm{\bK_{1}}^2} \sum_{i=1}^{2^{l_{\text{row}}}}\sum_{j=1}^{2^{l_2}}  \frac{\tilde{a}_{ij}^2}{q\mathbf{K}_{1}(1, 1)^2/\fnorm{\bK_{1}}^2} \\ &\quad+ \frac{q\mathbf{K}_{1}(1, 2)^2}{\fnorm{\bK_{1}}^2} \sum_{i=1}^{2^{l_{\text{row}}}}\sum_{j=2^{l_2} + 1}^{2^{l_2 + 1}} \frac{\tilde{a}_{ij}^2}{q\mathbf{K}_{1}(1, 2)^2/\fnorm{\bK_{1}}^2} \\
    &= q^{l_2} \left(\frac{q\bK_{1}(1, 1)^2}{\fnorm{\bK_{1}}^2} + \frac{q\bK_{1}(1, 2)^2}{\fnorm{\bK_{1}}^2}\right) = q^{l_2} \cdot q = q^{l_2+1}
\end{align*} 

Similarly, if the statement holds for $(l_{\text{row}}, l_{\text{col}}) = (l_1, l_2)$ and $l_1 \leq l_2$, the statement also holds for $(l_{\text{row}}, l_{\text{col}}) = (l_1 + 1, l_2 + 1)$.
By induction, the statement holds for all $0 \leq l_{\text{row}} \leq l_{\text{col}}$.
\end{proof}
This property follows from our careful design of Eq.~\eqref{eq:decode}, which is based on the Kronecker product. While $q$ can be set so that the square sum of entries of $\reA$ is equal to that of $\bA$, making it learnable leads to better compression since this gives more degrees of freedom to the model (see Section~\ref{sec:exp:method}).
As a result, the error becomes $\sum_{a_{ij} \neq 0}((a_{ij} - \atij)^2 - \atij^2) + q^{l_{col}}$, and thus the error and its gradient can be computed in time proportional to the number of non-zeros, without having to approximate zero entries in $\bA$ \blue{explicitly} (see Theorem~\ref{thm:time} in Section~\ref{sec:method:analysis}).
\blue{It should be noticed that we do use the loss function that encourages the model to fit all entries including zeros, and we speed up its computation without changing it.}
Gradient descent is used for updating the model parameters.



\smallsection{Implementation in practice} 
\blue{Since candidate pairs are disjoint, processing them, including computing Eq.~\eqref{eq:accept}, is performed in parallel in our implementation. 
Shingles are also computed in parallel.}

\vspace{-2mm}
\subsection{Theoretical Analysis} \label{sec:method:analysis}
We analyze the time and space complexity of \method. We assume that (a) $N\leq M$ for the input matrix $\bA\in \mathbb{R}^{N \times M}$ and (b) the dimension $h$ of LSTM is a constant (i.e., $O(1)$), which is a user-defined hyperparameter.
\method requires logarithmic time for approximation (Theorem~\ref{theorem:query}), \red{as confirmed empirically in Section~\change{2} of~\cite{appendix}.}
For training, it requires time proportional to the number of non-zero entries of $\bA$, denoted by $\text{nnz}(\mathbf{A})$ (Theorem~\ref{thm:time}).

\begin{theorem}[Approximation Time for Each Entry]
\label{theorem:query}
 The approximation of each entry by \method takes $\Theta(\log M)$ time.    
\end{theorem}
\begin{proof}
First, we need to encode the position of the given entry. Since we need the subdivision $\Theta(\log M)$ times, the time complexity of the encoding step is $\Theta(\log M)$. The computational cost to approximate an entry only depends on the length of the input of the LSTM, so the time complexity for inference is $\Theta(\log M)$. 
\end{proof}

\begin{theorem}[Training Time] \label{thm:time}
Each training epoch in \method takes $O(\text{nnz}(\mathbf{A}) \cdot \log{M})$ time.
\end{theorem}
\begin{proof}
The time complexity for inference is $O(\log M)$ for each input. Thus, computing the approximation error takes $O(\text{nnz}(\mathbf{A}) \cdot \log M)$ with Eq.~\eqref{eqn:new_loss} (see Lemma~\ref{lemma:sq_sum}). The time complexity for computing the gradients is also $O(\text{nnz}(\mathbf{A}) \cdot \log M)$, since the gradient of each component in the model, such as matrix multiplication and taking a non-linearity, does not require a greater time complexity. 
For optimizing the orders of rows and columns, computing the shingle values for rows and columns takes $O(\text{nnz}(\mathbf{A}))$ time since we need to look up all non-zero entries. 
Matching the rows and the columns as pairs requires $O(N + M)$ time.
Only the entries of the output that correspond to non-zero entries
are changed due to swaps and inference of a single element takes $O(\log M)$ time.
Thus, checking the criterion in Eq.~\eqref{eq:accept} takes $O(\text{nnz}(\mathbf{A}) \cdot \log M)$ time.
Therefore, the overall training time per epoch is $O(\text{nnz}(\mathbf{A}) \cdot \log M)$.
\end{proof}

While \method requires space proportional to the number of non-zero entries in the input matrix during training (Theorem~\ref{thm:space}), it gives a constant-size compression. (Theorem~\ref{thm:num_params}).
Refer to Appendix~\ref{app:proof} for the proofs of Theorems~\ref{thm:num_params} and \ref{thm:space}.
\begin{theorem}[Space Complexity of Outputs] \label{thm:num_params}
The number of model parameters of \method is $\Theta(1)$. 
\end{theorem}

\begin{theorem}[Space Complexity during Training] \label{thm:space}
\method requires $O(\text{nnz}(\mathbf{A})+ M)$ space during training.
\end{theorem}


\section{Extension to Tensors}
\label{sec:tensor}
We extend \method to sparse reorderable tensors. 
Theoretical analyses are available at Section~\change{3} of \cite{appendix}.
\vspace{-2mm}

\subsection{Model}
For a given $D$-order tensor $\mathcal{X} \in \mathbb{R}^{N_1 \times \cdots \times N_D}$ (we assume $N_1 \leq \cdots \leq N_D$ without loss of generality), we first compute $l_i = \ceil{\log_2 N_i}$ for each $i \in [D]$ and extend $\mathcal{X}$ to the tensor of size $2^{l_1} \times \cdots \times 2^{l_D}$ with additional entries filled with zeros.
As in Section \ref{sec:model:general}, for encoding, \method first recursively divides the extended tensor into \red{$2^D$} partitions $l_1$ times to obtain a partition has a size of $1 \times 2^{l_2 - l_1} \times \cdots \times 2^{l_D - l_1}$.
Then, it recursively divides the partition as it handles a $(D-1)$-order tensor. As a result, the $k$-th entry of the encoded sequence for the position $(i_1, \cdots i_D)$ is $(t_1(i_1, k), \cdots t_D(i_D, k))$, where $t_d$ is identical to Eq. \eqref{eq:position:non-square}. We provide an example of \method on a $3$-order tensor in Figure \ref{fig:tensor}.
%
After encoding, \method produces an approximation using the Kronecker product $\mathcal{K}_1 \otimes \cdots \otimes \mathcal{K}_{l_D}$ from $D$ linear layers for handling tensors of $D$ different sizes.

\vspace{-2mm}
\subsection{Training Strategies}

The main difference in training strategies lies in computing shingles. 
For a $D$-order tensor, $D$ random bijective functions are used, thus each mode index has $D-1$ shingles from those functions except for the function of the same mode.
In our extension, we match positions $i_1$ and $i_2$ as a pair only if the $D-1$ shingles of $i_1$ and those of index $j$ are all the same, and
\red{the orders of indices are randomly initialized.}
All other procedures are identical to the original \method.
\begin{figure}[t]
    \vspace{-2mm}
    \centering
    \includegraphics[width=0.95\linewidth]{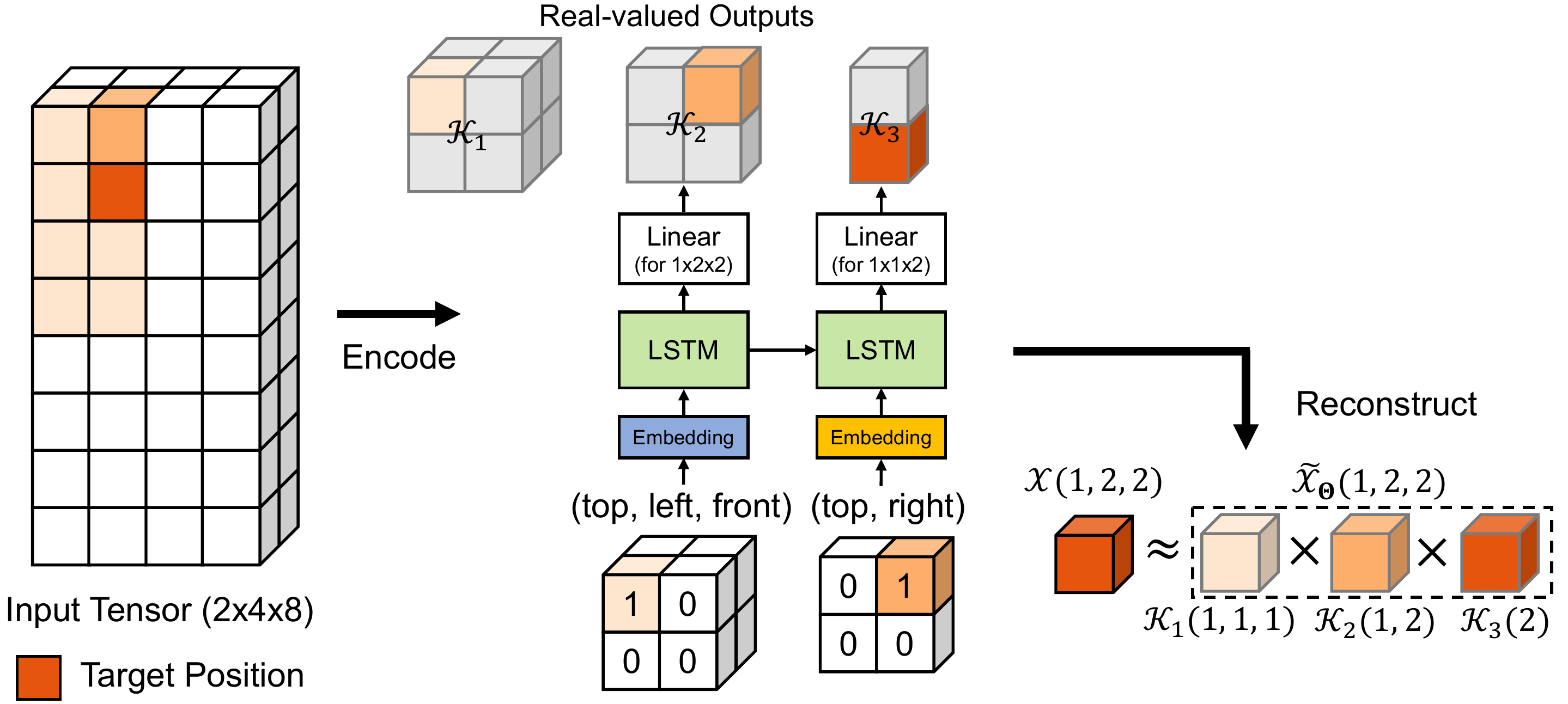}
    \caption{Example of \method on an $3$-order tensor $\mathcal{X}$.}
    \label{fig:tensor}
\end{figure}
\vspace{-1.5mm}
\section{Experiments}
\label{sec:experiments}
\noindent We conducted experiments to answer the following questions:
\vspace{-0.5mm}
\begin{enumerate}
	\item[Q1.] \textbf{Compression Performance:} 
	Does \method perform more compact and accurate compression than its best competitors?
	\item[Q2.] \textbf{Ablation Study:} 
	How effective are \method's training strategies for compression performance?
	\item[Q3.] \textbf{Scalability and Speed:} 
	Does \method scale linearly with the number of non-zero entries of input data?
	\item[Q4.] \textbf{Approximation Analysis:} 
	How does the approximation error of \method vary depending on entry values?
	{\item[Q5.] \textbf{Effects of Data Properties:} 
	How do the skewness, order, and dimension of the input 
	affect \kijung{the approximation error?}
	}
\end{enumerate}
\vspace{-0.5mm}
The answers for Q3, Q4, and Q5 are provided in Appendix~\ref{sec:supple:exp}.

\vspace{-2mm}
\subsection{Experiment Specifications}
\smallsection{Machine} We ran experiments for \method on a machine with $4$ RTX 2080Ti GPUs and 128GB RAM. 
For competitors, which do not require GPUs, we ran experiments on a desktop with a 3.8GHz AMD Ryzen 3900X CPU and 128GB RAM.
\kijung{\textbf{Note that outputs and compression ratios 
do not depend on machine specifications.}}

 \smallsection{Datasets} 
\blue{We used six real-world matrices and four real-world tensors listed in Table \ref{tab:datasets}. 
All the datasets are weighted (i.e., non-binary matrices and tensors) except for the \textt{email} and \textt{threads} datasets.}
Detailed semantics \kijung{and structural properties} of the datasets are provided in Table~3 \kijung{and Table~7} of \cite{appendix}\kijung{, respectively.}

\smallsection{Competitors} 
For matrices, we compared \method with \kronfit~\cite{leskovec2010kronecker}, T-SVD (truncated SVD), CMD~\cite{sun2007less}, ACCAMS~\cite{beutel2015accams}, CUR~\cite{drineas2006fast}, and bCCAMS~\cite{beutel2015accams}. 
In order to compress matrices of any size, 
we extended \kronfit so that it (a) fits a non-square seed matrix, (b) permutes rows and columns separately, and (c) aims to minimize the approximation error in Problem~\ref{prob}.
\blue{We did not consider methods designed for unipartite and/or unweighted graphs (e.g., \cite{lee2020ssumm,lefevre2010grass,riondato2017graph}) as competitors since they are not applicable to most of the datasets.}
For tensors, we compared \method with CP~\cite{bader2008efficient} and Tucker~\cite{kolda2008scalable} decompositions and CSF~\cite{smith2015tensor}, which is lossless. 
\red{The competitors are described in Section~\ref{sec:related}, and see \cite{appendix} for implementation details.}

\smallsection{Experimental Setup}
We trained \method and its competitors 
under the following stopping condition with the patience of $100$ epochs: $\frac{\mathcal{E}_\text{min} - \mathcal{E}_\text{curr}}{\mathcal{E}_\text{min}} < 10^{-5}$,
%
%
where $\mathcal{E}_\text{min}$ is the lowest approximation error so far, 
and $\mathcal{E}_\text{curr}$ is the current approximation error.
For all experiments, we set $T_p$ in Algorithm~\ref{algo:neukron} to $2$, and set $\gamma$ in Eq.~\eqref{eq:accept} to $10$, after a preliminary study (see Section~\change{6} of \cite{appendix}). 
\method was trained by Adam optimizer whose learning rate was set to $10^{-3}$ for the \textt{email} and \textt{threads} datasets, and $10^{-2}$ for the others. 
Unless otherwise stated, we set the hidden dimension $h$ to $30$ in the \textt{email}, \textt{nyc}, and \textt{tky} datasets and to $60$ in the \textt{kasandr}, \textt{nips}, and \textt{threads} datasets. 
For the other datasets, we set $h$ to $90$.
We ran all experiments $5$ times with different random seeds and reported the average error.
\red{The setups for the competitors are depicted in \cite{appendix}.} 
\vspace{-2mm}

\begin{table}[t]
    \vspace{-2mm}
    \centering
    \caption{Real-world datasets used in the paper. All datasets are publicly available, and links to them are available in \cite{appendix}.} 
    \label{tab:datasets}
    \scalebox{0.75}{
    \begin{tabular}{c|r|r|r}
        \toprule
        \textbf{Type} &\textbf{Name} & \textbf{Size} & \textbf{\# of non-zeros} \\
        \midrule
        \multirow{6}{*}{Matrix} & \textt{email} 
        & $1,005 \times 25,919$ & $92,159$ \\
        & \textt{nyc} 
        & $1,083 \times 38,333$ & $91,024$ \\
        & \textt{tky} 
        & $2,293 \times 61,858$ & $211,955$ \\
        & \textt{kasandr} 
        & $414,520 \times 503,702$ & $903,366$ \\
        & \textt{threads} 
        & $176,445 \times 595,778$ & $1,457,727$ \\
        & \textt{twitch} 
        & $790,100 \times 15,524,309$ & $234,422,289$ \\
        \midrule
        \multirow{4}{*}{Tensor} & \textt{nips} 
        & $2,482 \times 2,862 \times 14,036$ & $3,101,609$ \\
        & \textt{4-gram} 
        & $48K \times 54K \times 55K \times 58K$ & $7,495,550$ \\
        & \textt{3-gram} 
        & $88K \times 100K \times 110K$ & $9,778,281$ \\
        & \textt{enron} 
        & $5,699 \times 6,066 \times 244K$ & $31,312,375$ \\
        \bottomrule
    \end{tabular}
    }
\end{table}

\begin{figure*}[ht]
    \centering
    \vspace{-2mm}
    \includegraphics[width=0.73\linewidth]{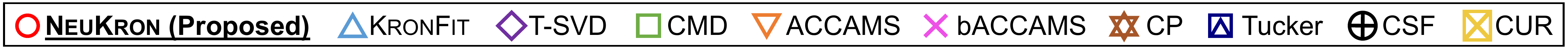} \\
    \vspace{-1.5mm}
    \subfigure[\textt{kasandr}]{
        \includegraphics[width=0.17\linewidth]{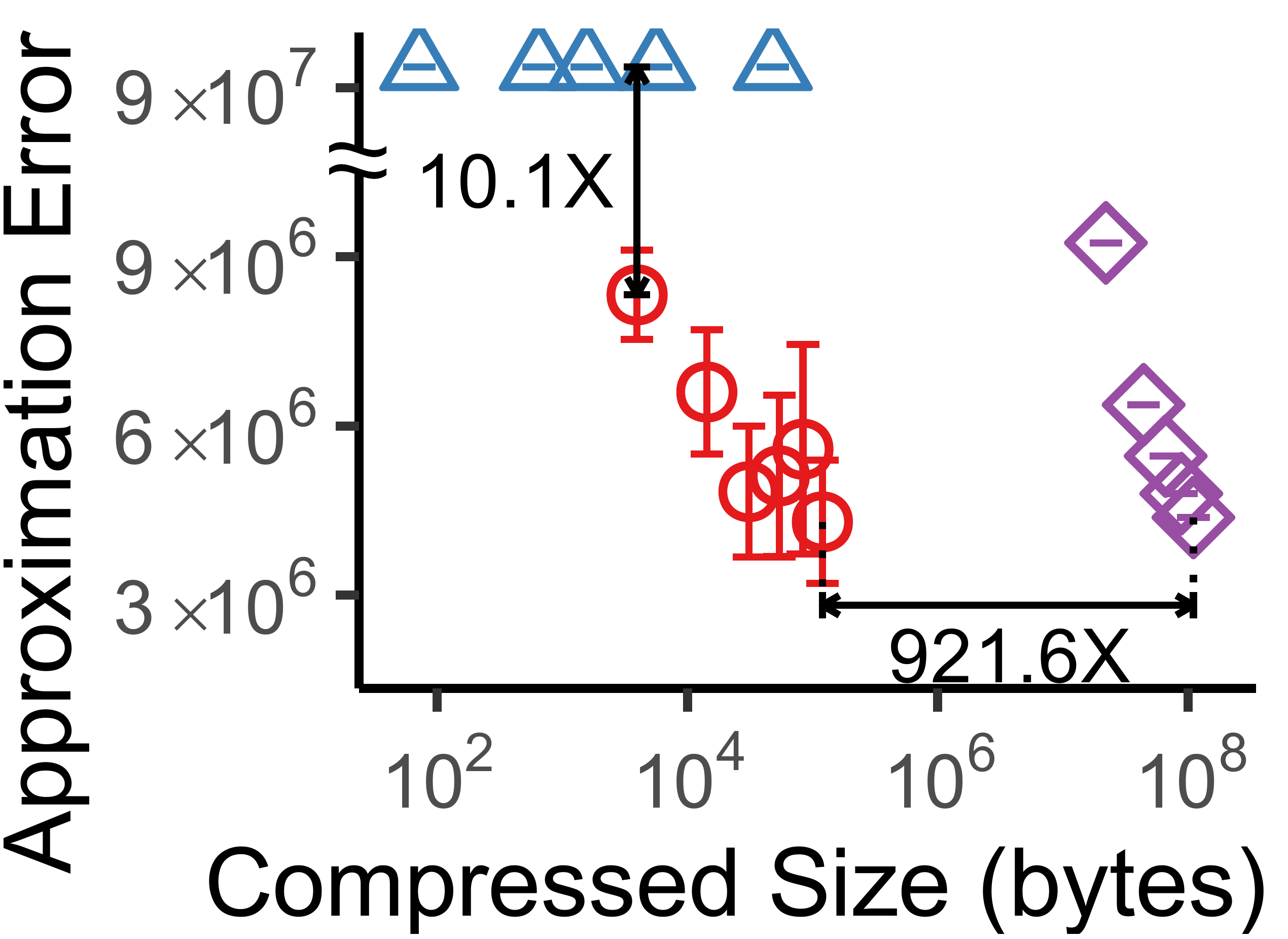}
    }
    \subfigure[\textt{twitch}]{
        \includegraphics[width=0.17\linewidth]{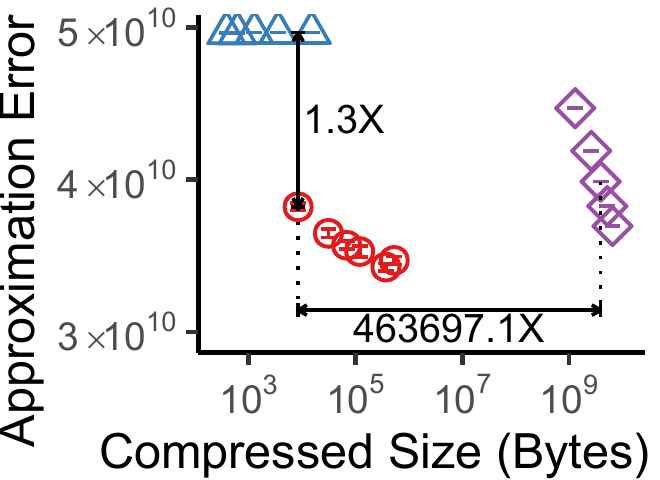}
    }
    \subfigure[\textt{email}]{
        \includegraphics[width=0.17\linewidth]{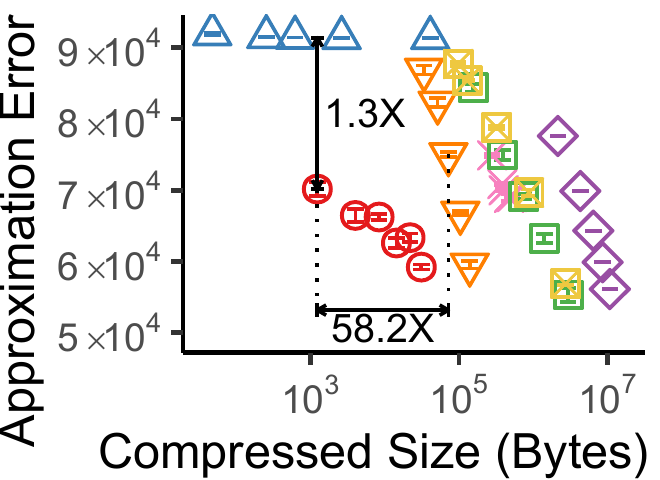}
    }
    \subfigure[\textt{nyc}]{
        \includegraphics[width=0.17\linewidth]{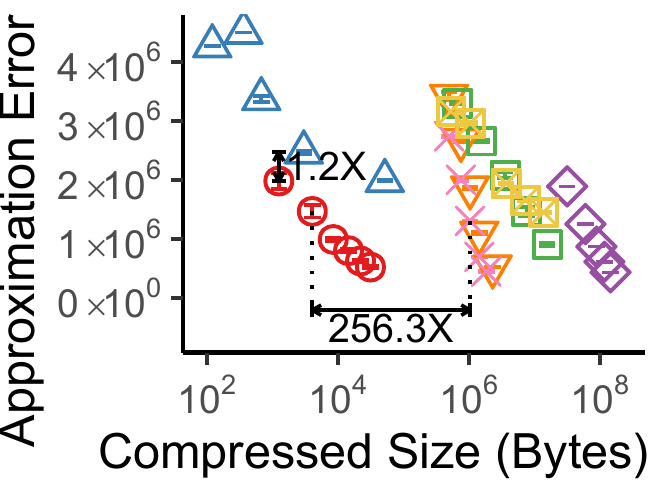}
    }
    \subfigure[\textt{tky}]{
        \includegraphics[width=0.17\linewidth]{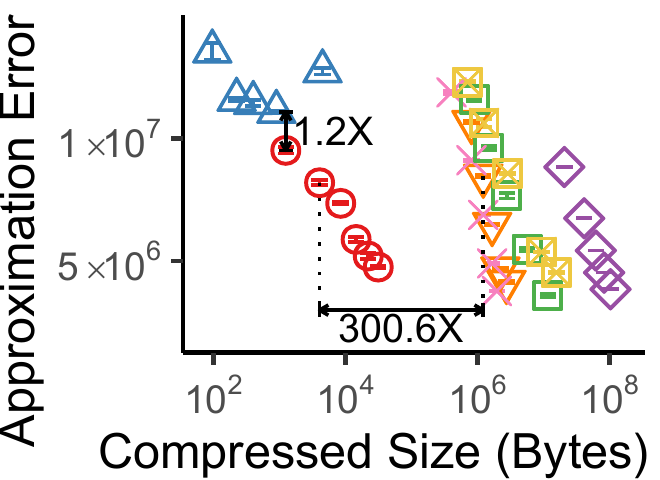}
    } \\
    \vspace{-3mm}
    \subfigure[\textt{threads}]{
        \includegraphics[width=0.17\linewidth]{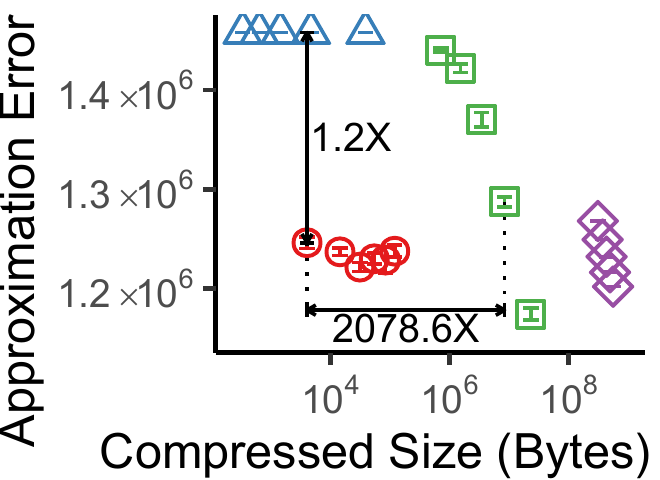}
    }
    \subfigure[\textt{nips}]{
        \includegraphics[width=0.17\linewidth]{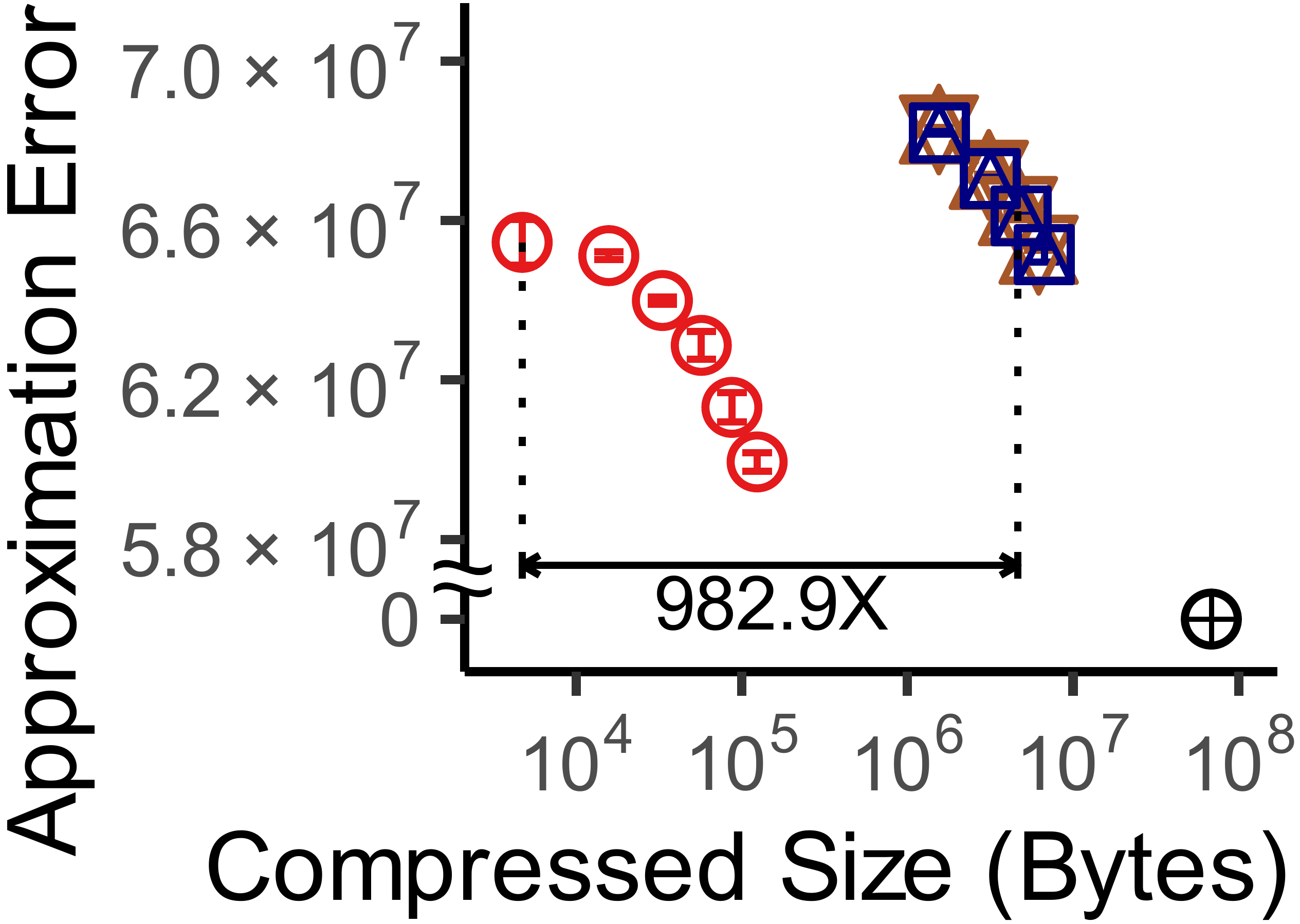}
    }
    \subfigure[\textt{enron}]{
        \includegraphics[width=0.17\linewidth]{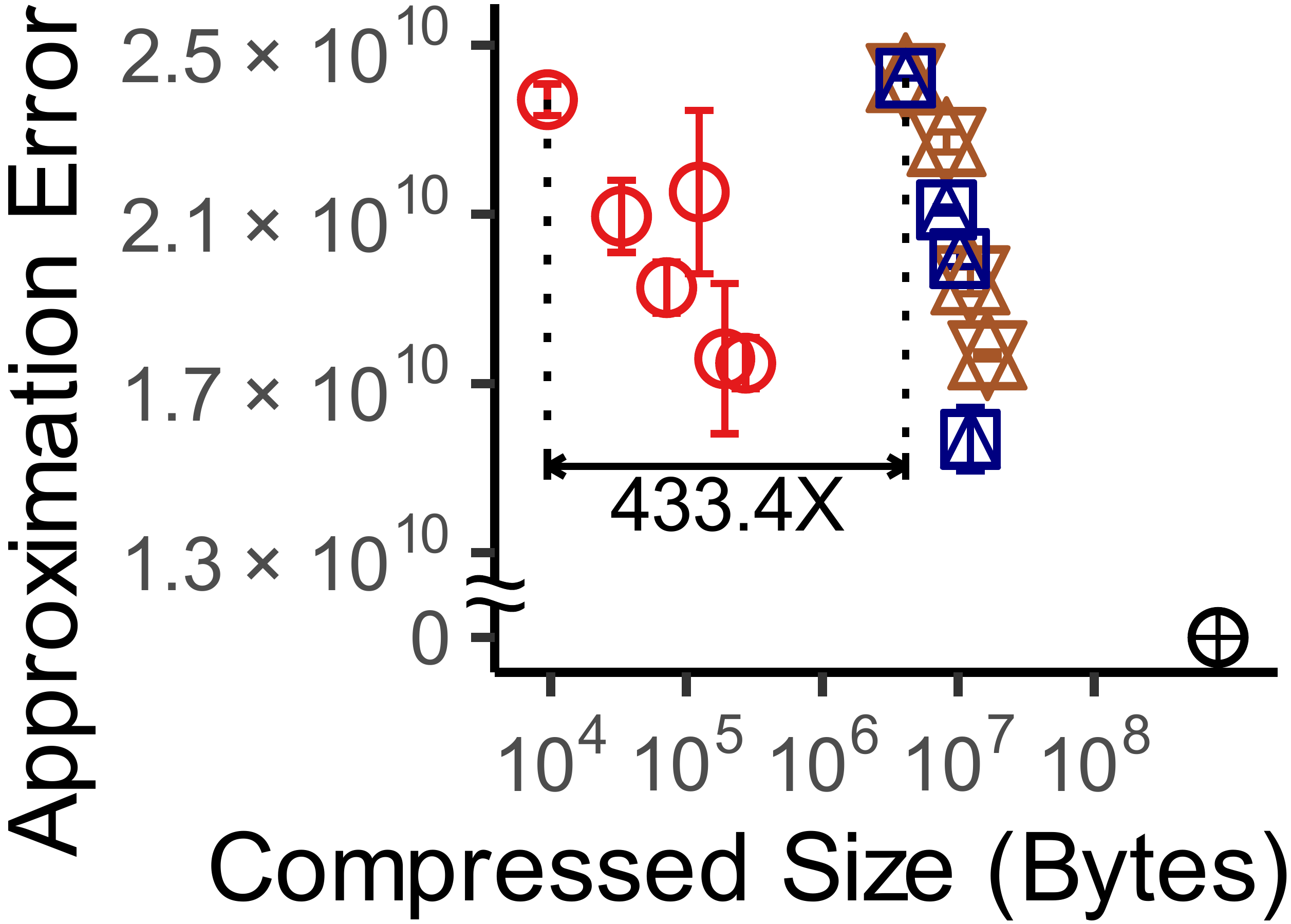}
    }
    \subfigure[\textt{3-gram}]{
        \includegraphics[width=0.17\linewidth]{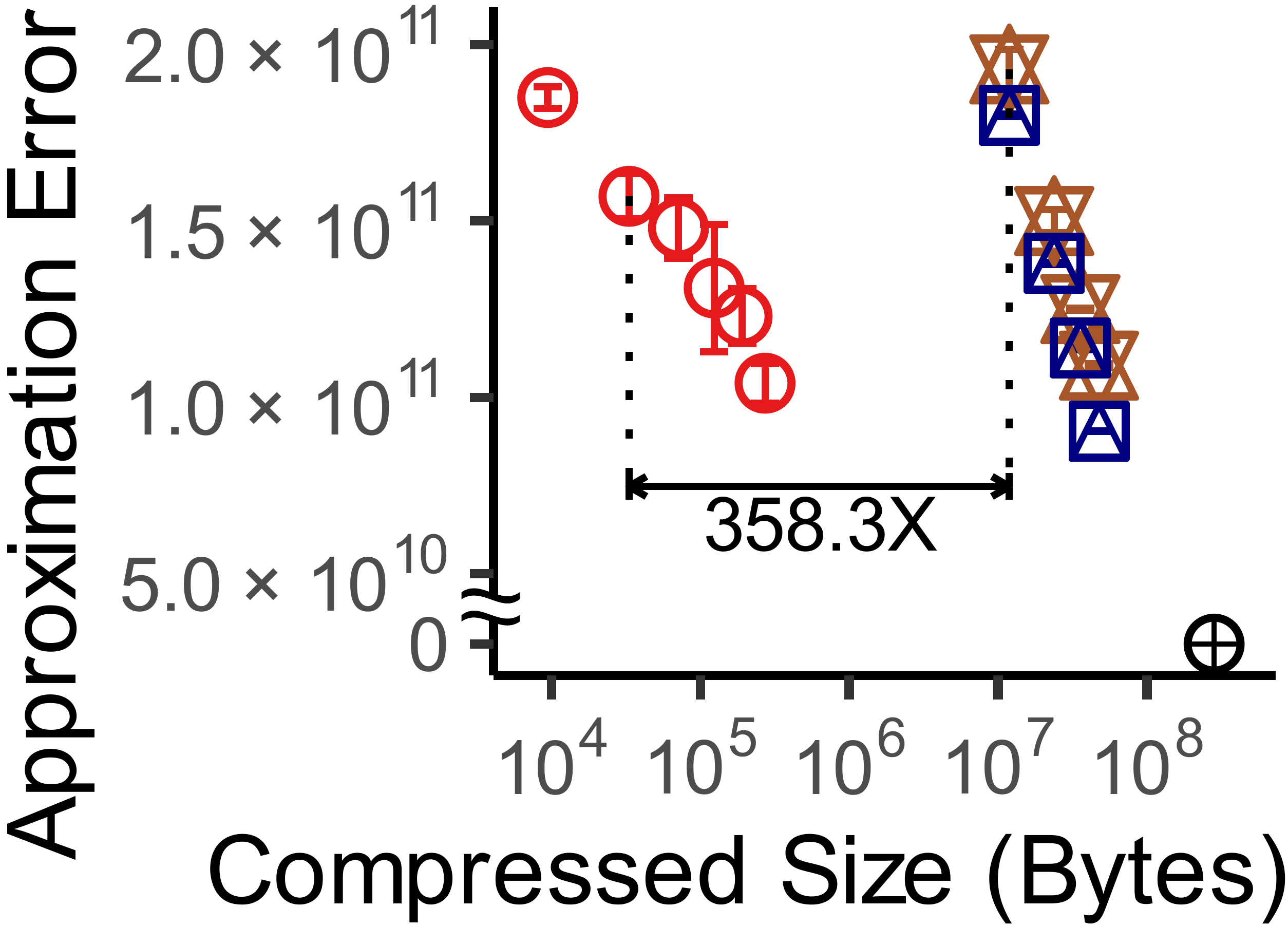}
    }
    \subfigure[\textt{4-gram}]{
        \includegraphics[width=0.17\linewidth]{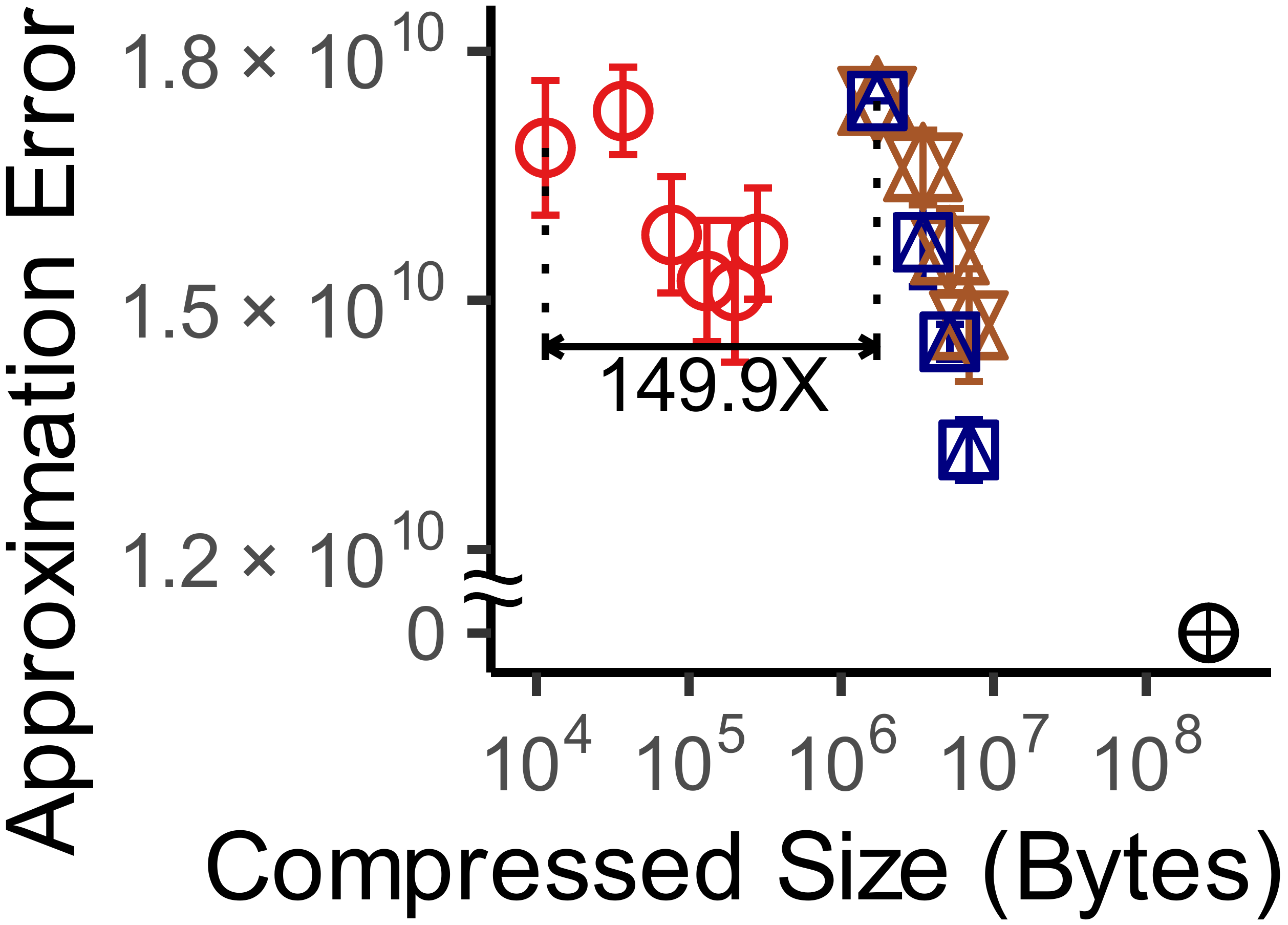}
    } \\
    \vspace{-2mm}
    \caption{\underline{\smash{\method provides concise and accurate compressions.}} The outputs of \method are up to five orders of magnitude smaller than those of the competitors when the approximation errors in them are similar.
    When the sizes of the outputs are similar, the approximation error was up to $10.1\times$ smaller in the outputs of \method than those in the competitors. 
    ACCAMS, bACCAMS, CUR, and CMD ran out of memory in some datasets, and their results do not appear in the corresponding plots. 
    \blue{Note that the errors of \kronfit do not always decrease as the number of parameters increases, as previously reported in \cite{leskovec2010kronecker}.}
    }
    \label{fig:compression}
\end{figure*}

\subsection{Q1. Compression Performance} \label{sec:exp:model}
We compared the (a) size in bytes\footnote{In our implementation, each floating-point number took $4$ bytes.} %
and (b) approximation error of the compressed output obtained by the considered algorithms.
We varied the hidden dimension $h$ of \method from $5$ to $30$ for the \textt{email}, \textt{nyc}, and \textt{tky} datasets and from $10$ to $60$ for the \textt{kasandr}, \textt{nips} and \textt{threads} datasets. For the others, we varied $h$ from $15$ to $90$. 
Similarly, we varied the hyperparameters of each competitor as
to reveal its trade-off between the size and error \blue{(refer to \cite{appendix})}.

%

\textbf{For all datasets, \method achieved the best trade-off between the approximation error and the compressed size.} As seen in Figure \ref{fig:compression}, the size was up to \textbf{five orders of magnitude smaller} in \method than in the competitors when their errors were similar. 
The error was also up to \textbf{10.1$\times$ smaller} in \method than in the competitors when the outputs were of similar size.
Note that the errors of \kronfit do not always decrease as the number of parameters increases, as previously reported in \cite{leskovec2010kronecker}. 

\smallsection{Performance on non-reorderable data}
\method can also be applied to non-reorderable data if the mapping between the original and new orders of rows and columns are stored additionally.
Even when we assume that the datasets are non-reorderable and consider the extra cost, \method gives by far the best trade-off between size and approximation error, as shown in Figure~\ref{fig:appdix:trade-off}.

\begin{figure}[t]
    \centering
    \includegraphics[width=\linewidth]{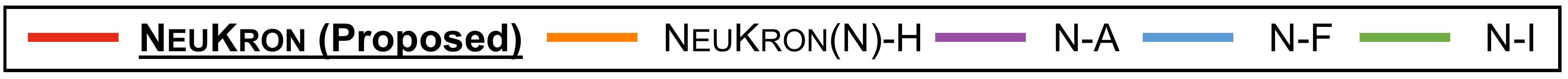} \\
    \vspace{-1mm}
    \subfigure[Effect of min-hashing and the auto-regressive architecture]{
        \includegraphics[trim=0 12.5 0 0,clip, width=0.475\linewidth]{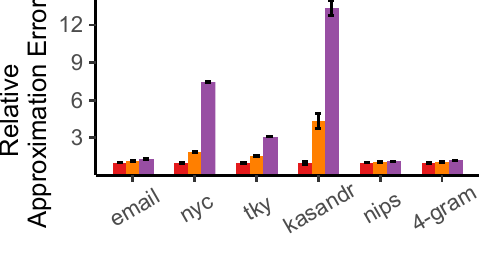}
        \label{fig:effect_init}}
    \subfigure[Effect of $q$ and initialization]{
        \includegraphics[trim=0 12.5 0 0,clip, width=0.475\linewidth]{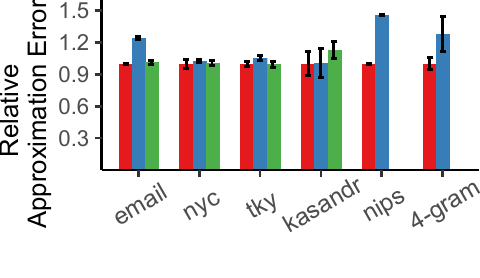}
        \label{fig:effect_T}
    }
    \vspace{-2mm}
    \caption{\label{fig:ablation}\underline{\smash{Effectiveness of the components of \method.}}
    We report the approximation errors of variants relative to that of \method.
    Results of \method-I on tensors are omitted since, for tensors, \method also randomly initializes orders.}
    \vspace{-2mm}
\end{figure}

\begin{figure}[t]
    \vspace{-3mm}
    \centering
    \includegraphics[width=\linewidth]{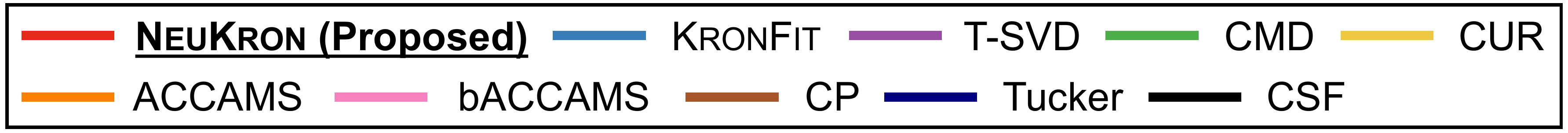} \\
    \includegraphics[width=\linewidth]{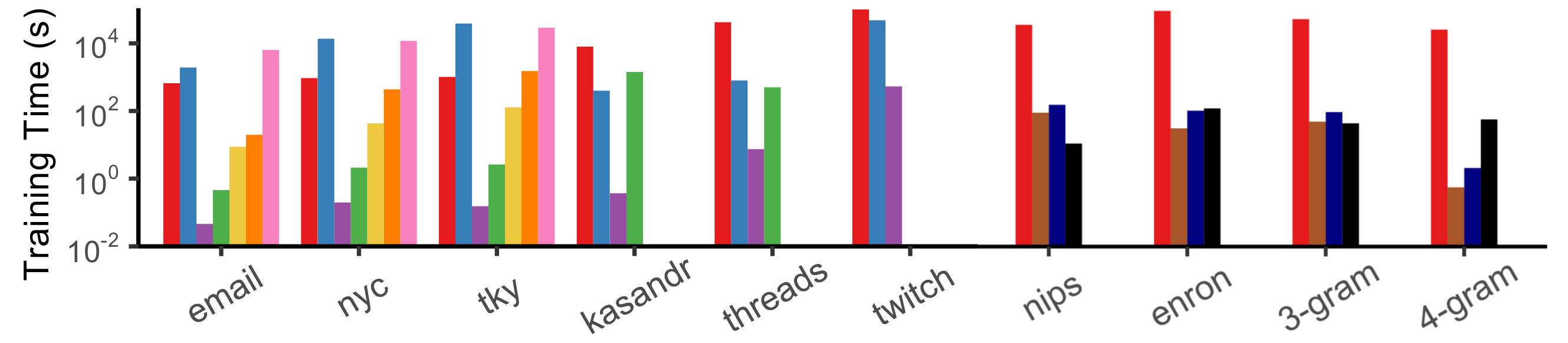}
    \vspace{-5mm}
    \caption{\label{fig:trainingtime}\kijung{\underline{\smash{Training time of \method and the competitors.}} Note that \method requires much longer training time than many competitors, while it provides the best trade-off between space and accuracy (Figure \ref{fig:compression}). See Appendix \ref{sec:exp:scalability} for detailed hyperparameter settings for each method.}}
    \vspace{-1mm}
\end{figure}

\vspace{-2mm}
\subsection{Q2. Ablation Study}
\label{sec:exp:method}
\kijung{On the four smallest matrices and the two smallest tensors,} we demonstrate the effectiveness of the components of \method illustrated in Section~\ref{sec:method} by comparing it with the following variants: 
\vspace{-3mm}
\begin{enumerate}[leftmargin=*,label=(\alph*)]
    \item \method: the proposed method with all components. 
    \item \method-H (N-H): a variant that uniformly samples pairs of rows and columns without using min-\textbf{h}ashing. 
    \item \method-I (N-I): a variant that randomly \textbf{i}nitializes the orders of rows and columns without using the scheme in \cite{jung2020fast}.
    \item \method-F (N-F): a variant that \textbf{f}ixes $q$ to the sum of the squares of all entries in the input.
    \item \method-A (N-A): a variant without any \textbf{a}uto-regressive architecture.
    It only uses two learnable matrices $\mathbf{K}_{\text{square}} \in \mathbb{R}^{2\times2}$ and  $\mathbf{K}_{\text{rect}} \in \mathbb{R}^{1\times2}$, \red{as in \kronfit}, to compute $\kpower{\mathbf{K}_{\text{square}}}{l_{\text{row}}} \otimes \kpower{\mathbf{K}_{\text{rect}}}{(l_{\text{col}}-l_{\text{row}})}$ \red{for approximation}. 
    Similarly, it uses $D$ learnable tensors to approximate  $N$-order tensors.
\end{enumerate}

As seen in Figure \ref{fig:ablation}, \method outperformed \method-A and \method-H, which indicates that the auto-regressive architecture (i.e., LSTM) and the min-hashing technique are crucial to enhance the performance of \method.
Moreover, making $q$ learnable was effective 
especially on the \textt{email}, \textt{nips}, and \textt{4-gram} datasets.
For the order initialization, \method-I showed comparable or slightly poor performance than \method, implying that how the rows and columns are initialized can affect the compression quality. 

\smallsection{Extra Results}
For details results regarding Q3-Q5, refer to Appendix~\ref{sec:supple:exp}.
A training time comparison is available in Figure~\ref{fig:trainingtime}.

\section{Conclusion}
\label{sec:conclusion}
We focus on compressing sparse reorderable matrices and tensors into a constant-size space.
Our contributions are three-fold:
\begin{itemize}[leftmargin=*]
    \item \textbf{Compact yet Accurate Method:} We proposed \method, which lossily compresses matrices and fixed-order tensors of any size with a constant number of parameters.
   \method provided an output that is up to five orders of magnitude smaller than the outputs of the best competitors when the approximation errors in them are similar (Figure~\ref{fig:compression}).  
    \item \textbf{Theoretical Analysis:} We carefully designed \method so that, for sparse reorderable matrices and fixed-order tensors of any size (a) the number of parameters is constant, (b) each entry is approximated in a logarithmic time, and (c) the model is fitted to an input in time proportional to the number of non-zero entries in it. We proved these desirable properties (Theorems~\ref{theorem:query}-\ref{thm:num_params}).
    \item \textbf{Extensive Experiments:}
    Through extensive experiments on $10$ real-world datasets, we demonstrated the effectiveness and scalability of \method (Figures~\ref{fig:compression} and \ref{fig:scalability}).
    Especially, we showed that \method successfully compressed a matrix with up to $230$ millions of non-zero entries.
\end{itemize}
\vspace{-1mm}
\smallsection{Reproducibility}
The code and datasets used are available at~\supplelink.

\subsection*{Acknowledgements}
This work was supported by Samsung Electronics Co., Ltd., National Research Foundation of Korea (NRF) grant funded by the Korea government (MSIT) (No.2021R1C1C1008526), and  Institute of Information \& Communications Technology Planning \& Evaluation (IITP) grant funded by the Korea government (MSIT)  (No. 2022-0-00157,
Robust, Fair, Extensible Data-Centric Continual Learning) (No. 2019-0-00075, Artificial Intelligence Graduate School Program (KAIST)) (No.2021-0-02068, Artificial Intelligence Innovation Hub).

\bibliographystyle{ACM-Reference-Format}
\bibliography{bib}
\pagebreak

\appendix







\section{Additional Experimental Results}
\label{sec:supple:exp}
\begin{figure}[t]
    \vspace{-2mm}
    \centering
    \includegraphics[width=0.7\linewidth]{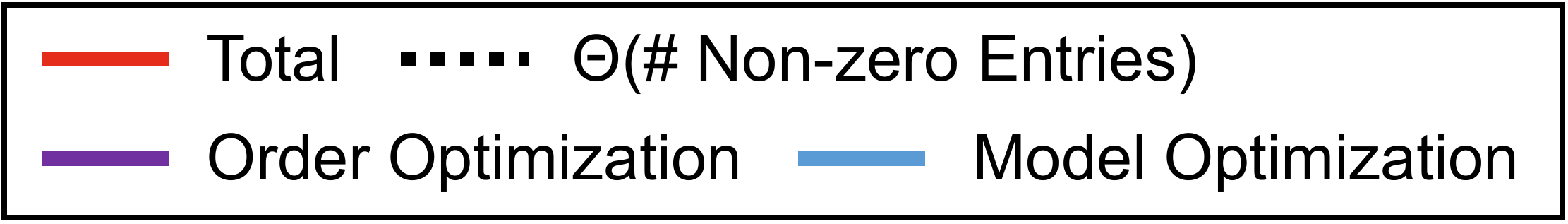}
    \subfigure[\textt{twitch}]{
        \centering
        \includegraphics[width=0.31\linewidth]{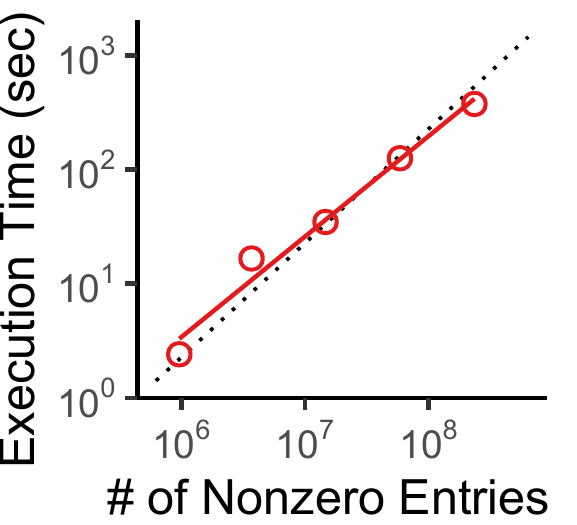}
        \includegraphics[width=0.31\linewidth]{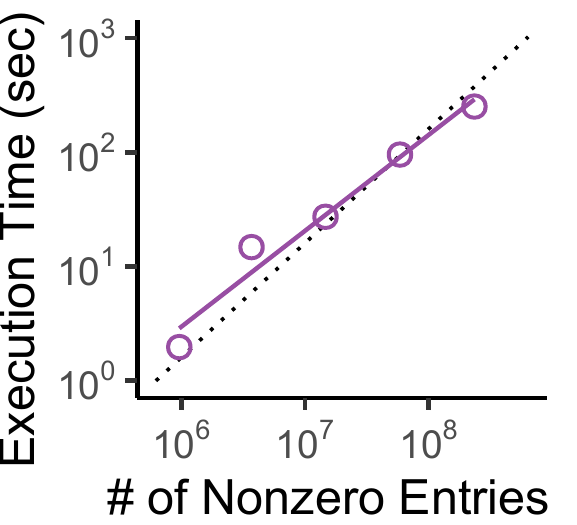}
        \includegraphics[width=0.31\linewidth]{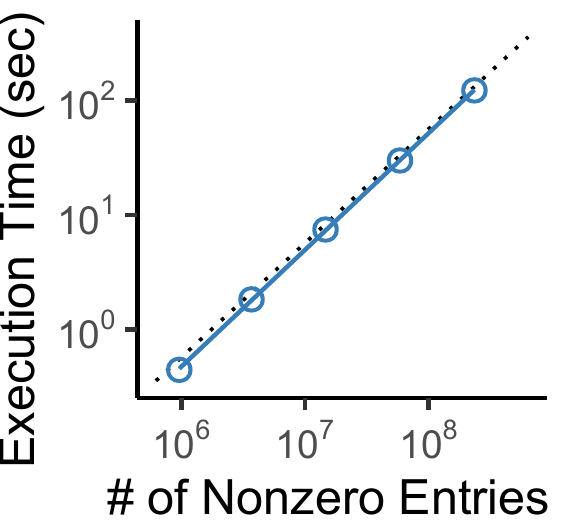}
    } \\
    \vspace{-2mm}
    \subfigure[\textt{threads}]{
        \centering
        \includegraphics[width=0.31\linewidth]{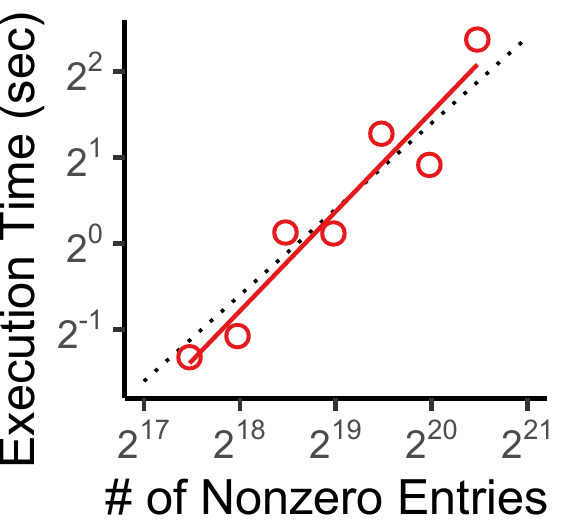}
        \includegraphics[width=0.31\linewidth]{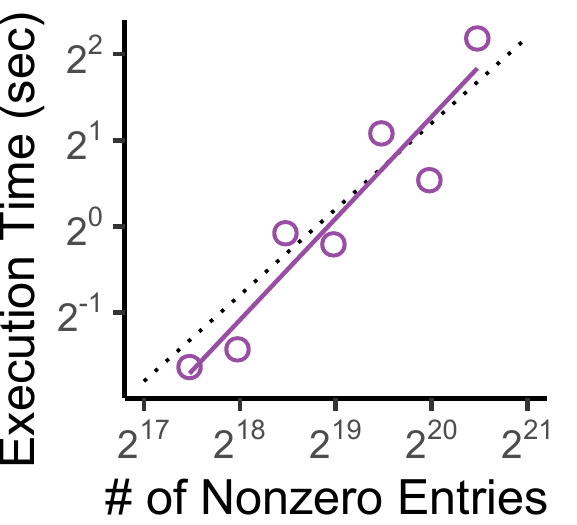}
        \includegraphics[width=0.31\linewidth]{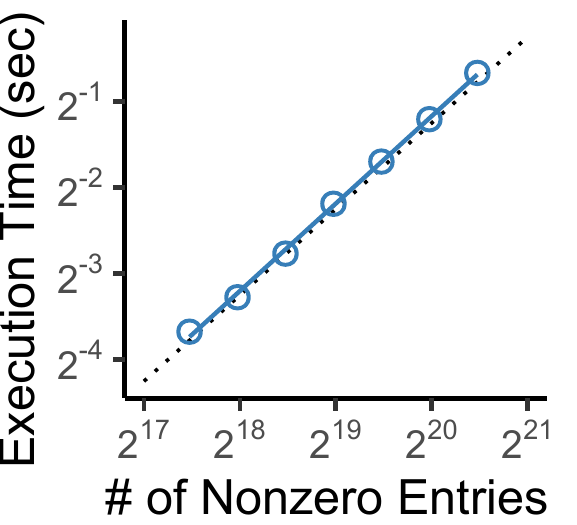}
    } \\
    \vspace{-2mm}
    \caption{\underline{\smash{The training process of \method is scalable.}} 
    Both model and order optimizations scale near-linearly with the number of non-zeros in the input. }
    \label{fig:scalability}
\end{figure}

\subsection{Q3. Scalability and Speed}
\label{sec:exp:scalability}
%

In order to evaluate the scalability of \method, we generated multiple matrices of various sizes from the \textt{threads} and \textt{twitch} datasets by tracking their evolutions over time.
In them, we measured the training time per epoch for model and order optimizations in addition to the total training time per epoch. The hidden dimension $h$ was fixed to $60$. 
%
As shown in Figure \ref{fig:scalability}, the individual and overall training processes of \method scaled \textbf{linearly with the number of non-zeros}, which is consistent with the theoretical results in Section~\ref{sec:method:analysis}. 
We further confirmed the linear scalability of \method on tensor datasets and in hidden dimensions in Section~8 of \cite{appendix}.

\kijung{We compared the training time of \method and the competitors in Figure~\ref{fig:trainingtime}. We followed the hyperparameter settings in Section~\ref{sec:exp:model}.
For \method, we reported the result with the smallest hidden dimensions that we considered. 
For all competitors except for KronFit, we reported their results when their approximation errors are closest to that of \method. For KronFit, we reported its result when its output size is closest to that of \method.}
Since our optimization problem is a mixed discrete-continuous optimization problem, which is notoriously difficult, the convergence of \method takes much longer than that of factorization-based methods.
While the convergence took long, the approximation error dropped rapidly in early iterations in most cases. 
\red{The detailed training curves are given in Figure~\change{2} of \cite{appendix}.}

\subsection{Q4. Approximation Analysis}
\label{sec:exp:approximation}
We analyzed how the approximation error by \method varies depending on the ground-truth value of approximated entries.
In each dataset, we grouped the approximated entries by log-binning of their ground-truth values,  
and for each group, we computed the root mean squared error (RMSE) of the approximation errors.
As seen in Figure \ref{fig:reconst}, RMSE tended to increase with respect to ground-truth entry values.
\blue{We also checked at most $1,000$ largest singular values of matrices obtained by \method and the two strongest competitors.}
For each method, we used the hyperparameter settings that led to the  least approximation error in Figures~\ref{fig:compression} and \ref{fig:appdix:trade-off}.
\blue{As seen in Figure~\ref{fig:singular_value}, the singular values obtained by \method were closest to the singular values of the input matrices.}

\begin{figure}[t]
    \centering
    \vspace{-2mm}
    \includegraphics[width=0.91\linewidth]{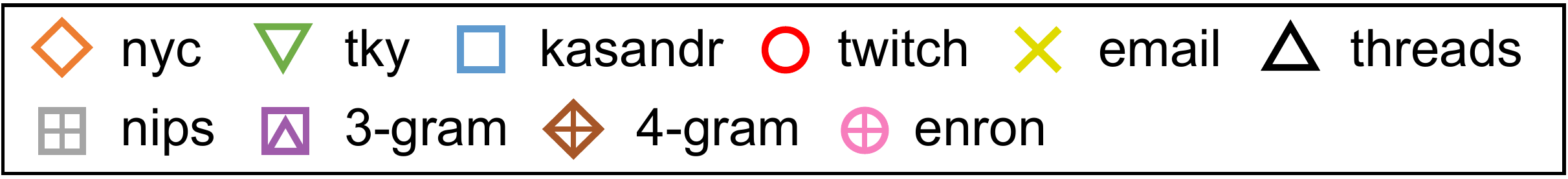}
    \vspace{-1mm}
    \subfigure[RMSE on sparse matrices]{
        \includegraphics[width=0.41\linewidth]{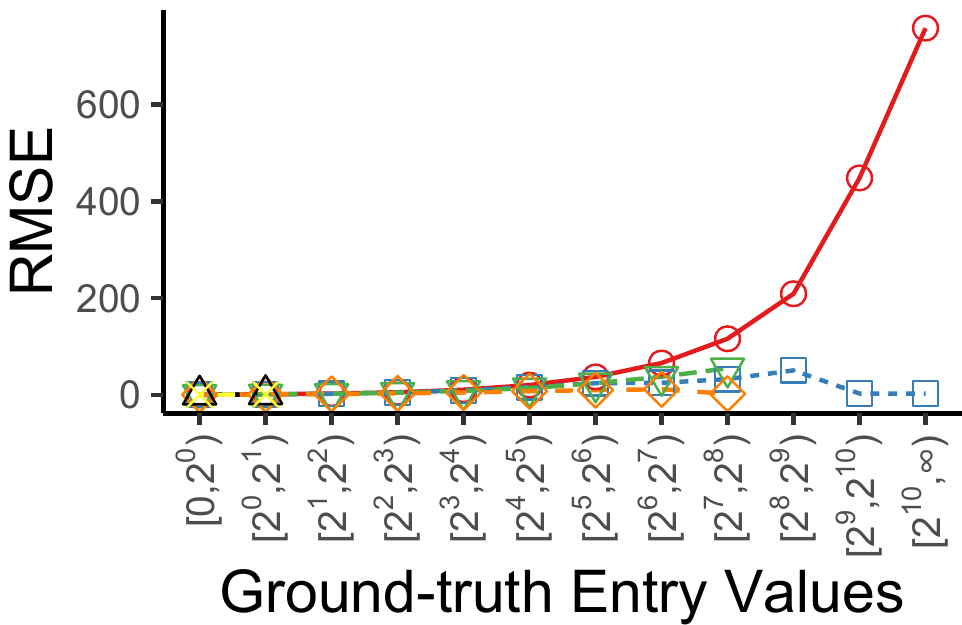} 
    }
    \subfigure[RMSE on sparse tensors]{
        \includegraphics[width=0.53\linewidth]{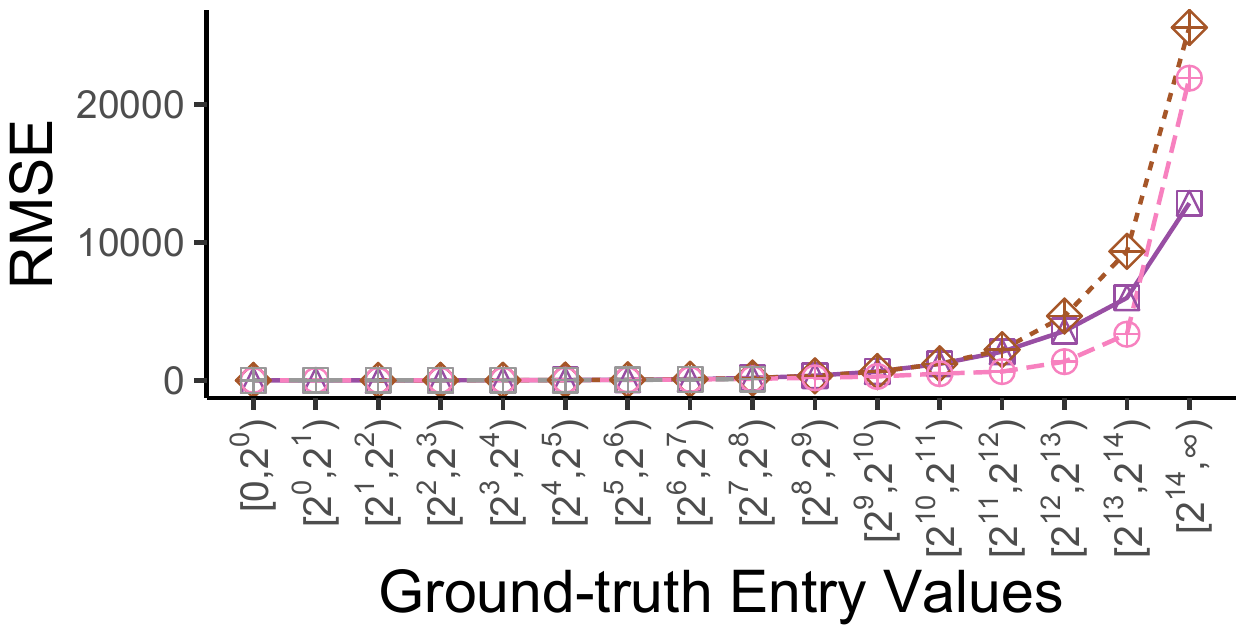} 
    }
    \\
    \vspace{-1mm}    
    \caption{\label{fig:reconst}\underline{\smash{Analysis of approximation errors of \method.}} The errors tend to increase with respect to the ground truth values of approximated entries. The $x$-axis is in the log scale.}
\end{figure}
\vspace{-2mm}
\begin{figure}[t]
    \centering
    \vspace{-2mm}
    \includegraphics[width=1\linewidth]{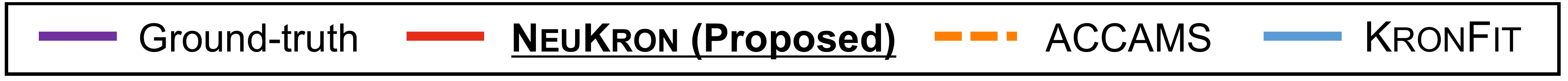} \\
    \vspace{-2mm}
    \subfigure[\textt{email}]{
        \includegraphics[width=0.33\linewidth]{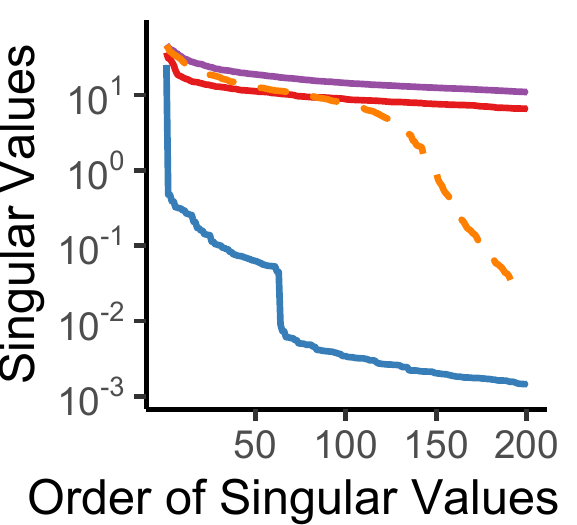}
    }
    \hspace{-3.7mm}
    \subfigure[\textt{nyc}]{
        \includegraphics[width=0.33\linewidth]{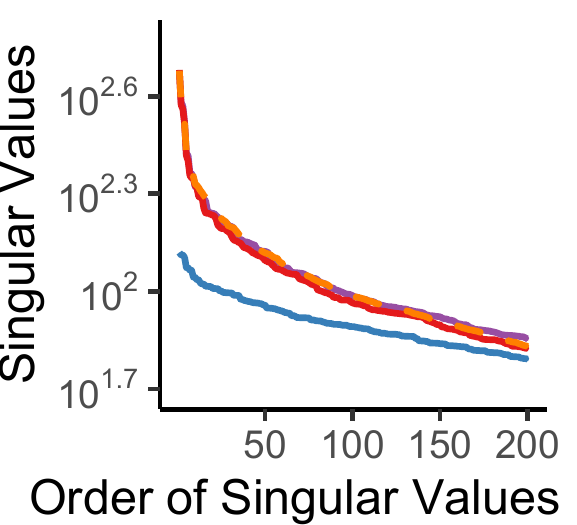}
    }
    \hspace{-3.8mm}
    \subfigure[\textt{tky}]{
        \includegraphics[width=0.33\linewidth]{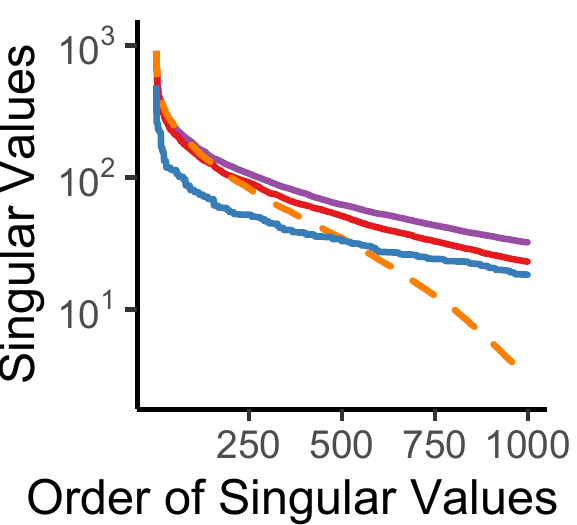}
    } \\
    \vspace{-2mm}
    \caption{\blue{\underline{\smash{\method preserves singular values well.}} The singular values of the matrix obtained by \method are closest to the ground-truth ones.
    We used the smallest datasets for this experiment since
    computing singular values requires approximating all entries, including zeros.
    }}
    \label{fig:singular_value}
\end{figure}

\begin{figure*}[t]
\vspace{-2mm}
\centering
\includegraphics[width=0.7\linewidth]{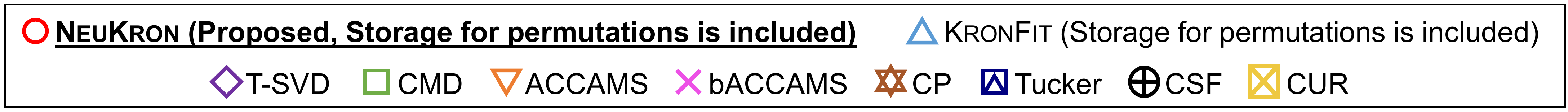} \\
\vspace{-2mm}
\subfigure[\textt{kasandr}]{
    \includegraphics[width=0.17\linewidth]{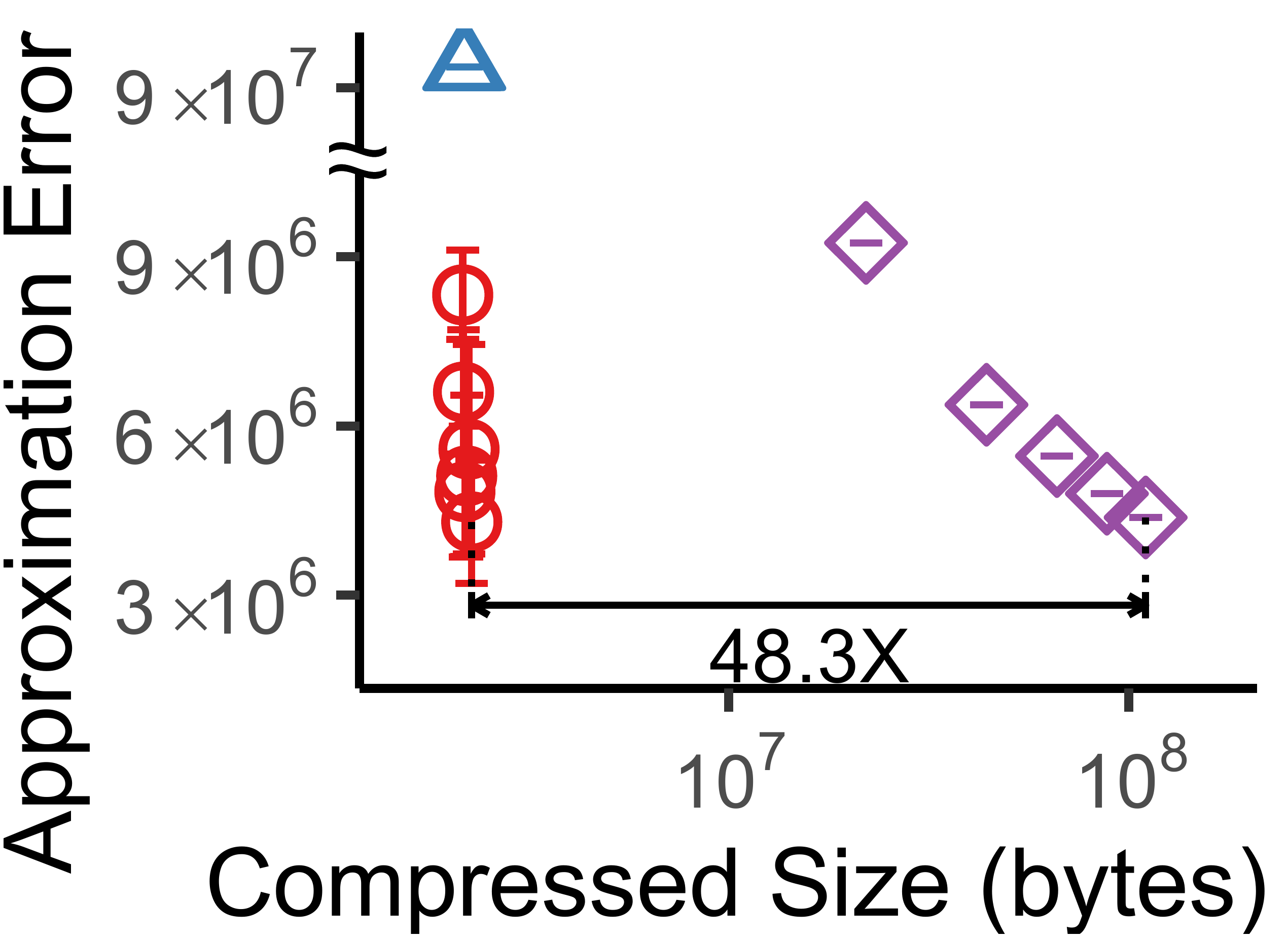}}
\subfigure[\textt{twitch}]{
    \includegraphics[width=0.17\linewidth]{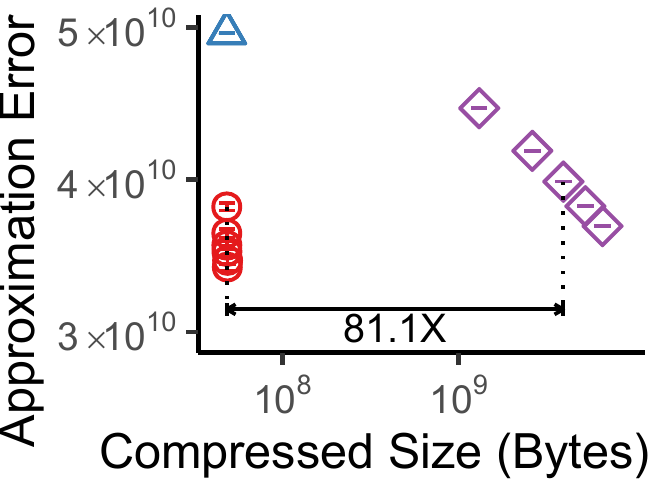}}
\subfigure[\textt{email}]{
    \includegraphics[width=0.17\linewidth]{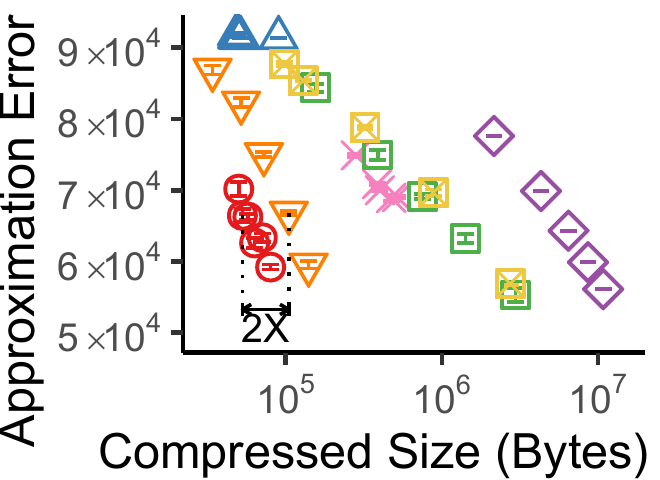}}
\subfigure[\textt{nyc}]{
    \includegraphics[width=0.17\linewidth]{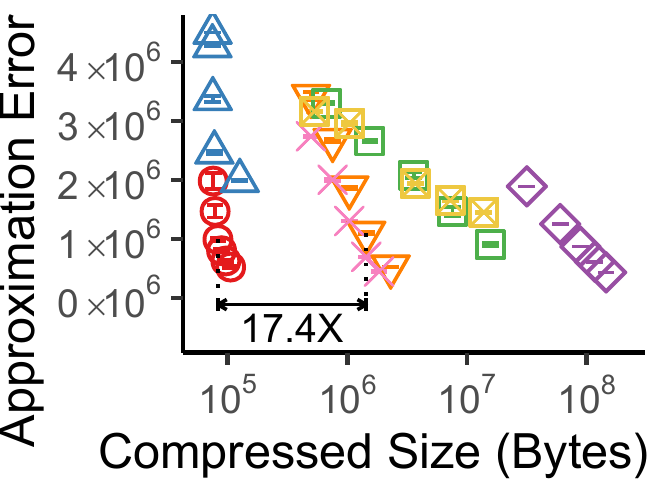}}
\subfigure[\textt{tky}]{
    \includegraphics[width=0.17\linewidth]{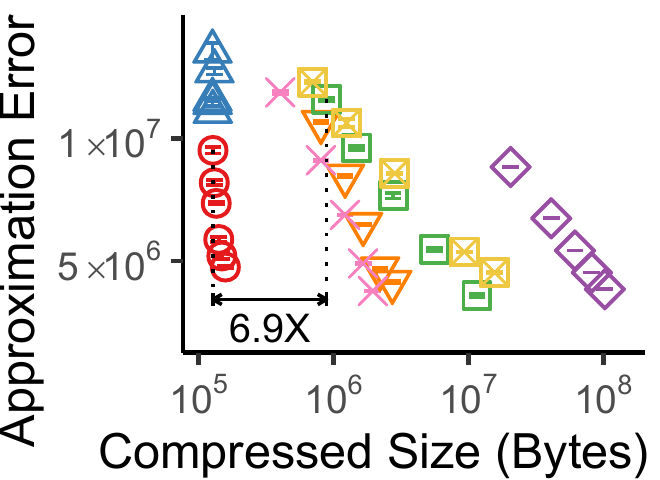}} \\
\vspace{-3mm}
\subfigure[\textt{threads}]{
    \includegraphics[width=0.17\linewidth]{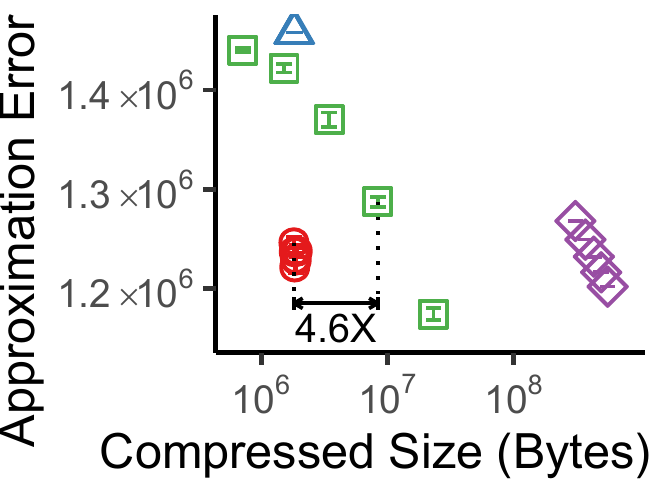}} 
\subfigure[\textt{nips}]{
    \includegraphics[width=0.17\linewidth]{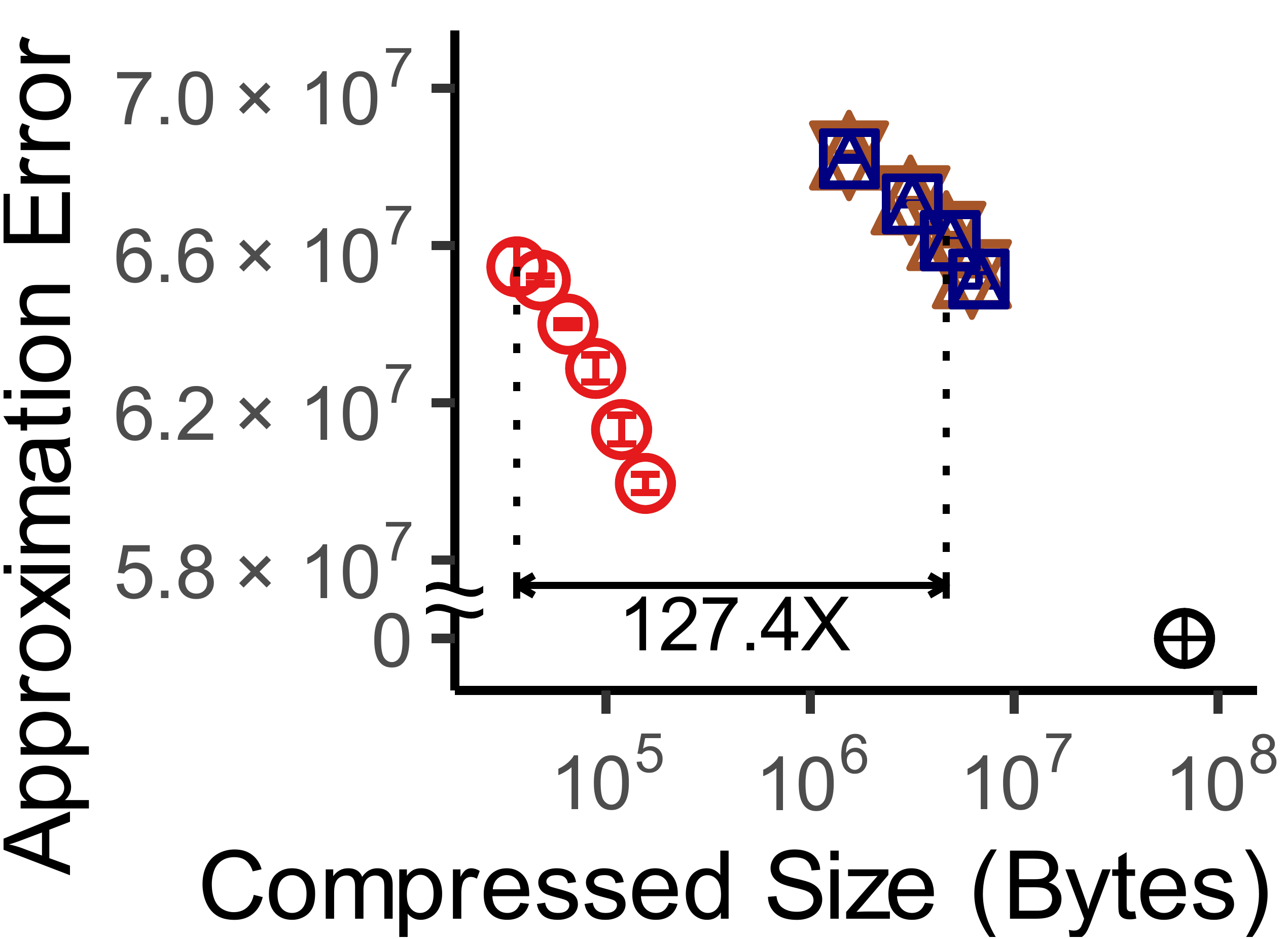}}
 \subfigure[\textt{enron}]{
    \includegraphics[width=0.17\linewidth]{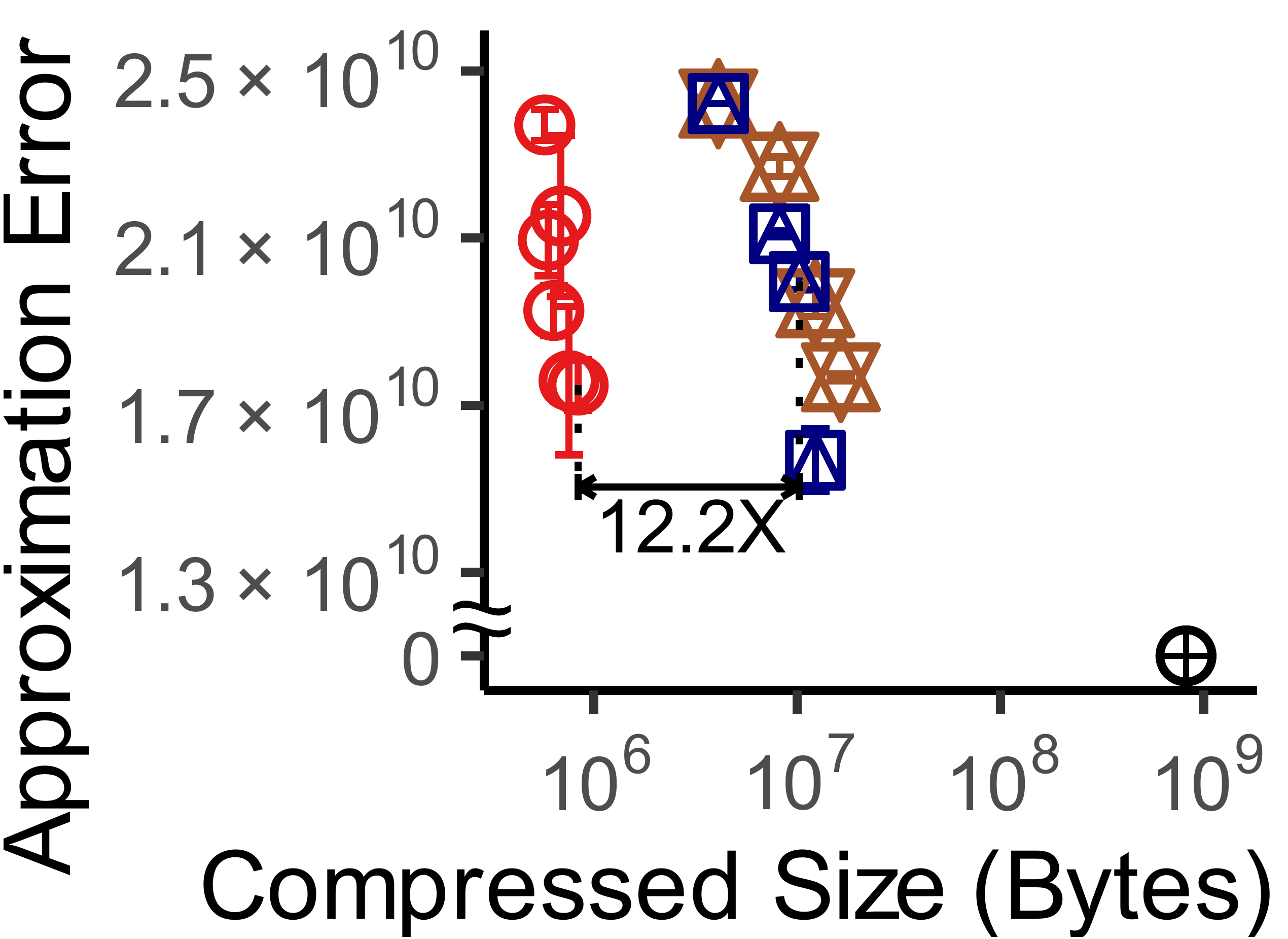}}
 \subfigure[\textt{3-gram}]{
    \includegraphics[width=0.17\linewidth]{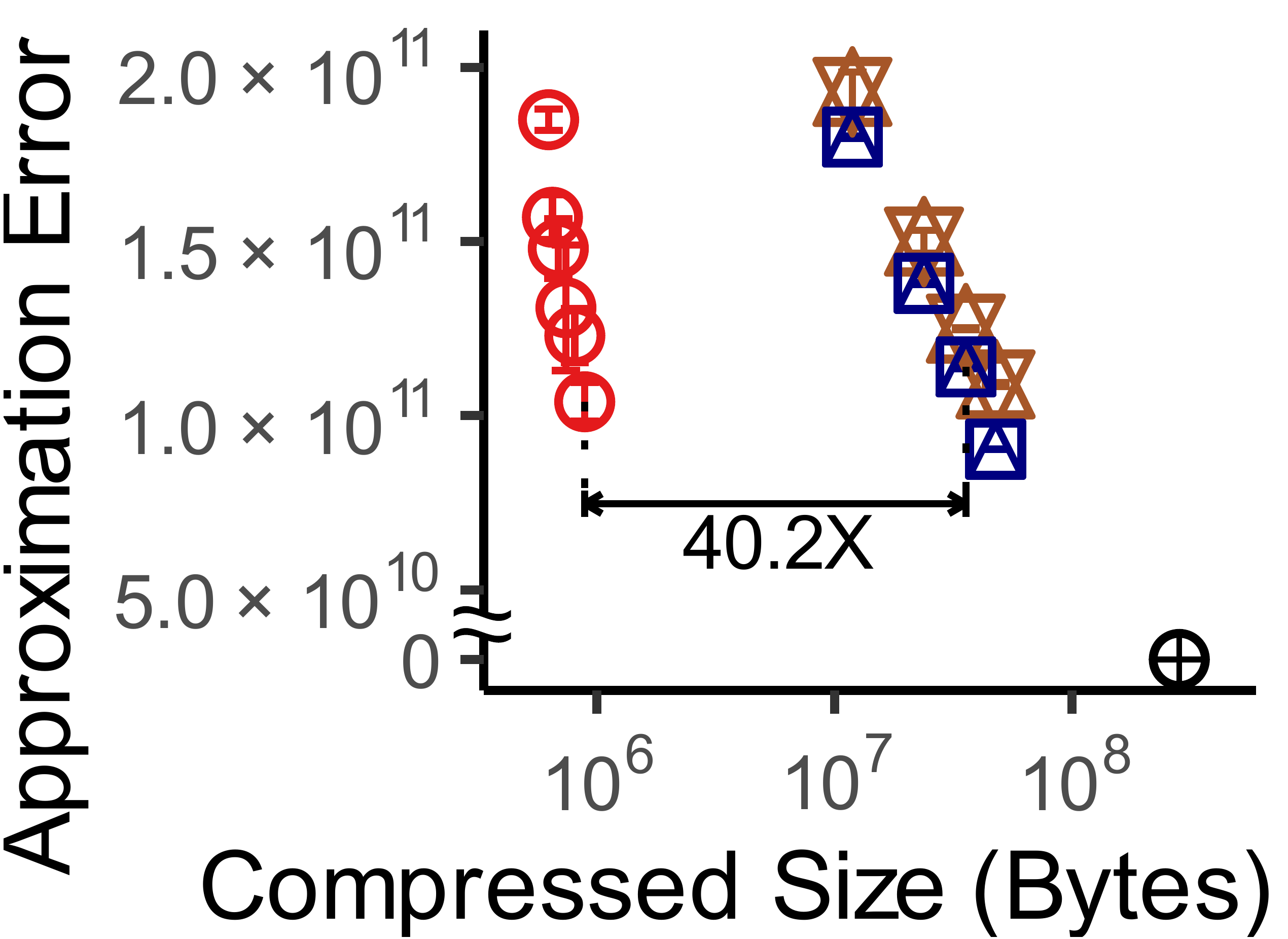}}
 \subfigure[\textt{4-gram}]{
    \includegraphics[width=0.17\linewidth]{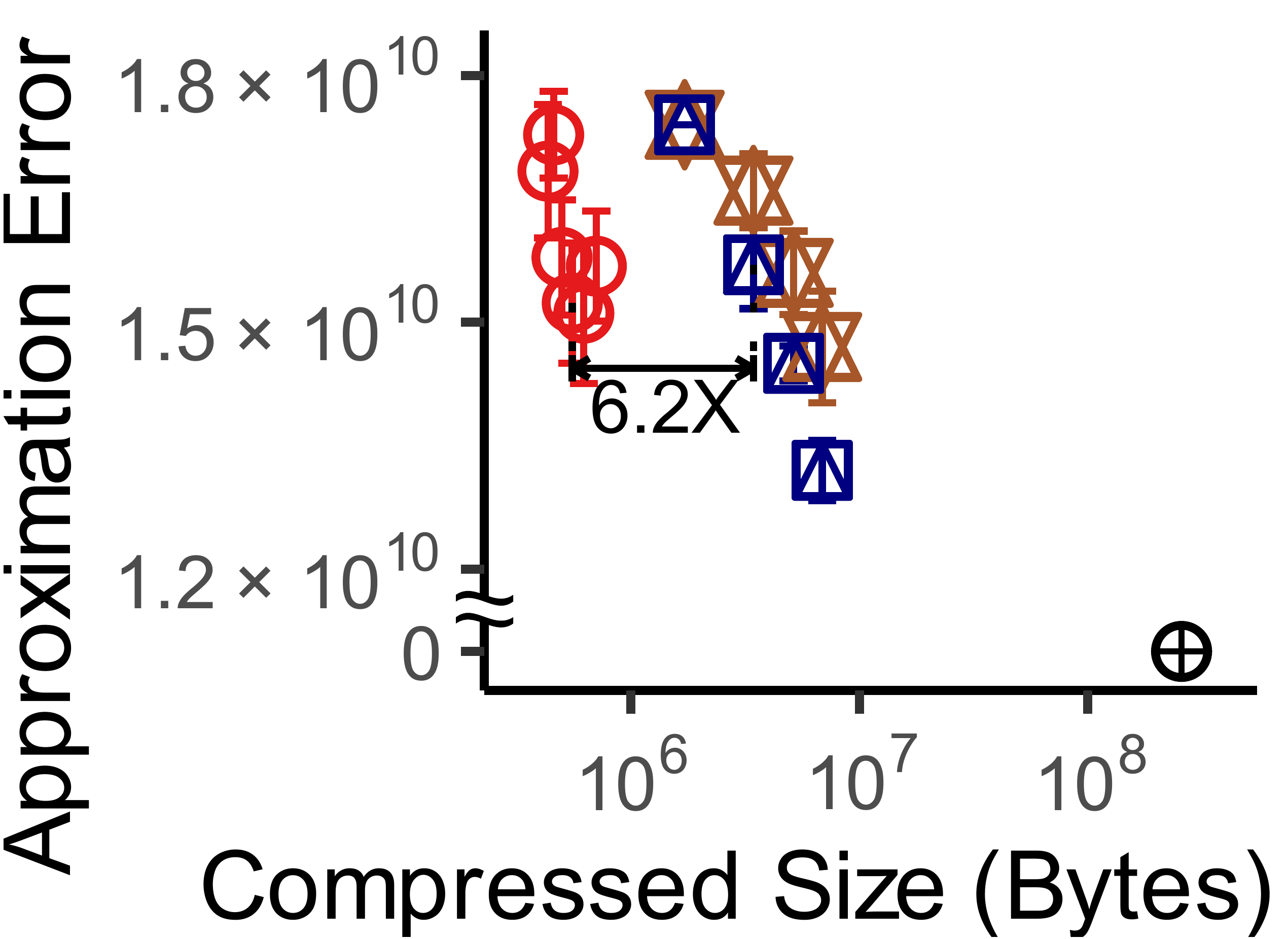}} \\
\vspace{-2mm}
\caption{\blue{\underline{\smash{\method significantly outperforms the competitors even if we assume that matrices and tensors are \textit{non-reorderable}}} \underline{\smash{and separately store the permutations of indices for all modes.}}}
Note that the outputs of \method require up to two orders of magnitude smaller space than those of the competitors with similar approximation error.}
\label{fig:appdix:trade-off}
\end{figure*}

\subsection{Q5. Effects of Data Properties}
\label{sec:exp:inputs}
We investigated the effects of properties of an input tensor $\cX$ on the performance of \method. 
For this experiment, we synthetically generated tensors using the multi-dimensional extension R-MAT~\cite{chakrabarti2004r}. Specifically, we first split each mode of a tensor into two partitions and then chose either the first partition with probability $p$ or the second one with probability $1-p$.
This process is repeated until the target position is determined. 
As a default setting, we set (a) $p$ to $0.8$, (b) the order $D$ to $3$, (c) the sum of all tensor entries to $10^6$, and (d) the number of entries to $2^{30}$.
We measured \textit{fitness}, which is defined as $1 - {\lVert \cX - \reX \rVert_F}/{\lVert \cX \rVert_F}$ (the higher, the better).
The fitness is widely used to compare the errors of approximations to different tensors.
%
\begin{figure}[t]
    \centering
    \vspace{-2mm}
    \subfigure[Effect of Skewness]{
        \includegraphics[width=0.32\linewidth]{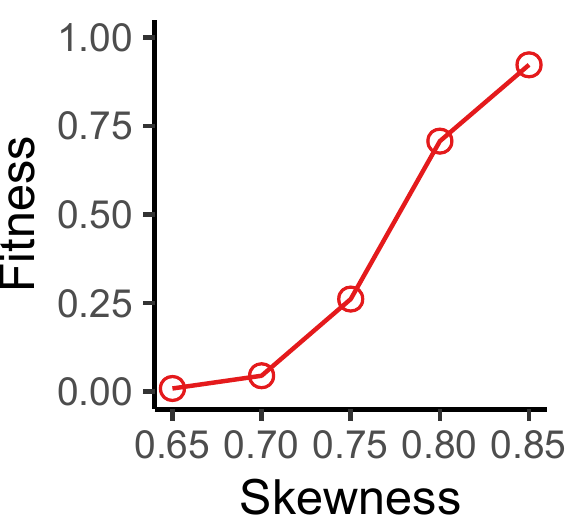}
        \vspace{-5mm}
        \label{fig:q5:prob}
    }
    \hspace{-1.35em}
    \subfigure[Effect of Order]{
        \includegraphics[width=0.32\linewidth]{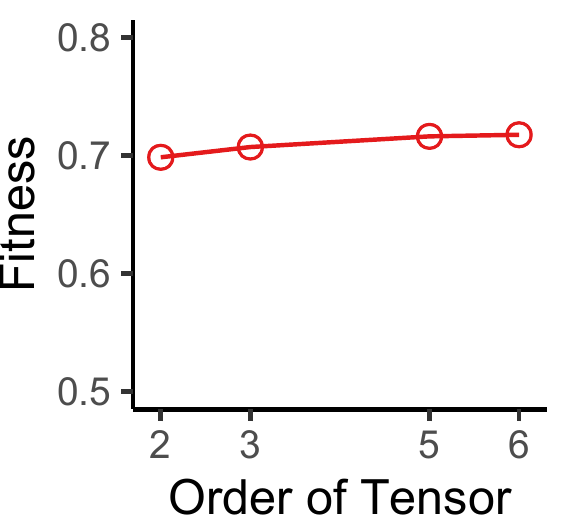}
        \vspace{-5mm}
        \label{fig:q5:order}
    }
    \hspace{-1.05em}
    \subfigure[Effect of Dim.]{
        \includegraphics[width=0.32\linewidth]{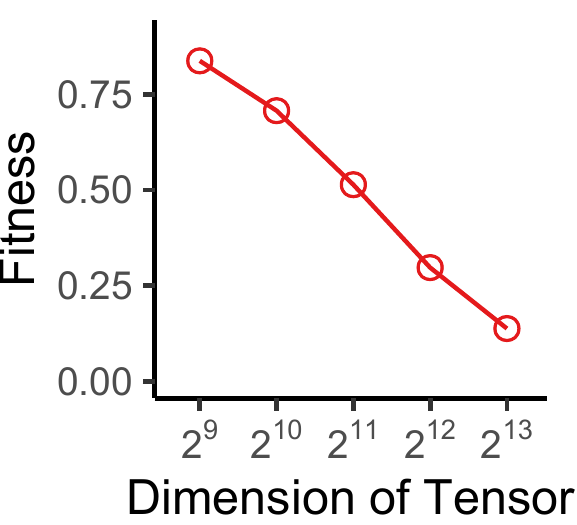}
        \vspace{-5mm}
        \label{fig:q5:density}
    }
    \label{fig:data_props}
    \vspace{-1mm}
    \caption{\label{fig:q5}\underline{\smash{Effects of data properties on \method.}} (a) The fitness increases as the skewness $p$ increases. (b) The order of tensors does not significantly affect the fitness. (c) As tensors become bigger, the fitness decreases.} 
\end{figure}
We varied the skewness $p$ from $0.65$ to $0.85$. Note that increasing $p$ makes the distribution of non-zero entries more skewed with distinct patterns, and decreasing $p$ makes the distribtuion more uniform without patterns.
As seen in Figure~\ref{fig:q5}(a), the fitness increased as $p$ increased, implying that \method provides better performance on skewed tensors with distinct patterns.
%
Next, we changed the order $D$ from $2$ to $6$, but no significant effect of $D$ was observed, as shown in Figure \ref{fig:q5}(b).
Lastly, we analyzed the effect of dimension (i.e., the number of indices in each mode) by changing the dimension of tensors while fixing the number of non-zeros.
As seen in Figure \ref{fig:q5}(c), the fitness of \method decreased as tensors became bigger and thus more entries were approximated by a fixed number of parameters.

\begin{algorithm}[h]
\caption{\textsc{UpdateRowOrder}}\label{algo:update_perm}
\SetKwInput{KwInput}{Input}
\SetKwInput{KwOutput}{Output}
\KwInput{(a) a reorderable matrix $\mathbf{A}$, (b) a hyperparameter $\gamma$ in Eq.~\eqref{eq:accept}}
\KwOutput{updated matrix $\mathbf{A}$}
 Sample $k \in \mathbb{N} \cup \{0\}$ so that $\text{P}(k=i) = 1/2^{i+1}$\\
 $k$ $\leftarrow$ $\min(k, l_{\text{row}}-1)$; $R$ $\leftarrow$ $\emptyset$; $P$ $\leftarrow$ $\emptyset$ \\
 Generate a random hash bijective functions $h_{\text{col}}$ \\
 \ForEach{$i \in \{i \in [n] : (i-1) \text{ \textnormal{AND} } 2^{k} = 0\}$}{
    $u \sim U(0, 1)$ \\
    \textbf{if} $u$ $<$ $1/2$ \textbf{then}  $R$ $\leftarrow$ $R \cup \{i\}$ \\
    \textbf{else} $R$ $\leftarrow$ $R \cup \{i + 2^{k}\}$
 }
 \ForEach{$i \in R$}{
  $f_{\text{row}}(i)$ $\leftarrow$ $\min_{a_{ij} \neq 0}(h_{\text{col}}(j))$ \\
 }
 \While{$\exists (i_1, i_2)$ s.t. $f_{\text{row}}(i_1) = f_{\text{row}}(i_2)$}{
  $P\leftarrow P \cup \{(i_1, (i_2 - 1) \text{ XOR } 2^{k} + 1 ), (i_2, (i_1 - 1) \text{ XOR } 2^{k} + 1 )\}$ \\
  $R$ $\leftarrow$ $R \setminus \{i_1, i_2\}$
 }
 $R\leftarrow R \cup \{(r-1)\text{ XOR }2^k+1: r \in R\}$ \\
 \While{$R \neq \emptyset$}{
  Randomly sample $(i_1, i_2)$ from $R$\\
  $P$ $\leftarrow$ $P \cup \{(i_1, i_2)\}$; \ \ $R$ $\leftarrow$ $R \setminus \{i_1, i_2\}$
 }
 $P_{\text{accept}}$ $\leftarrow$ $\emptyset$ \\
 \ForEach{$(i_1, i_2) \in P$}{
    $u \sim U(0, 1)$ \\
    $\Delta$ $\leftarrow$ change in the approximation error\\
 
    \textbf{if} $u \geq exp(-\gamma \cdot \Delta)$ \textbf{then} $P_{\text{accept}} \leftarrow P_{\text{accept}} \cup \{(i_1, i_2)\}$
 }
 
 \ForEach{$(i_1, i_2) \in P_{\text{accept}}$}{
    $\mathbf{A}_{i_1,:}$, $\mathbf{A}_{i_2,:}$ $\leftarrow$ $\mathbf{A}_{i_2,:}$, $\mathbf{A}_{i_1,:}$
 }
\end{algorithm}

\section{Pseudocode for ordering rows}
\label{app:code}
The pseudocode of \textsc{UpdateRowOrder} described in Section \ref{sec:method:perm}, is given in Algorithm \ref{algo:update_perm}, where binary representations start from $0$, while row indices start from $1$. 

\section{Effectiveness of \method on non-reorderable data}
\label{app:non_reorder}

\method can also be applied to non-reorderable matrices and tensors if the mapping between the original and new orders of mode indices are stored additionally. 
Even with this additional space requirement, \method still yielded the best trade-off between the approximation error and the compressed size, as seen in Figure~\ref{fig:appdix:trade-off}, where we assume that the input matrices and tensors are not reorderable.
Remark that \kronfit also needs space for storing orders when it is applied to non-reorderable matrices.

\section{Proof of Theorems} \label{app:proof}
\smallsection{Proof of Theorem~\ref{thm:num_params}} 
In \method, the number of the parameters for LSTM is $\Theta(h^2)$. The embedding layer before the LSTM and the linear layers after LSTM require $\Theta(h)$ of parameters. 
The number of parameters for $\mathbf{K}_1$ in Algorithm~\ref{algo:model} is $4$, which is the number of entries.
We consider $\Theta(h)=\Theta(1)$ as $h$ is a constant; thus, the number of parameters is $\Theta(1)$.

\smallsection{Proof of Theorem~\ref{thm:space}}
Storing the input matrix in a sparse format requires $O(\text{nnz}(\mathbf{A}))$ space. 
Changing the orders of rows and columns requires $O(M+N)$ space for saving random hash functions, shingles, and changes of losses.
If we assume that the batch size and the number of parameters of LSTM are constants, its memory usage during inference and backpropagation is $O(\log M)$ since the input length is $O(\log M)$.
Thus, the overall space complexity during training is $O(\text{nnz}(\mathbf{A})+M)$.

\end{document}


\title{NeuKron: Constant-Size Lossy Compression of Sparse \\ Reorderable Matrices and Tensors (Supplementary Document)}
\author{}
\settopmatter{printacmref=false, printfolios=false}

\maketitle
\begin{figure*}
    \centering
    \vspace{-2mm}
    \includegraphics[width=0.4\linewidth]{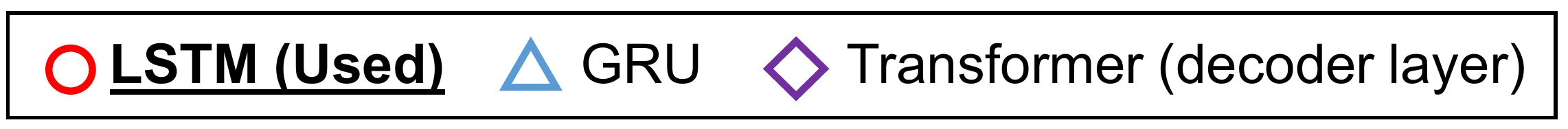}
    \linebreak
    \subfigure[\textt{email}]{
        \includegraphics[width=0.2\linewidth]{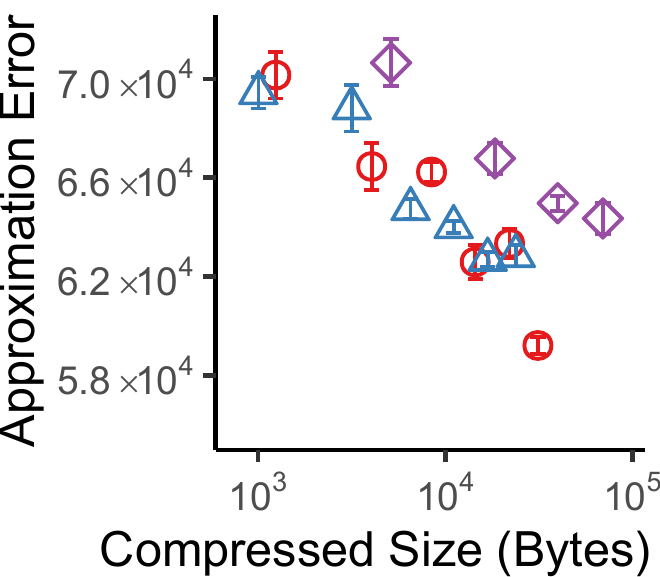}
    }
    \subfigure[\textt{nyc}]{
        \includegraphics[width=0.2\linewidth]{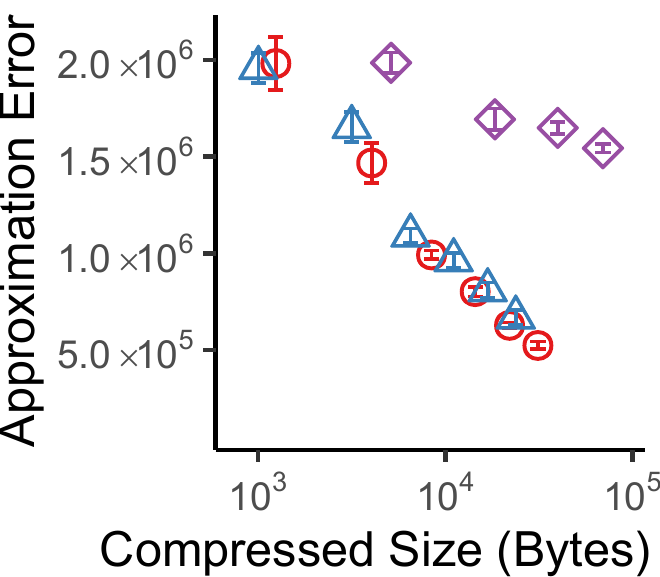}
    }
    \subfigure[\textt{tky}]{
        \includegraphics[width=0.2\linewidth]{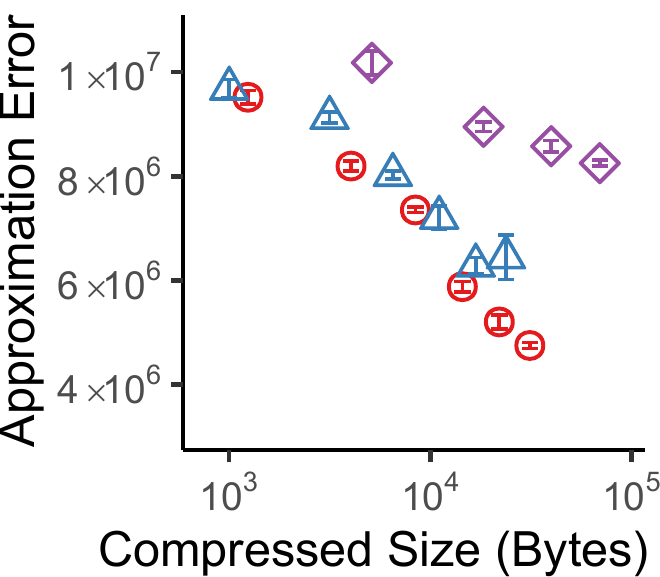}
    } \\
    \vspace{-2mm}
    \caption{\underline{\smash{\red{When equipped with \method}, LSTM leads to concise and accurate compression.}} \method with LSTM and that with GRU show perform similarly, but \method with the decoder layer of Transformer requires significantly more space for the same level of approximation error.}
    \label{fig:model_type}
\end{figure*}

\section{Analysis for the type of the sequential model (Related to Section 4.1)}
We compared the performances of auto-regressive sequence models, \red{when they are equipped with \method.}
We varied the hidden dimesion $h$ of \method from 5 to 30 for LSTM and GRU, and the model dimension $d_{model}$ from 8 to 32 for the decoder layer of Transformer.
As seen in Figure~\ref{fig:model_type}, when equipped with \method, \red{LSTM and GRU performed similarly, outperforming the decoder layer of Transformer.}

\section{Analysis on inference time \\ (Related to Section 4.3)}
We measure the inference time for $10^6$ elements randomly chosen from square matrices of which numbers of rows and cols vary from $2^7$ to $2^{16}$.
We ran 5 experiments for each size and report the average of them.
As expected from Theorem 1 of the main paper, the approximation of each entry by \method is almost in $\Omega(\log{M})$ (see Figure~\ref{fig:inference_time}).

\begin{figure*}[ht]
    \vspace{-4mm}
    \centering
    \subfigure[\textt{email}]{
        \includegraphics[width=0.188\linewidth]{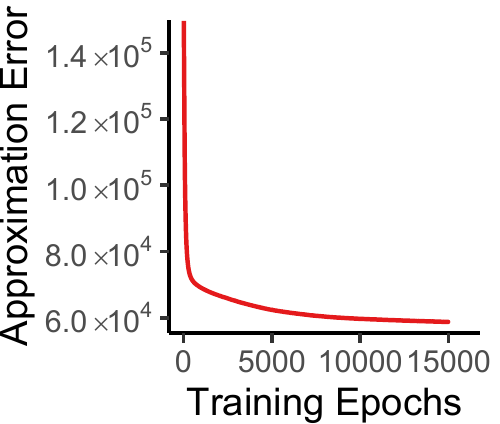}}
    \subfigure[\textt{nyc}]{
        \includegraphics[width=0.188\linewidth]{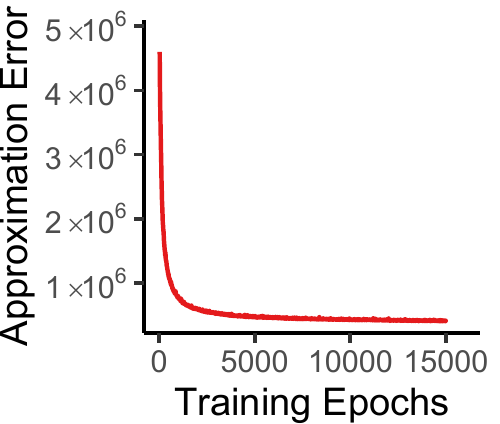}}
    \subfigure[\textt{tky}]{
        \includegraphics[width=0.188\linewidth]{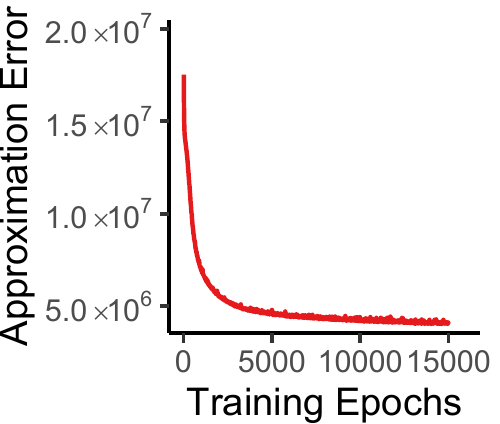}}
    \subfigure[\textt{kasandr}]{
        \includegraphics[width=0.188\linewidth]{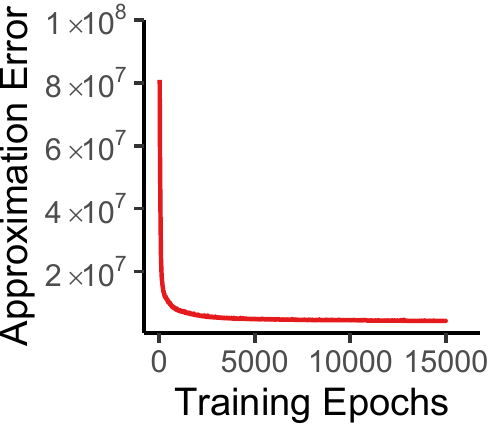}}
    \subfigure[\textt{threads}]{
        \includegraphics[width=0.188\linewidth]{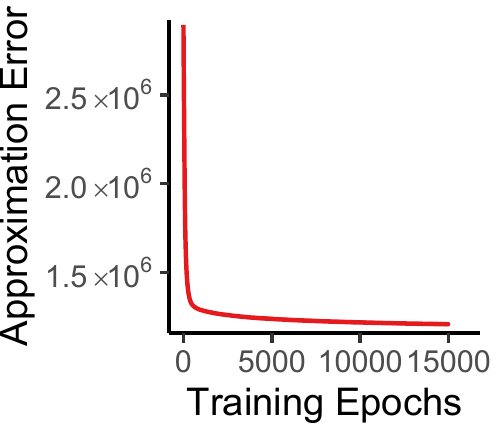}} \\
    \vspace{-2mm}
    \subfigure[\textt{twitch}]{
        \includegraphics[width=0.188\linewidth]{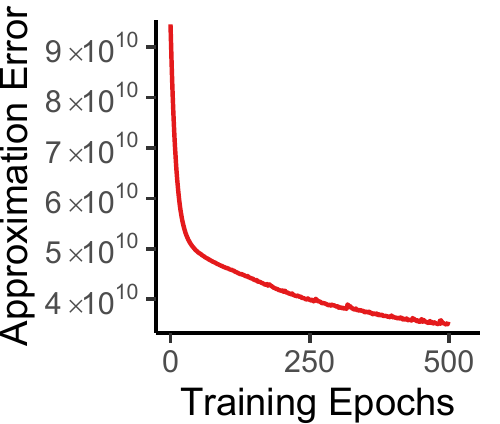}}
    \subfigure[\textt{nips}]{
        \includegraphics[width=0.188\linewidth]{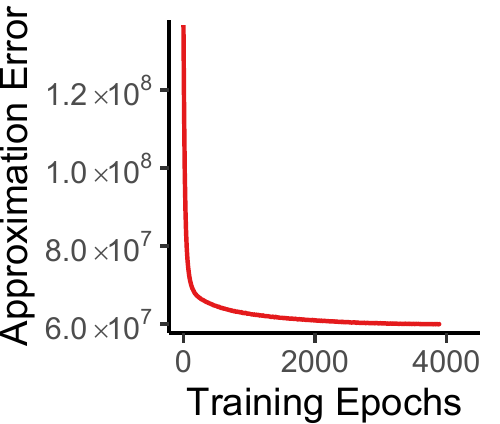}}
    \subfigure[\textt{enron}]{
        \includegraphics[width=0.188\linewidth]{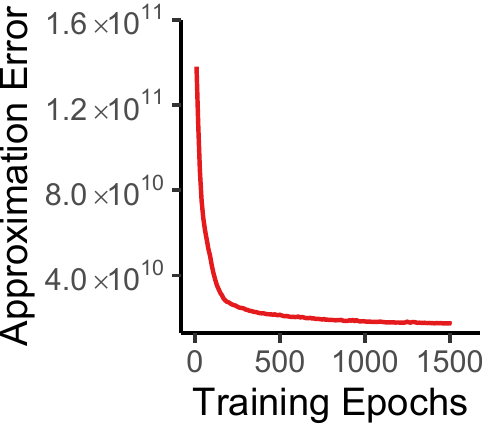}}
    \subfigure[\textt{3-gram}]{
        \includegraphics[width=0.188\linewidth]{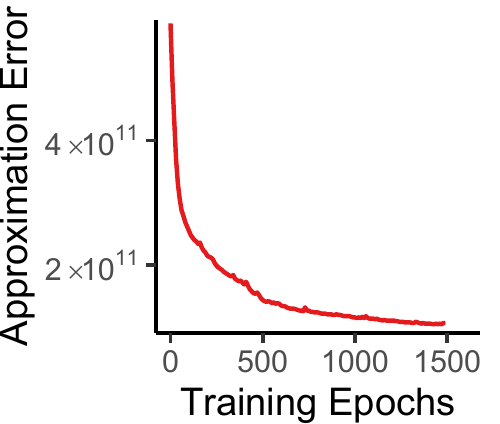}}
    \subfigure[\textt{4-gram}]{
        \includegraphics[width=0.188\linewidth]{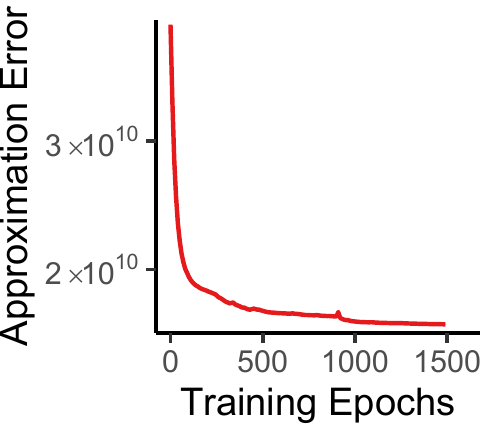}} \\
    \vspace{-2mm}
    \caption{\underline{\smash{Approximation error of \method after each epoch.}} In most cases, the approximation error drops rapidly in early iterations, especially within one third of the total epochs that are determined by the termination condition in Section 6.1. }
    \label{fig:appdix:training_plot}
\end{figure*}

\begin{figure}[ht]
    \centering
    \includegraphics[width=0.376\linewidth]{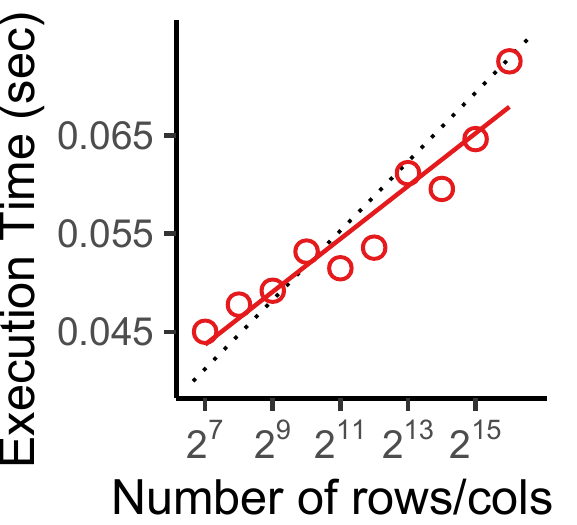}
    \caption{Inference time of \method is nearly proportional to the logarithm of the number of rows and columns.}
    \label{fig:inference_time}
\end{figure}

\section{Analysis of the Tensor Extension \\ (Related to Section 5)}
\label{app:extension}
Below, we analyze the time and the space complexities of \method extended to sparse reorderable tensors. For all proofs, we assume a $D$-order tensor $\mathcal{X} \in \mathbb{R}^{N_1 \times \cdots \times N_D}$ where $N_1 \leq N_2 \leq \cdots \leq N_D$, without loss of generality.
The complexities are the same with those in Section \rome{4} of the main paper if we assume a fixed-order tensor (i.e., if $D=O(1)$).

\begin{theorem}[Approximation Time for Each Entry]
\label{theorem:query_t}
The approximation of each entry by \method takes $\Theta(D \log N_D)$ time.
\end{theorem}
\begin{proof}
For encoding, \method subdivides the input tensor $O(\log N_{D})$ times and each subdivision takes $O(D)$. For approximation, the length of the input of the LSTM equals to the number of the subdivisions, so the time complexity for retrieving each entry is $O(D \log N_{D})$.
\end{proof}

\newpage
\begin{theorem}[Training Time]
\label{thm:time_t}
Each training epoch in \method takes $O(\text{nnz}(\mathcal{X}) \cdot D \log N_D + D^2 N_D)$.
\end{theorem}

\begin{proof}
Since the time complexity for inference is $O(D \log N_{D})$ for each input, model optimization takes $O(\text{nnz}(\mathcal{X}) \cdot D \log N_{D})$. For reordering, the time complexity of matching the indices as pairs for all the dimensions is bounded above to $O(D \cdot (D N_{D})) = O(D^2 N_{D})$. For checking the criterion for all pairs, we need to retrieve all the non-zero entries, and it takes $O(\text{nnz}(\mathcal{X}) \cdot D \log N_{D})$. Therefore, the overall training time per epoch is $O(\text{nnz}(\mathcal{X}) \cdot D \log N_{D} + D^2 N_{D})$.
\end{proof}

\begin{theorem}[Space Complexity during Training]
\method requires $O(D \cdot \text{nnz}(\mathcal{X}) + D^2 N_{D})$ space during training.
\end{theorem}

\begin{proof}
The bottleneck is storing the input tensor in a sparse format, the random hash functions and the shingle values, which require $O(D \cdot \text{nnz}(\cX))$, $O(\sum_{i=1}^{D} N_i)$, and $O((D-1) \cdot \sum_{i=1}^{D} N_i)$, respectively. 
Thus, the overall complexity during training is $O(D \cdot \text{nnz}(\cX) + (D-1) \cdot \sum_{i=1}^{D} N_i) = O(D \cdot \text{nnz}(\cX) + D^2 N_D)$. 
\end{proof}

\begin{theorem}[Space Complexity of Outputs] \label{thm:num_params_t}
The number of model parameters of \method is $\Theta(2^{D})$.
\end{theorem}
\begin{proof}
In \method, the number of parameters for LSTM does not depend on the order of the input tensor; thus, it is still in $\Theta(1)$. 
The embedding layer and the linear layers connected to the LSTM require $\Theta(2 + 2^2 + \dots + 2^{D}) = \Theta(2^{D})$  parameters. 
\end{proof}

\section{Semantics and Properties of Datasets (Related to Section~6.1)}
\label{app:datasets}
We provide the semantics of the datasets in Table \ref{tab:dataset:detail} \kijung{and the distributions of degrees, entry values, and connected-component sizes in Table \ref{tab:dataset:properties}. For degrees, we computed the sums of the rows and those of the columns for matrices. For connected-component sizes, we treated sparse matrices as bipartite graphs and used the number of nodes in each connected component as its size. Note that these properties are naturally extended to the tensors.
}

\section{Implementation Details \\ (Related to Section~6.1)}
\label{app:implementation}
We implemented \method in PyTorch. 
We implemented the extended version of \kronfit in C++. 
For ACCAMS, bACCAMS, and CMD, we used the open-source implementations provided by the authors. 
We used the \texttt{svds} function of SciPy 
for T-SVD.
We used the implementations of CP and Tucker decompositions in Tensor Toolbox \cite{bader2008efficient} in MATLAB. 
Below, we provide the detailed hyperparameter setups of each competitor.
\begin{itemize}[leftmargin=*]
    \item{\textbf{\kronfit}:}{
        The maximum size of the seed matrix was set as follows - \texttt{email}: 32 $\times$ 161, \texttt{nyc}: 33 $\times$ 196, \texttt{tky}: 14 $\times$ 40, \texttt{kasandr}: 75 $\times$ 80, \texttt{threads}: 57 $\times$ 85, \texttt{twitch}: 30 $\times$ 63. 
        \change{We tested the performance of \kronfit when $\gamma$ is 1 and 10, and fixed $\gamma$ to $10$ because it performs better when $\gamma$ is set to 10.} We performed a grid search for the learning rate in $\{10^{-1}, 10^{-2}, \cdots, 10^{-8}\}$.
    }
    \item{\textbf{T-SVD}:}{
        The ranks were up to 50 for \texttt{email}, 460 for \texttt{nyc}, 200 for \texttt{tky}, 15 for \texttt{kasandr}, 90 for \texttt{threads}, and 50 for \texttt{twitch}.
    }
    \item{\textbf{CUR}:}{
        \change{We selected ranks for CUR from \{10, 100, 1000\}.
        We sampled \{1\%, 1.25\%, 2.5\%, 5\%, 10\%\} of rows and columns in \textt{email}, \{3.3\%, 5\%, 10\%, 14.3\%, 20\%\} of rows and columns in \textt{nyc}, and \{1\%, 2\%, 4\%, 8.3\%, 11.1\%\} of rows and columns in \textt{tky}}}
    \item{\textbf{CMD}:}{
        \change{We sampled (\# rows, \# columns) as much as \{(30, 150), (60, 350), (90, 700), (100, 1400), (150, 2500)\} for \texttt{email}, \{(65, 2125), (125, 4250), (250, 8500), (500, 17000), (1000, 34000)\} for \texttt{nyc}, \{(45, 1315), (90, 2625), (175, 5250), (350, 10500), (700, 21000)\} for \texttt{tky}, and \{(55, 184), (109, 368), (218, 736), (436, 1471), (871, 2941)\} for \textt{threads}.}
    }
    \item{\textbf{ACCAMS}:}{
        We used 5, 50, and 50 stencils for \texttt{email}, \texttt{nyc}, and \texttt{tky}, respectively.
        We used up to 48, 64, and 40 clusters of rows and columns for the aforementioned datasets, respectively.
    }
     \item{\textbf{bACCAMS}:}{  
        We set the maximum number of clusters of rows and columns to 48, 48, and 24 for \texttt{email}, \texttt{nyc}, and \texttt{tky}, respectively.
        We used 50 stencils for the datasets.
    }
    \item{\textbf{CP}:}{
        The ranks were set up to 40 for \texttt{nips}, 8 for \texttt{enron}, 20 for \texttt{3-gram}, and 4 for \texttt{4-gram}.
    }
    \item{\textbf{Tucker}:}{
        We used hypercubes as core tensors. The maximum dimension of a hypercube for each dataset is as follows - \texttt{nips}: 40, \texttt{enron}: 6, \texttt{3-gram}: 20, and \texttt{4-gram}: 4.
    }
\end{itemize}

\begin{table}[t]
    \centering
    \caption{Semantics of Real-world Datasets}
    \label{tab:dataset:detail}
    \scalebox{0.9}{
    \begin{tabular}{c|c}
        \toprule
        \textbf{Name} & \textbf{Description} \\
        \midrule
        \texttt{email} & e-mail addresses $\times$ e-mails [binary] \\
        \texttt{nyc}  & venues $\times$ users [check-in counts] \\
        \texttt{tky} & venues $\times$ users [check-in counts] \\
        \texttt{kasandr} & offers $\times$ users [clicks] \\
        \texttt{threads} & users $\times$ threads [participation] \\
        \texttt{twitch} & streamers $\times$ users [watching time] \\
        \texttt{nips} & papers $\times$ authors $\times$ words [counts]\\
        \texttt{4-gram} & words $\times$ words $\times$ words $\times$ words [counts] \\
        \texttt{3-gram} & words $\times$ words $\times$ words [counts]\\
        \texttt{enron} & receivers $\times$ senders $\times$ words [counts]\\
        \bottomrule
    \end{tabular}
    }
\end{table}

\begin{figure*}[ht]
    \centering
    \vspace{-2mm}
    \includegraphics[width=0.8\linewidth]{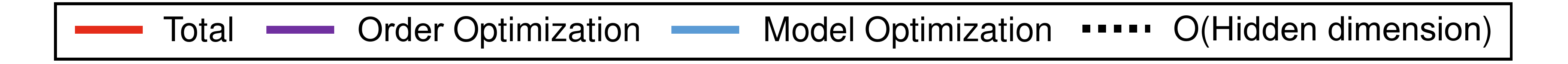}
    \subfigure[threads]{
        \includegraphics[width=0.15\linewidth]{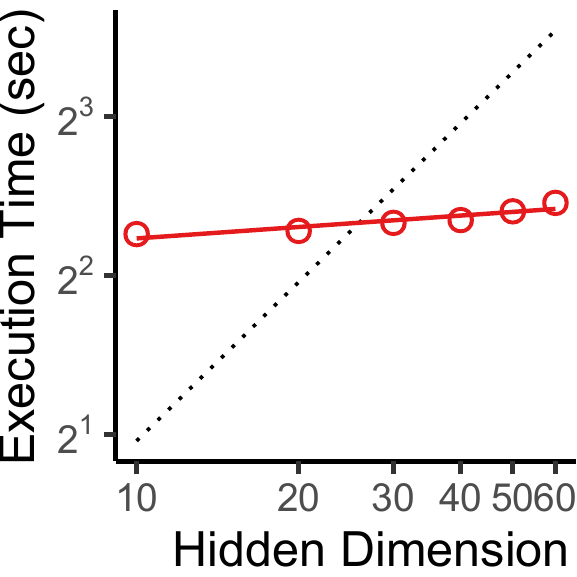}
        \includegraphics[width=0.15\linewidth]{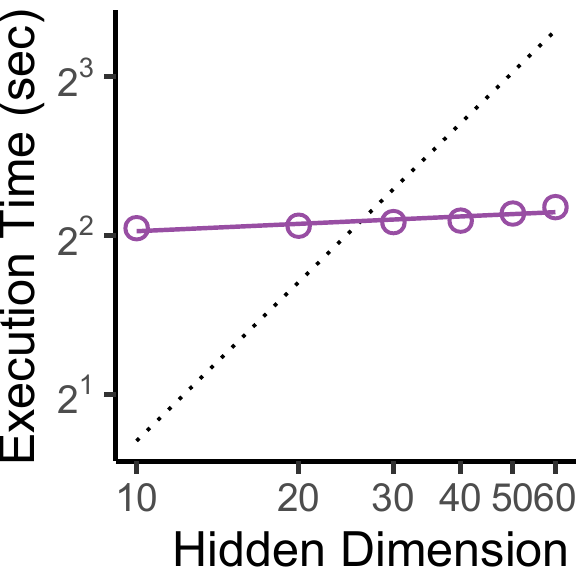}
        \includegraphics[width=0.15\linewidth]{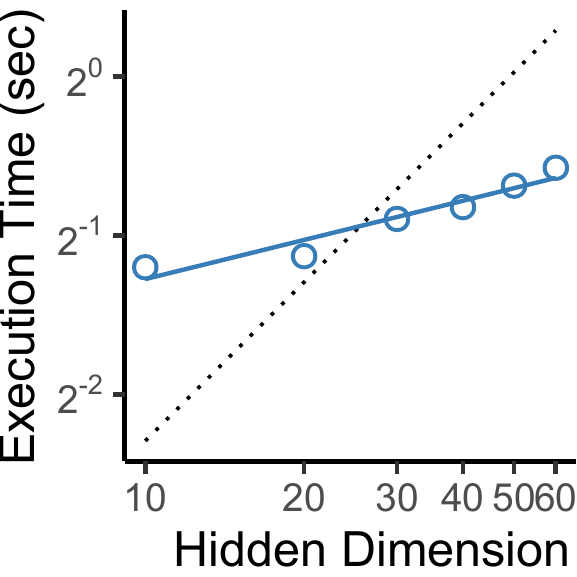}}
    \subfigure[twitch]{
        \includegraphics[width=0.15\linewidth]{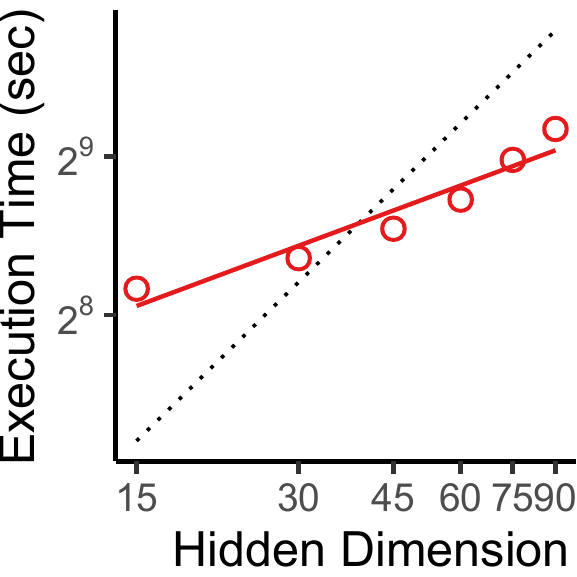}
        \includegraphics[width=0.15\linewidth]{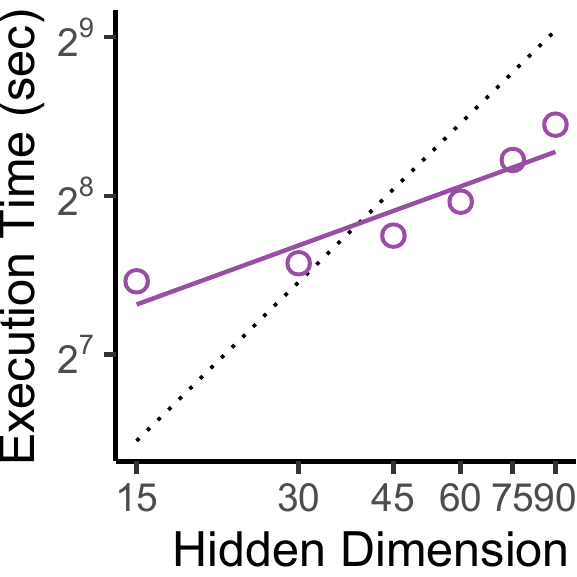}
        \includegraphics[width=0.15\linewidth]{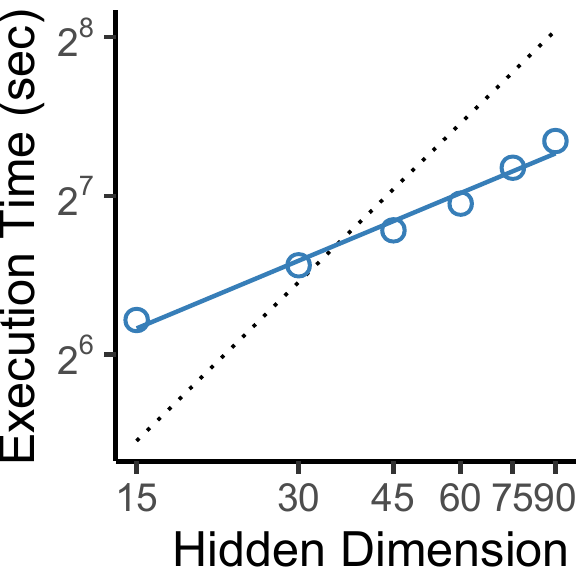}
        }
        \\
    \vspace{-2mm}
    \subfigure[4-gram]{
        \includegraphics[width=0.15\linewidth]{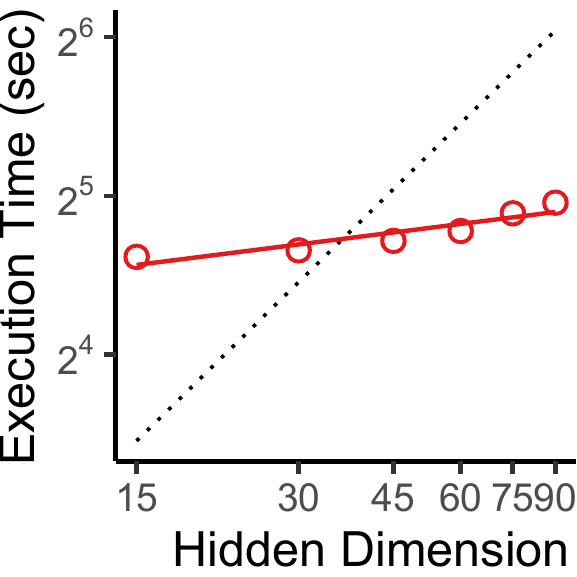}
        \includegraphics[width=0.15\linewidth]{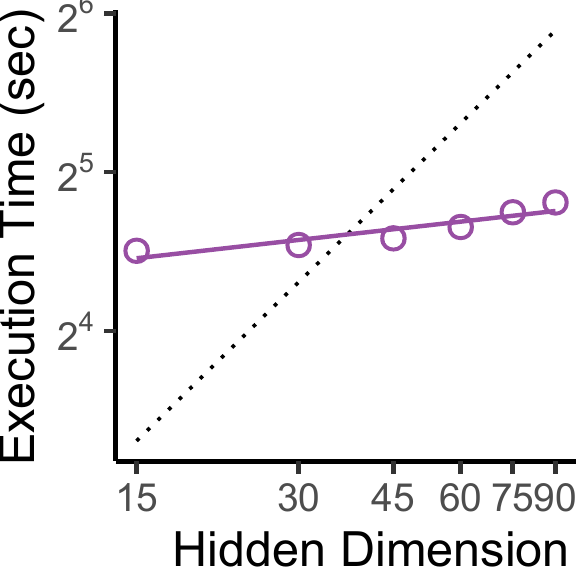}
        \includegraphics[width=0.15\linewidth]{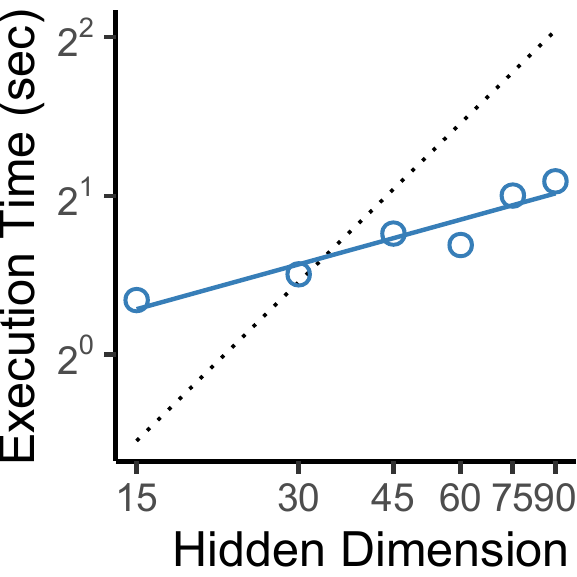}
    }
    \subfigure[enron]{
        \includegraphics[width=0.15\linewidth]{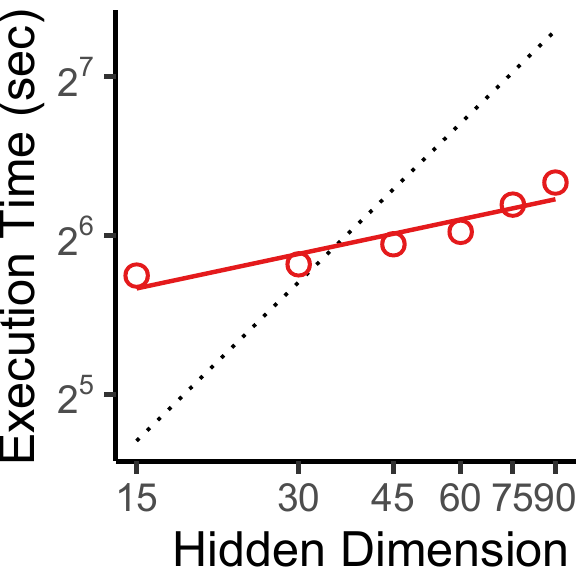}
        \includegraphics[width=0.15\linewidth]{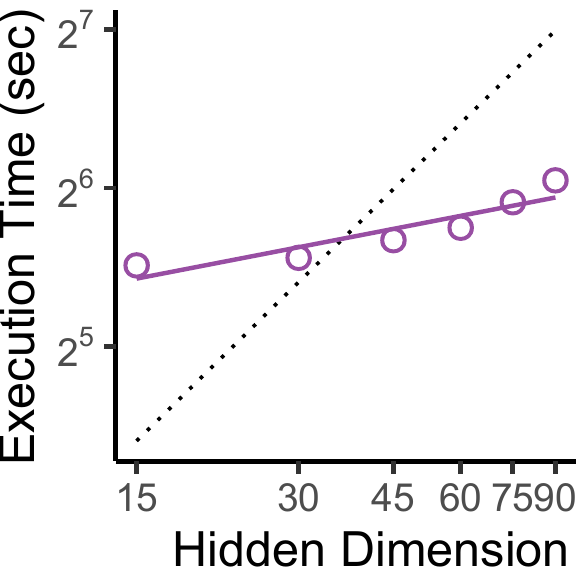}
        \includegraphics[width=0.15\linewidth]{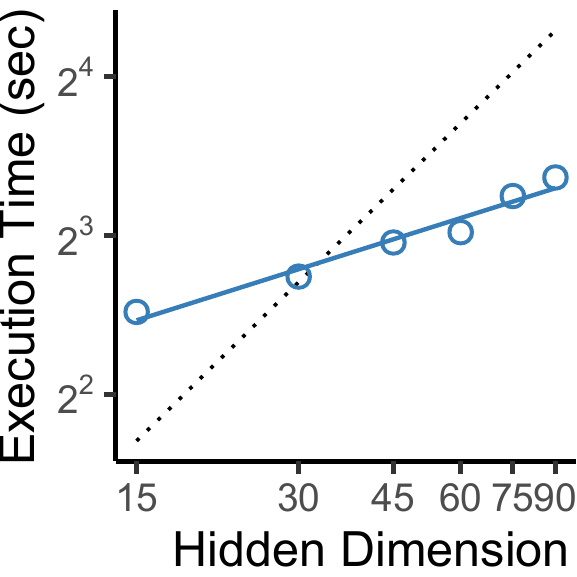}
    }\\
    \vspace{-2mm}
    \caption{\underline{\smash{The training time of \method is empirically sub-linear in the hidden dimension $h$ of \method.}} We measure the total elapsed time, the elapsed time for order optimization, and the elapsed time for model optimization.}
    \label{fig:hiddendim}
\end{figure*}

We followed the default setting in the official code from the authors for the other hyperparameters of ACCAMS and bACCAMS. The implementations of \kronfit, T-SVD, CP, and Tucker used $8$ bytes for real numbers. 
The implementations of ACCAMS and bACCAMS used $4$ bytes for real numbers and assumed the Huffman coding for clustering results.

\section{Hyperparameter analysis \\ (Related to Section 6.1)}
We investigate how the approximation error of \method varies depending on $\gamma$ values. We considered three $\gamma$ values and four datasets (\texttt{email}, \texttt{nyc}, \texttt{tky}, and \texttt{kasandr}) and reported the approximation error in Table~\ref{tab:check}.
Note that setting $\gamma$ to $\infty$ results in a hill climbing algorithm that switches rows/column in pairs only if the approximation error decreases.
The results show that, empirically, the approximation error was smallest when $\gamma$ was set to 10 on all datasets except for the \texttt{nyc} dataset.

\begin{table}[t]
    \centering
    \caption{The effect of $\gamma$ on approximation error. We report the means and standard errors of approximation errors on the \texttt{email}, \texttt{nyc}, \texttt{tky}, and \texttt{kasandr} datasets.}
    \scalebox{0.95}{
    \begin{tabular}{c|c|c}
    \toprule
        Dataset & $\gamma$ & Approximation error \\
    \midrule
        \multirow{3}{*}{\texttt{email}} & 1 & 90561.25 $\pm$ 467.996 \\
         & \textbf{10} & \textbf{58691.88}
 $\pm$ 335.143 \\
         & $\infty$ & 59113.75 $\pm$ 891.544 \\
    \midrule
        \multirow{3}{*}{\texttt{nyc}} & 1 & 421451.2 $\pm$ 4842.068 \\
         & 10 & 402673.6 $\pm$ 17291.959 \\
         & $\pmb{\infty}$ & \textbf{397947.5} $\pm$ 2393.016 \\
    \midrule
        \multirow{3}{*}{\texttt{tky}} & 1 & 4166292.3 $\pm$ 143013.605 \\
         & \textbf{10} & \textbf{3981669.6} $\pm$ 91907.201 \\
         & $\infty$ & 4034389.1 $\pm$ 48117.964 \\
    \midrule
        \multirow{3}{*}{\texttt{kasandr}} & 1 & 6315784.36 $\pm$ 140974.6535 \\
         & \textbf{10} & \textbf{4300280.71} $\pm$ 488804.599 \\
         & $\infty$ & 4385800.32 $\pm$ 496004.629 \\
    \bottomrule
    \end{tabular}}
    \label{tab:check}
\end{table}

\begin{table*}[t]
    \vspace{-2mm}
    \centering
    \caption{\kijung{Comparison of lossy-compression methods for sparse matrices and tensors.
    $nnz(\mathcal{X})$: the number of non-zeros in a matrix/tensor $\mathcal{X}$.
    $D$: the order of the input tensor.
    $N$ \& $M$: the numbers of rows and columns of the input matrix.
    $N_{\max}$: the maximum dimensionality (i.e., $N_{\max}=\max(N_1, \cdots, N_{D}$).
    $R$: rank. $S_c$ \& $S_r$: the numbers of sampled rows and columns.
    $h$: the hidden dimension of the model of \method.
    $T$: the number of iterations of an inner loop. $k$: the number of clusters of rows and columns. $w$: the weight parameter for the criterion of switching. 
    $\alpha, \beta$: parameters for the probability distributions of clusters.
    $E_r, E_c$: the number of rows and columns of a seed matrix.}}
        \setlength{\tabcolsep}{3pt}
        \scalebox{1}{
        \begin{tabular}{llll}
        \toprule
        \multirow{2}{*}{Methods} & Training & Inference & \multirow{2}{*}{Hyperparameters} \\
        & Complexity & Complexity & \\
        \midrule
        \midrule
        \method & $O(h^2nnz(\mathcal{X})\log(M))$ & $O(h^2 \log(N_{\max}))$  & $h$, $w$, optimizer, learning rate  \\
        \midrule
        \midrule
        T-SVD \cite{eckart1936approximation,baglama2005augmented} & $O(nnz(\mathcal{X})R + R^3)$ & $O(R)$ & $R$ \\
        \midrule
        CMD \cite{sun2007less} & $O(nnz(\mathcal{X}) + S_c^3+S_cS_r)$ & $O(S_r)$ & $S_c, S_r$ \\
        \midrule
        CUR \cite{drineas2006fast} & $O(nnz(\mathcal{X}) + S_c^3+S_c^2S_r)$ & $O(S_r)$ & $S_c, S_r$ \\
        \midrule
        ACCAMS \cite{beutel2015accams} & $O(NM + nnz(\mathcal{X})Tk)$ & $O(R)$ & $k, R$ \\
        \midrule
        \multirow{2}{*}{bACCAMS\cite{beutel2015accams}}  & $O(T\{k(N+M)$ & \multirow{2}{*}{$O(R)$} & \multirow{2}{*}{$k, R, \alpha, \beta$} \\    
        & $+nnz(\mathcal{X}) + NM + k^2\})$ & & \\
        \midrule
        \kronfit \cite{leskovec2007scalable,leskovec2010kronecker}& $O(nnz(\mathcal{X})\log(M))$ & $O(\log(N_{\max}))$ & $E_r, E_c$, optimizer, learning rate \\
        \midrule
        \midrule
        CP\cite{carroll1970analysis} & $O(nnz(\mathcal{X})DR)$ & $O(DR)$ & $R$\\
        \midrule
        Tucker\cite{tucker1966some} & $O(nnz(\mathcal{X})DR)$ & $O(DR^D)$ & $R$\\
        \bottomrule
    \end{tabular}}
    \label{tab:proscons:detail}    
\end{table*}

\section{Speed and Scalability on hidden dimension (Related to Appendix~A.1)}

We report the average the training time per epoch of \method in Table \ref{tab:time_per_epoch}. 
The training time per epoch varied from less than 1 second to more than 9 minutes depending on the dataset.
As seen in Figure~\ref{fig:appdix:training_plot}, the training plots of all datasets dropped dramatically within one third of total epochs that were determined by the termination condition in Section 6.1. 
Thus, a model that worked well enough could be obtained before convergence.

We also analyzed the effect of the hidden dimension $h$ on the training time per epoch of \method. 
As seen in Figure \ref{fig:hiddendim}, both the elapsed time for order optimization and the elapsed for model optimization were empirically sublinear in the hidden dimension.

\begin{figure}[t]
    \centering
    \includegraphics[width=0.777\linewidth]{figures/legends/legends5.pdf}
    \subfigure[\texttt{4-gram}]{
        \centering
        \includegraphics[width=0.3\linewidth]{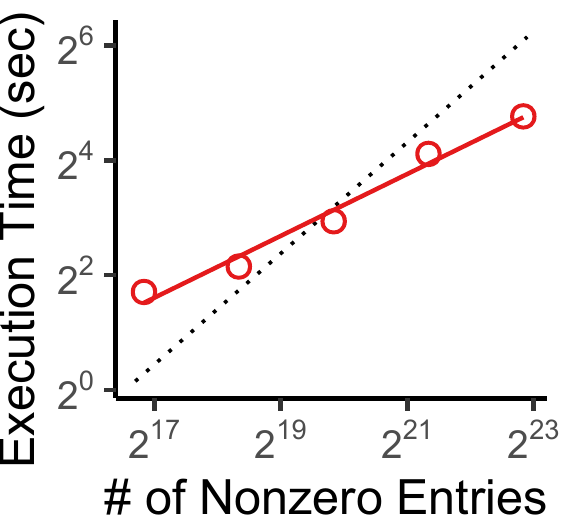}
        \includegraphics[width=0.3\linewidth]{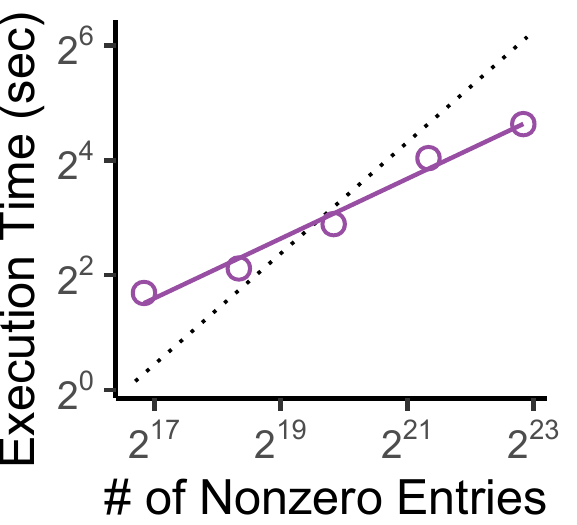}
        \includegraphics[width=0.3\linewidth]{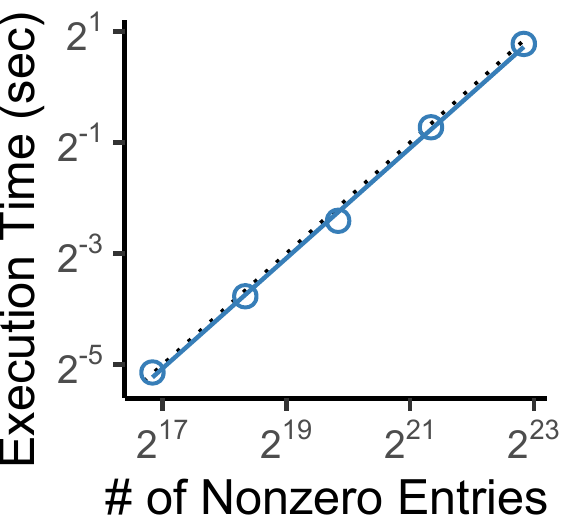}
    } \\
    \vspace{-2mm}
    \subfigure[\texttt{enron}]{
        \centering
        \includegraphics[width=0.3\linewidth]{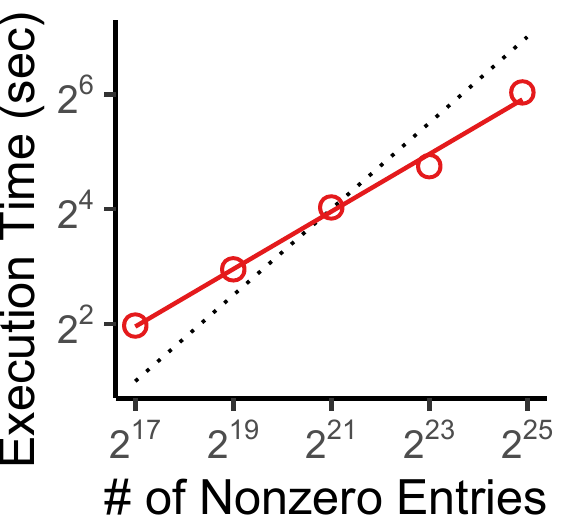}
        \includegraphics[width=0.3\linewidth]{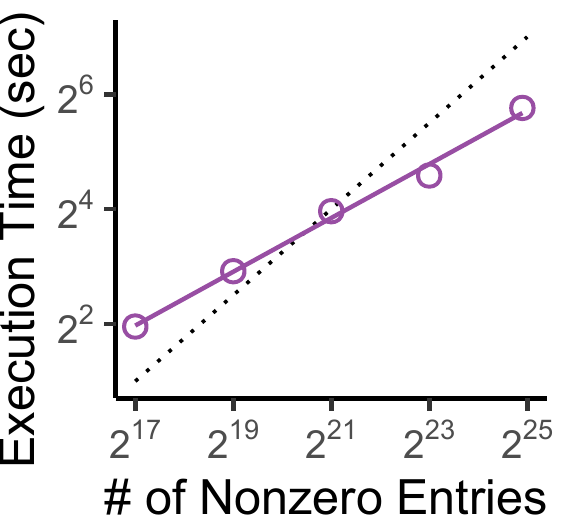}
        \includegraphics[width=0.3\linewidth]{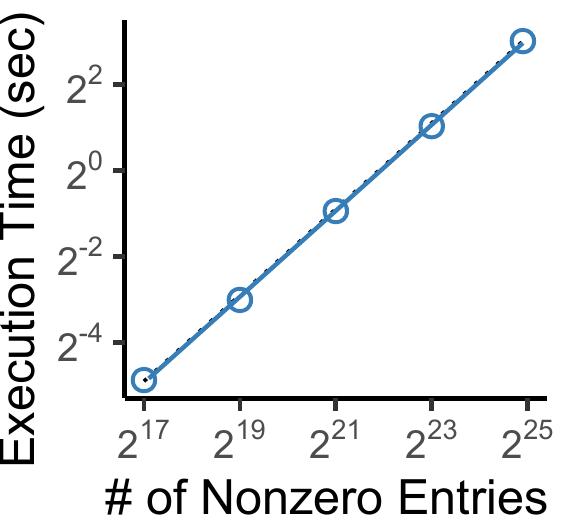}
    }\\
    \vspace{-2mm}
    \caption{\underline{\smash{The training process of \method on tensor is}} \underline{\smash{also scalable.}} 
    Both model and order optimizations scale near-linearly with the number of non-zeros in the input.
    }
    \label{fig:scalability_t}
\end{figure}

\section{Scalability on Tensor Datasets \\ (Related to Appendix~A.1)} 
\label{app:exp:scalability}
For the \texttt{4-gram} and \texttt{enron} datasets, we generated multiple smaller tensors by sampling non-zero entries uniformly at random. The hidden dimension was fixed to $60$. Consistently with the results on matrices, the overall training process of $\method$ is also linearly scalable on sparse tensors, as seen in Figure \ref{fig:scalability_t}.

\begin{table}[t]
    \centering
    \caption{Training time per epoch on all datasets. We report the means and standard errors.}
    \begin{tabular}{c|c}
    \toprule
        Dataset & \multirow{2}{*}{Training time}\\
        (Hidden Dimension) & \\
    \midrule
        \texttt{email} (30) & 0.19 $\pm$ 0.010 \\
        \texttt{nyc} (30) & 0.21 $\pm$ 0.004 \\
        \texttt{tky} (30) & 0.32 $\pm$ 0.005 \\
        \texttt{kasandr} (60) & 1.93 $\pm$ 0.005 \\
        \texttt{threads} (60) & 5.49 $\pm$ 0.012 \\
        \texttt{twitch} (90) & 566.82 $\pm$ 3.308 \\
        \texttt{nips} (50) & 6.31 $\pm$ 0.081 \\
        \texttt{enron} (90) & 80.69 $\pm$ 0.266 \\
        \texttt{3-gram} (90) & 27.19 $\pm$ 0.089 \\
        \texttt{4-gram} (90) & 41.09 $\pm$ 0.785 \\
    \bottomrule
    \end{tabular}
    \label{tab:time_per_epoch}
\end{table}

\begin{table*}[ht]
    \centering
    \caption{\kijung{Structural properties of real-world datasets.}}
    \def\arraystretch{0.7}
    \setlength{\tabcolsep}{2.5pt}
    \begin{tabular}{c|cccc|c|c}
    \toprule
        Dataset & \multicolumn{4}{c|}{Degrees} & Entry Values & Connected Components \\
        \midrule
         \rotatebox[origin=l]{90}{\hspace{4mm}\texttt{email}} & & \includegraphics[width=0.14\linewidth]{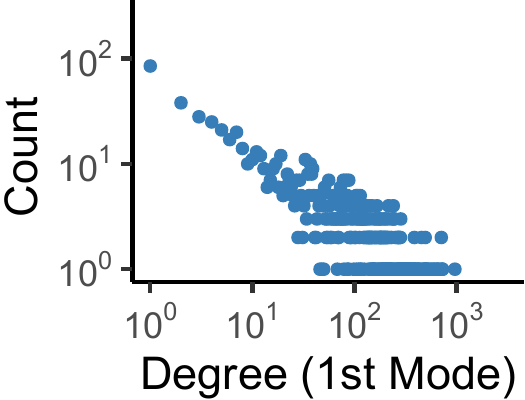} & \includegraphics[width=0.14\linewidth]{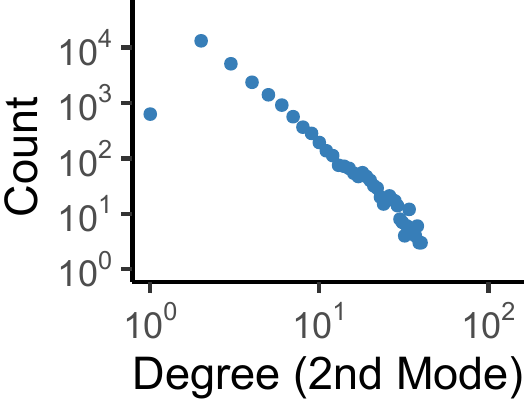} &  & \includegraphics[width=0.14\linewidth]{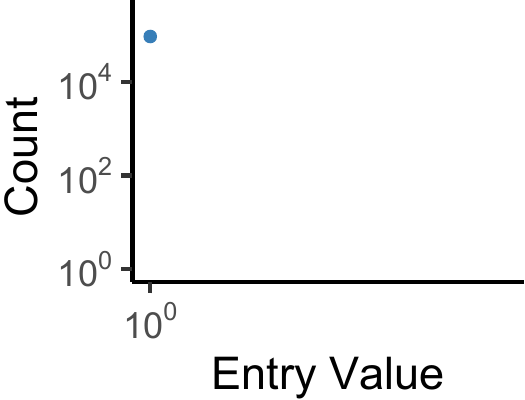} & \includegraphics[width=0.14\linewidth]{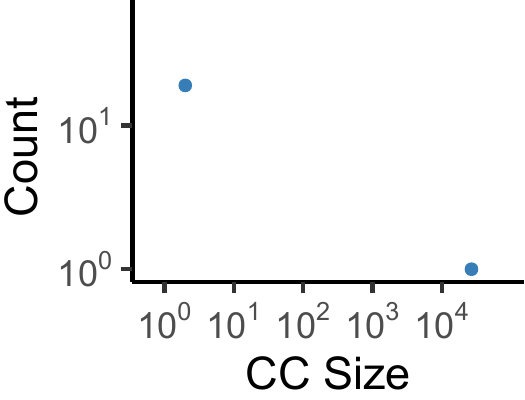} \\
         \midrule
         \rotatebox[origin=l]{90}{\hspace{6mm}\texttt{nyc}} &  & \includegraphics[width=0.14\linewidth]{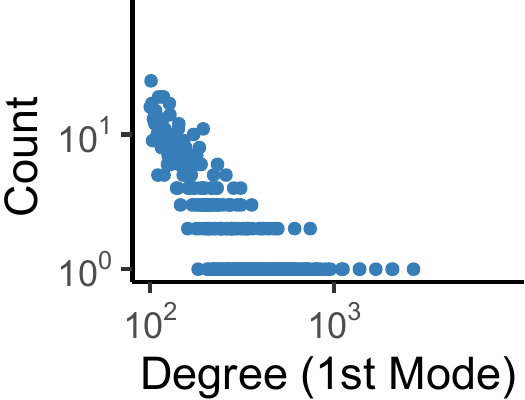} & \includegraphics[width=0.14\linewidth]{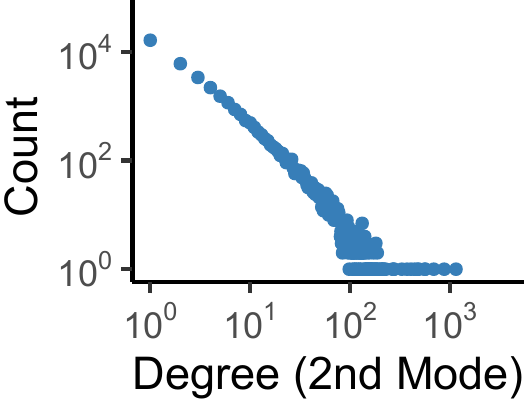} &  & \includegraphics[width=0.14\linewidth]{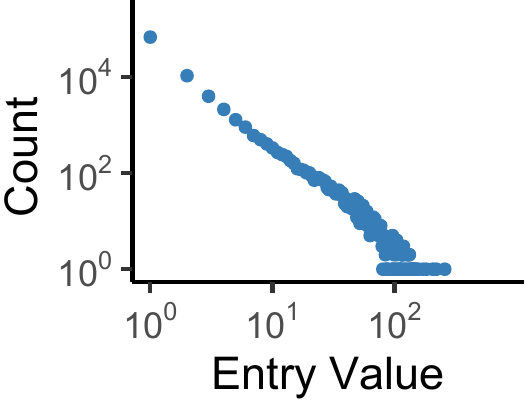} & \includegraphics[width=0.14\linewidth]{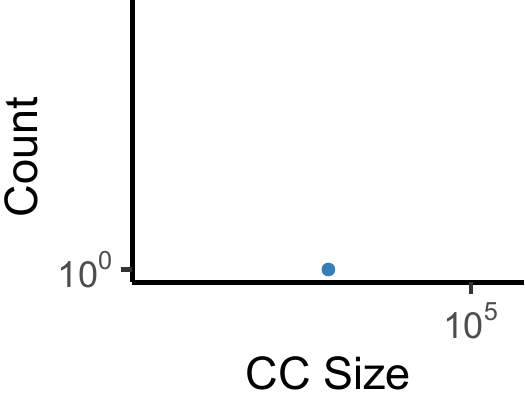} \\
         \midrule
         \rotatebox[origin=l]{90}{\hspace{6mm}\texttt{tky}} &  & \includegraphics[width=0.14\linewidth]{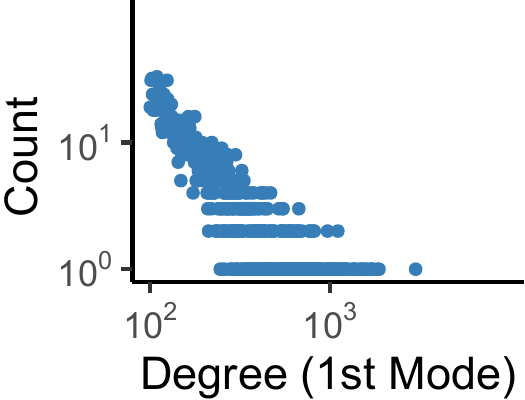} & \includegraphics[width=0.14\linewidth]{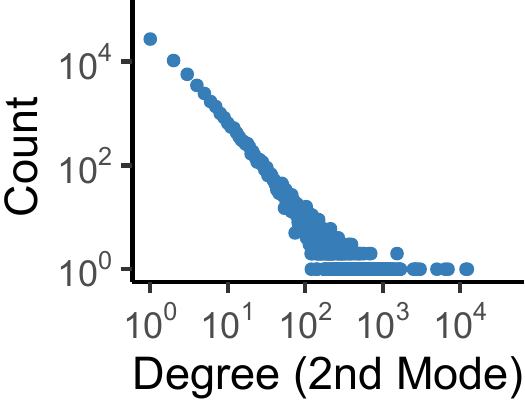} &  & \includegraphics[width=0.14\linewidth]{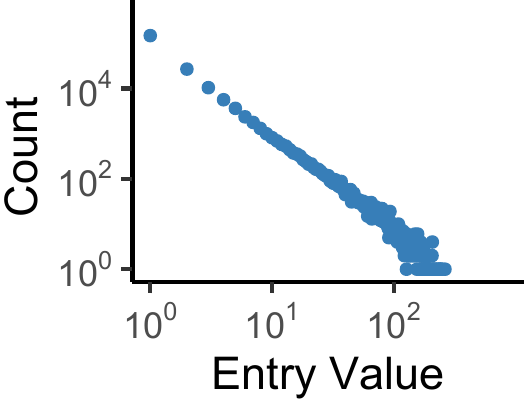} & \includegraphics[width=0.14\linewidth]{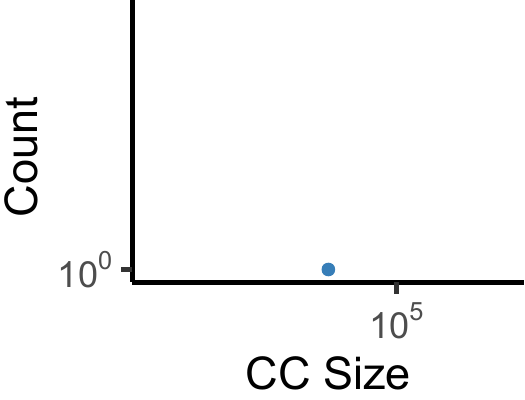} \\
         \midrule
         \rotatebox[origin=l]{90}{\hspace{2.5mm}\texttt{kasandr}} & & \includegraphics[width=0.14\linewidth]{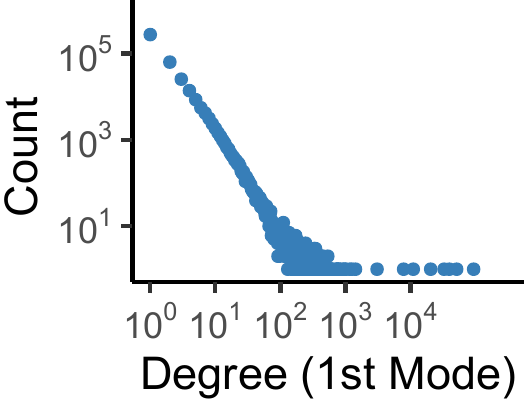} & \includegraphics[width=0.14\linewidth]{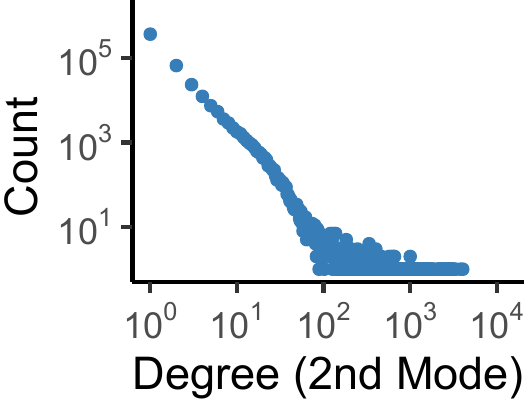} &  & \includegraphics[width=0.14\linewidth]{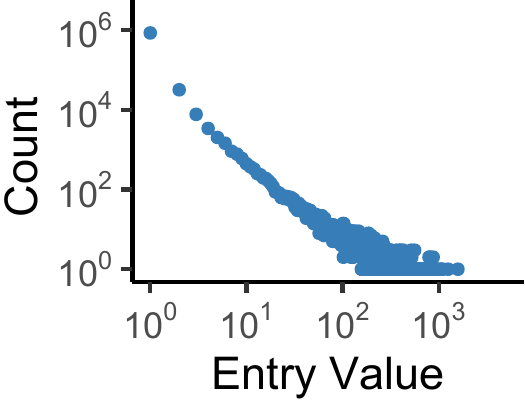} & \includegraphics[width=0.14\linewidth]{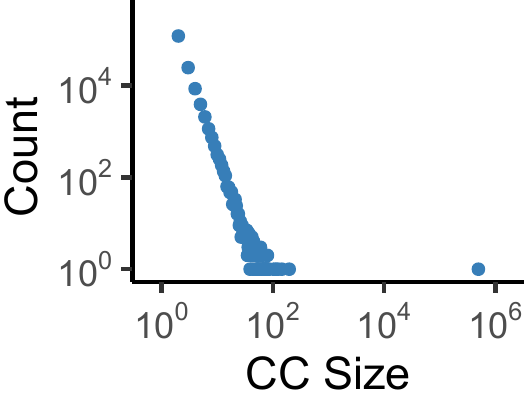} \\
         \midrule
         \rotatebox[origin=l]{90}{\hspace{2.5mm}\texttt{threads}} & & \includegraphics[width=0.14\linewidth]{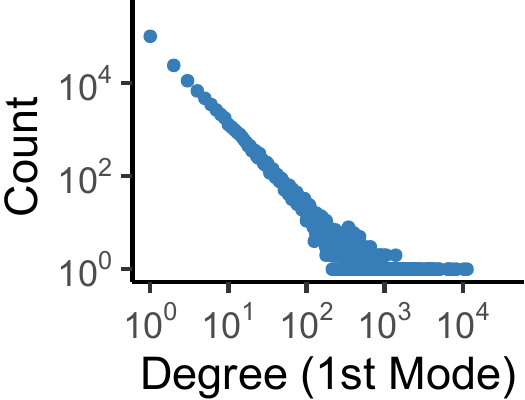} & \includegraphics[width=0.14\linewidth]{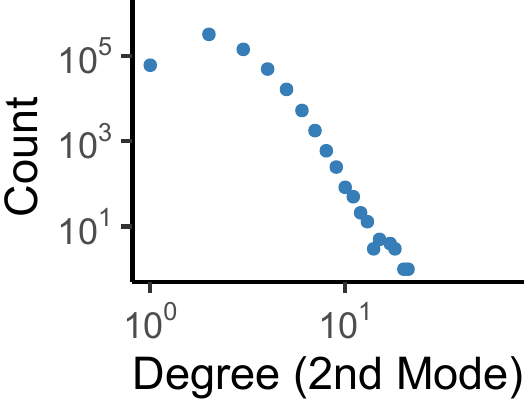} &  & \includegraphics[width=0.14\linewidth]{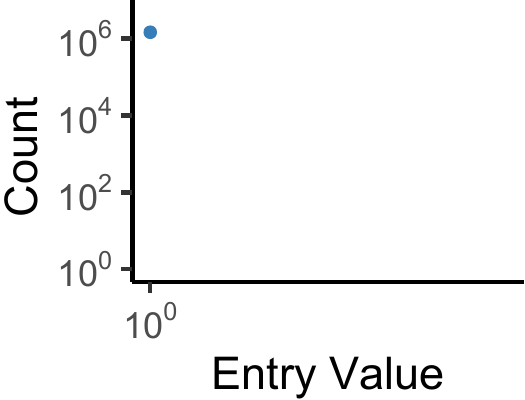} & \includegraphics[width=0.14\linewidth]{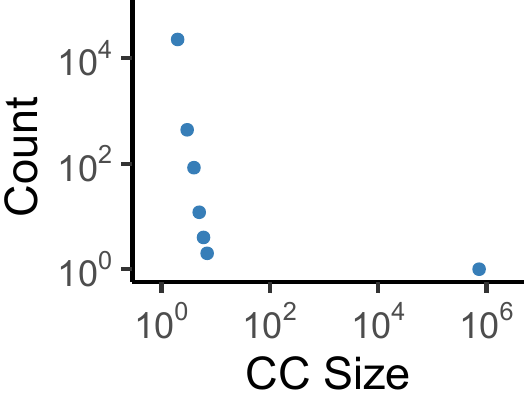} \\
         \midrule
         \rotatebox[origin=l]{90}{\hspace{3mm}\texttt{twitch}} & & \includegraphics[width=0.14\linewidth]{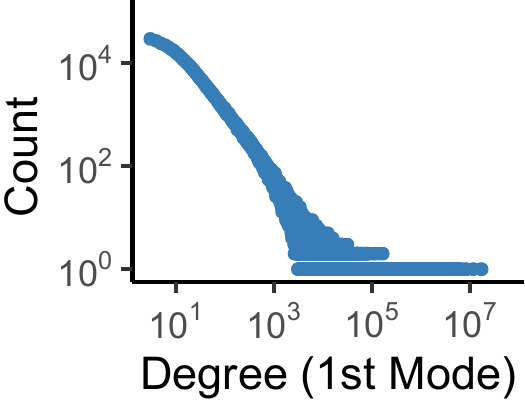} & \includegraphics[width=0.14\linewidth]{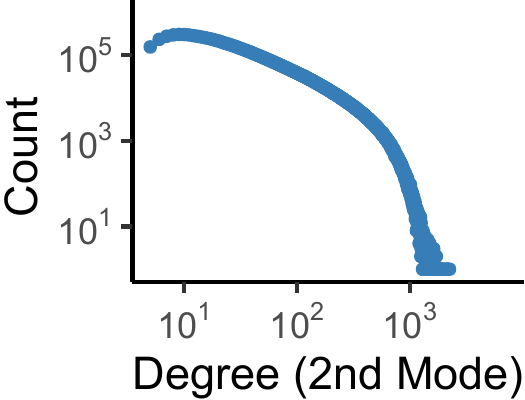} &  & \includegraphics[width=0.14\linewidth]{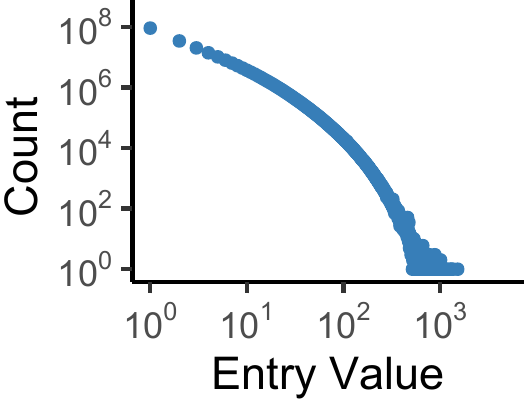} & \includegraphics[width=0.14\linewidth]{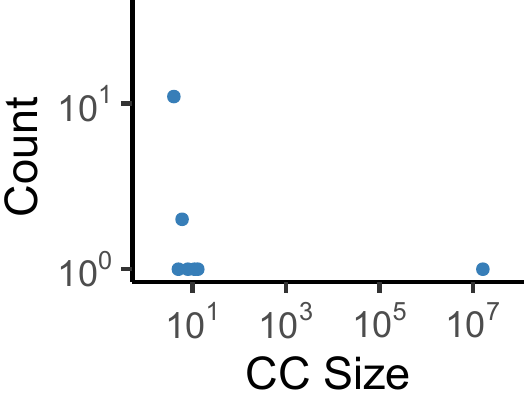} \\
         \midrule
         \rotatebox[origin=l]{90}{\hspace{5mm}\texttt{nips}} & \multicolumn{4}{c|}{\includegraphics[width=0.14\linewidth]{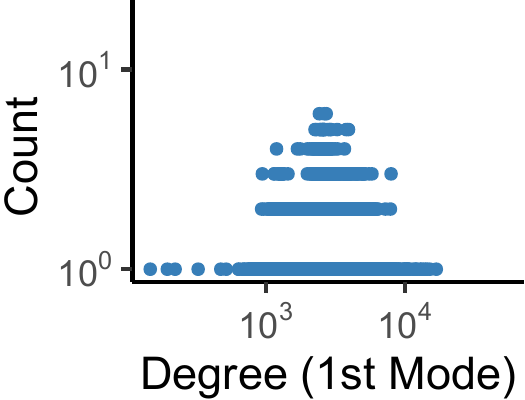} \includegraphics[width=0.14\linewidth]{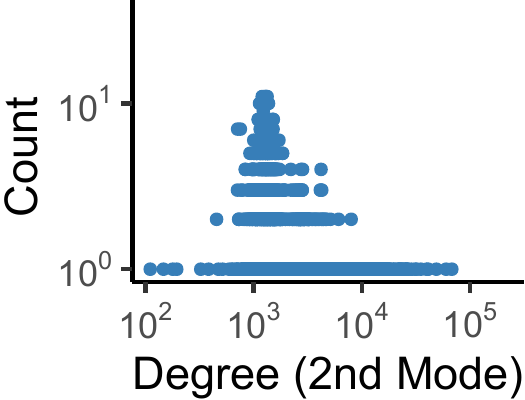} \includegraphics[width=0.14\linewidth]{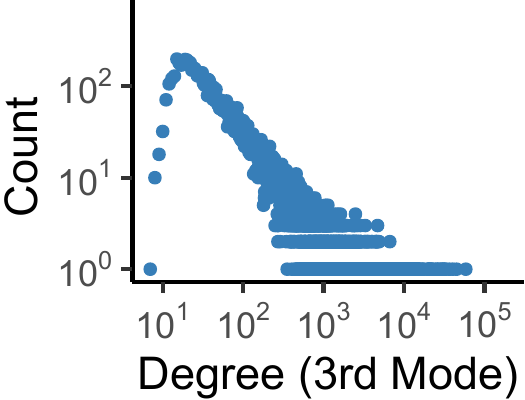}} & \includegraphics[width=0.14\linewidth]{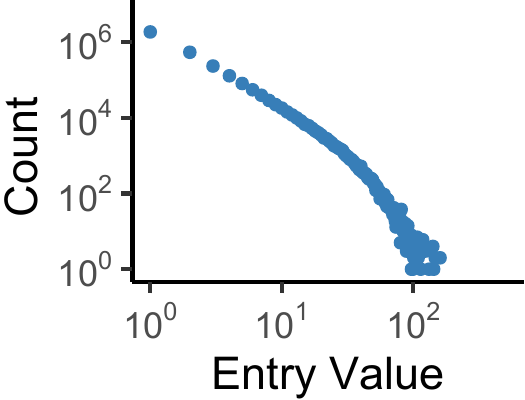} & \includegraphics[width=0.14\linewidth]{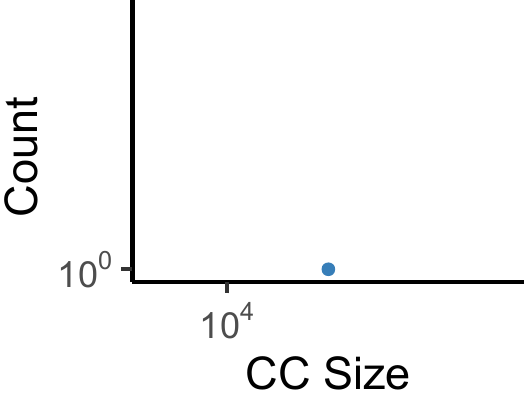} \\
         \midrule
         \rotatebox[origin=l]{90}{\hspace{4mm}\texttt{enron}} & \multicolumn{4}{c|}{\includegraphics[width=0.14\linewidth]{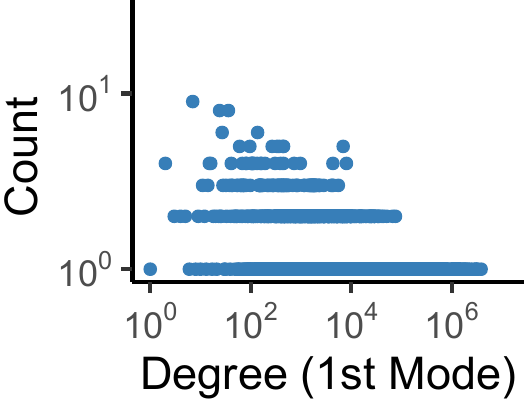} \includegraphics[width=0.14\linewidth]{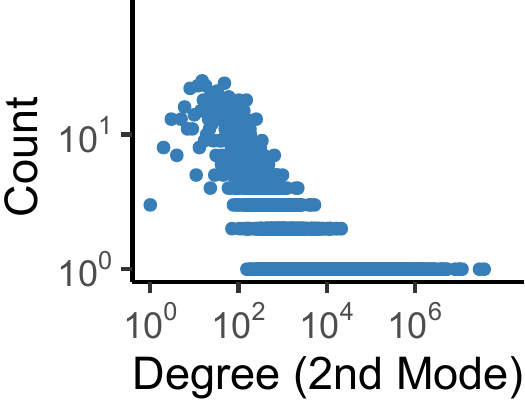} \includegraphics[width=0.14\linewidth]{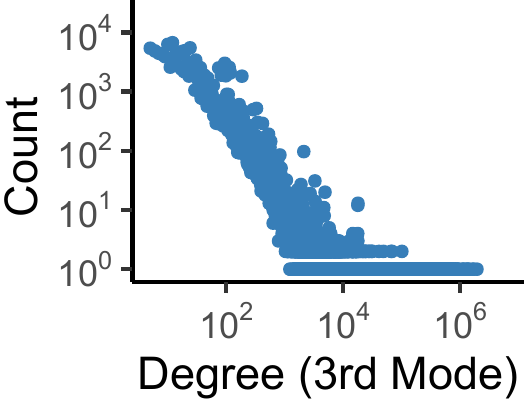}} & \includegraphics[width=0.14\linewidth]{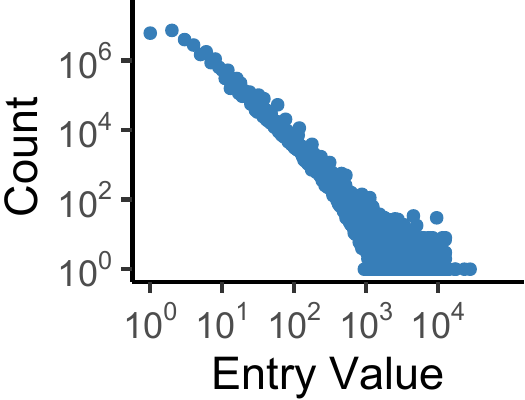} & \includegraphics[width=0.14\linewidth]{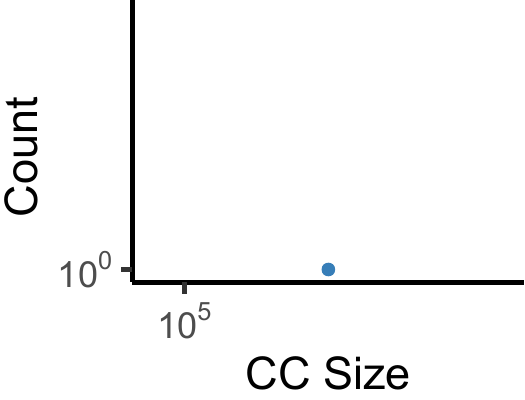} \\
         \midrule
         \rotatebox[origin=l]{90}{\hspace{3mm}\texttt{3-gram}} & \multicolumn{4}{c|}{\includegraphics[width=0.14\linewidth]{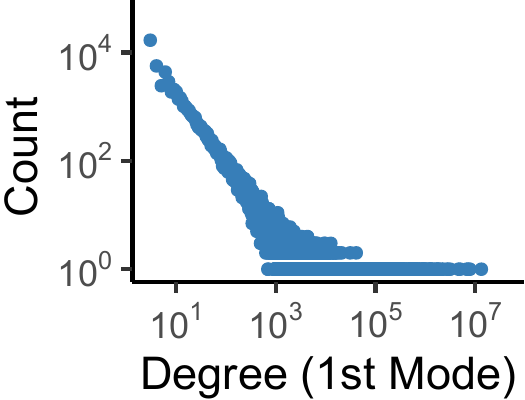} \includegraphics[width=0.14\linewidth]{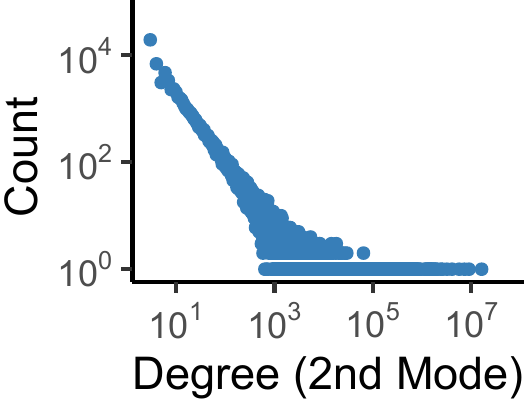} \includegraphics[width=0.14\linewidth]{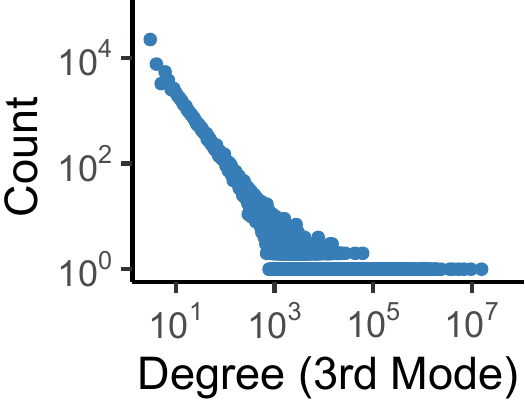}} & \includegraphics[width=0.14\linewidth]{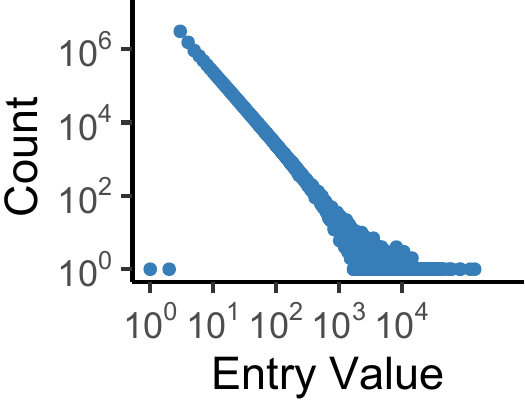} & \includegraphics[width=0.14\linewidth]{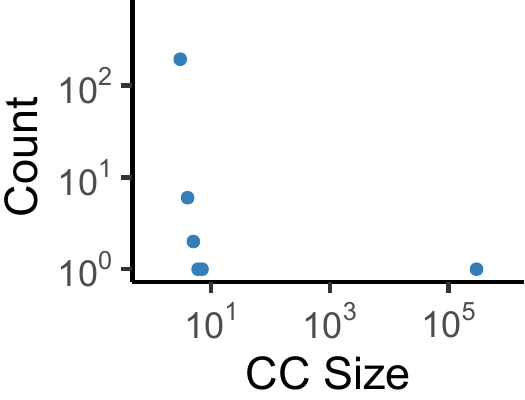} \\
         \midrule
         \rotatebox[origin=l]{90}{\hspace{3mm}\texttt{4-gram}} & \includegraphics[width=0.14\linewidth]{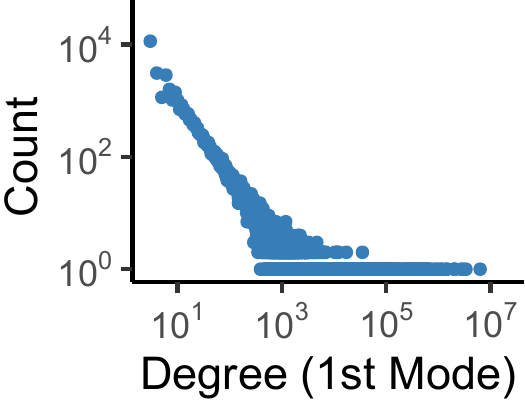} & \includegraphics[width=0.14\linewidth]{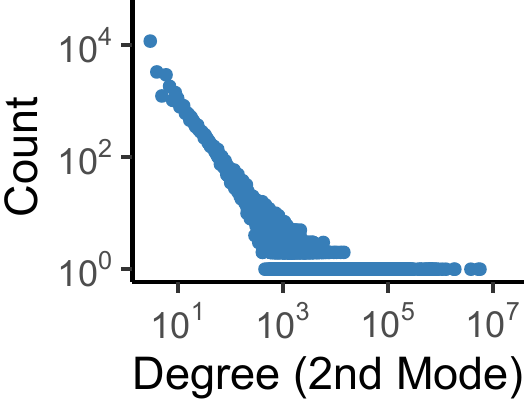} & \includegraphics[width=0.14\linewidth]{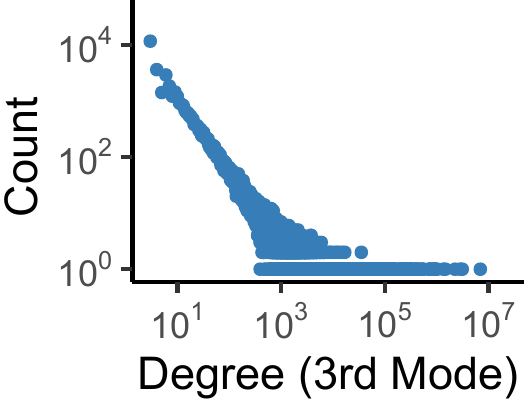} & \includegraphics[width=0.14\linewidth]{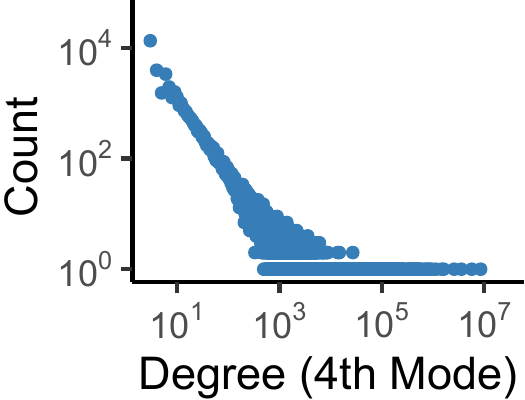} & \includegraphics[width=0.14\linewidth]{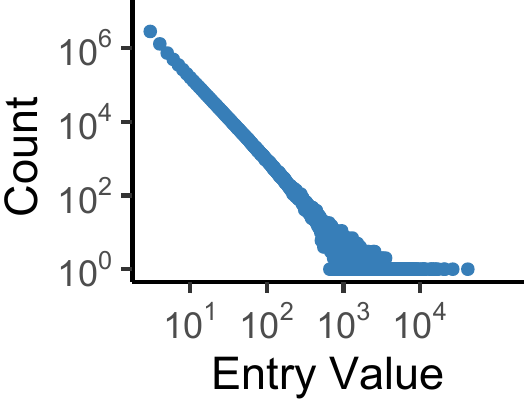} & \includegraphics[width=0.14\linewidth]{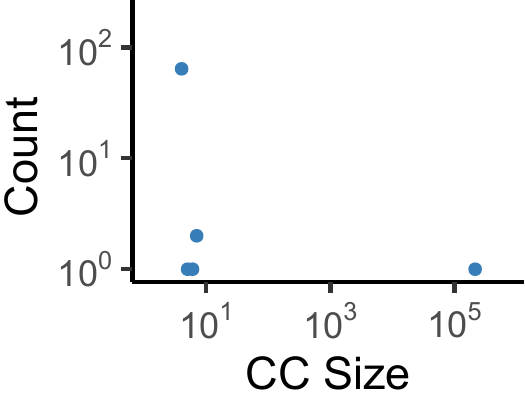} \\
         \bottomrule
    \end{tabular}
    \label{tab:dataset:properties}
\end{table*}

\section{Comparison of Lossy Compression Methods \\ (Related to Section~2)}
\kijung{In Table~\ref{tab:proscons:detail}, we provide a comparison of lossy compression methods for sparase matrices and tensors, which supplement Table~1 in the main paper.}

\balance

\bibliographystyle{IEEEtran}
\bibliography{bib}